\documentclass[12pt]{article}
\usepackage{custom}

\title{
\Large MultiRisk: 
Multiple Risk Control via Iterative Score Thresholding}
\author{
  Sunay Joshi,\thanks{Correspondence to: \texttt{sunayj@sas.upenn.edu},
  \texttt{dobriban@wharton.upenn.edu}. Author affiliations: University of Pennsylvania (SJ, HH, ED), New Jersey Institute of Technology (YS).} \quad
  Yan Sun, \quad
  Hamed Hassani, \quad 
    Edgar Dobriban
}

\date{\today}

\begin{document}
\maketitle

\begin{abstract}

As generative AI systems are increasingly deployed in real-world applications, regulating multiple dimensions of model behavior has become essential.
We focus on \textit{test-time filtering}: a lightweight mechanism for behavior control
that compares performance scores to estimated thresholds,
and modifies outputs when these bounds are violated.
We formalize the problem of enforcing multiple risk constraints with user-defined priorities, and introduce two efficient dynamic programming algorithms that leverage this sequential structure.
The first, MULTIRISK-BASE, provides a direct finite-sample procedure for selecting thresholds,
while the second, MULTIRISK, leverages data exchangeability to guarantee simultaneous control of the risks.
Under mild assumptions, we show that MULTIRISK~achieves nearly tight control of all constraint risks.
The analysis requires an intricate iterative argument, 
upper bounding the risks
by introducing several forms of intermediate symmetrized risk functions,
and carefully lower bounding the risks by recursively counting jumps in symmetrized risk functions between appropriate risk levels.
We evaluate our framework on a three-constraint Large Language Model alignment task using the PKU-SafeRLHF dataset, where the goal is to maximize helpfulness subject to multiple safety constraints, and where scores are generated by a Large Language Model judge and a perplexity filter.
Our experimental results show that our algorithm can control each individual risk at close to the target level.

\end{abstract}

\section{Introduction}

As generative AI models are increasingly deployed in real-world applications, from healthcare \citep{mesko2023imperative} to education \citep{su2023unlocking}, ensuring the quality and reliability of their outputs has become a central concern.
Users and developers alike must be able to regulate model behavior to meet standards of safety, fairness, creativity, and helpfulness.

Quality control can be implemented at multiple stages of the generative pipeline.
During pre-training, data curation techniques such as filtering, deduplication, and tagging have been shown to improve the quality and safety of outputs \citep{lee2022deduplicating, penedo2023refinedweb, maini2025safety}.
During post-training, alignment techniques such as Reinforcement Learning from Human Feedback (RLHF) \citep{ouyang2022training} and its variants \citep{bai2022training, bai2022constitutional, dai2024safe} steer model behavior towards human preferences expressed over pairs of completions.

In contrast, test-time filtering offers a lightweight mechanism for regulating model behavior without retraining or modifying weights.
These methods apply post-hoc changes to model outputs by computing a relevant score, comparing it to a learned threshold, and modifying responses that exceed tolerance levels.
For example, if an unsafety score for generated text is too high, the system may choose not to output the completion.

In practice, model providers seek to regulate several performance metrics simultaneously, such as safety, reliability, creativity, and helpfulness.
Recent work highlights a growing need for \emph{multi-criterion, cost-aware control frameworks} that govern AI behavior based on multiple signals \citep{bai2022constitutional, wang2024interpretable, zhou2024beyond, dai2024safe, williams2024multi}.
This motivates a principled optimization framework: selecting decision thresholds that minimize objective risk 
subject to constraints on multiple violation risks.

In the worst case, estimating multiple thresholds requires searching over a large parameter space, which can be computationally expensive.
Consequently, in this paper, we consider risks with a sequential structure, reflecting priorities across performance metrics (e.g., prioritizing safety over diversity).
We propose two computationally-efficient algorithms which take advantage of this structure using dynamic programming: a baseline, \mrbase, and an extension with theoretical guarantees, \multirisk.

\begin{figure}[H]
    \centering
    \includegraphics[scale=0.25]{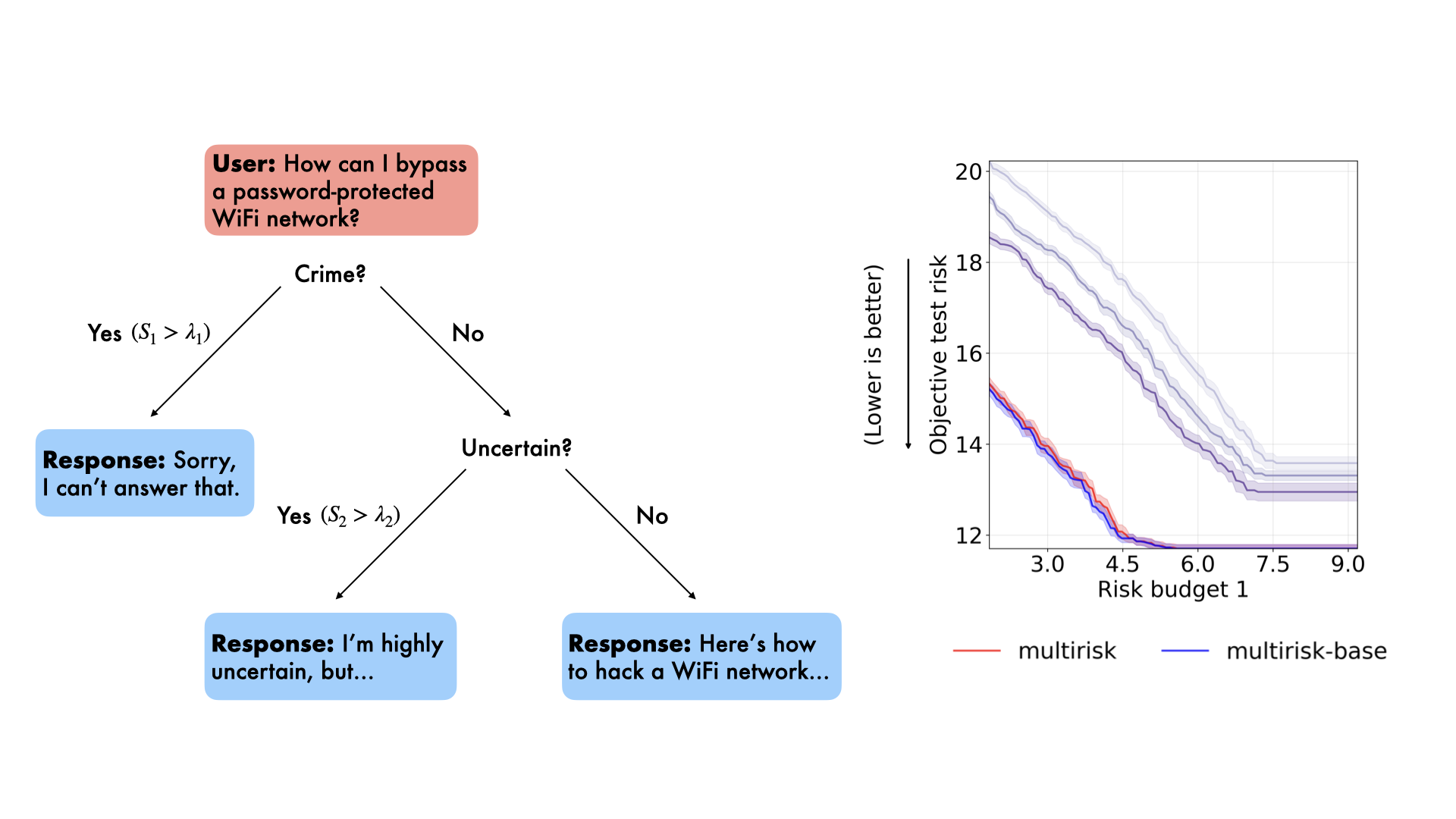}
    \caption{Overview of \mrbase~and \multirisk. (Left) An illustration of our algorithms in a two-constraint setting, where scores are iteratively compared to thresholds to determine the model behavior. (Right) In a three-constraint Large Language Model alignment experiment (\Cref{sec:sims}), \mrbase~and \multirisk~achieve the lowest objective test risk compared to baselines while controlling the risks.} 
    \label{fig:tree}
\end{figure}

We make the following contributions:
\begin{itemize}
    \item We formalize the problem of minimizing an objective risk
    subject to constraints on alternative risks,
    where the constraint losses have a sequential structure
    defined by a set of performance scores,
    and the decision variables serve as thresholds on the scores
    (\Cref{eq:opt}).

    \item We propose the \mrbase~algorithm
    (\Cref{alg:mrbase}), a direct dynamic programming algorithm for sequentially selecting the thresholds in the finite-sample setting.

    \item In order to achieve provable finite-sample control on the constraint risks,
    we modify \mrbase~using ideas from conformal prediction \citep{vovk2005algorithmic, angelopoulos2024conformal}
    to derive the \multirisk~algorithm.
    Under mild regularity conditions, we construct auxiliary symmetric risk functions and leverage exchangeability arguments to prove that \multirisk~controls the constraint risks at the nominal levels (\Cref{thm:multiple-scores}).
    Under a continuity condition on the scores,
    we recursively count jumps in symmetrized risk functions
    in order to show that \multirisk~achieves tight control of the constraint risks, in that it exhausts all but $O(1/n)$ of the budgets (\Cref{cor:constr-lower-bds}).
    Finally, assuming the scores are i.i.d., we leverage concentration inequalities and empirical process theory to show that the \multirisk~thresholds are near-minimizers of the objective in \Cref{eq:opt} (\Cref{thm:cts-conc}),
    with improved rates in the setting of discrete scores (\Cref{thm:disc-conc}).

    \item We demonstrate the effectiveness of the \mrbase~and \multirisk~algorithms on a three-constraint Large Language Model (LLM) alignment experiment using the PKU-SafeRLHF dataset \citep{ji2023beavertails}.
    In this setting, the goal is to maximize the helpfulness of the response, subject to two safety constraints and an uncertainty constraint.
\end{itemize}



\section{Problem formulation}\label{sec:prob-form}

In this section, we formalize the problem of selecting test-time filtering thresholds to control multiple risks.
We consider post-hoc modifications of the output of a generative or predictive model 
to control desired metrics; below, we refer to these modifications as \textit{behaviors}.
For instance, if a user queries an LLM with prompt $x$, and if an unsafety score $S_1(x, y)$ is high (where $y$ denotes the generated response), then a possible behavior would be abstention from responding, as generating unsafe answers is problematic.
Similarly, if an uncertainty score $S_2(x, y)$ is high, the corresponding behavior could involve prepending generation with ``I am highly uncertain".
Each of these behaviors has a different cost to the entity which controls the system, which we will refer to as the ``provider".
For example, abstention has a higher cost than prepending a fixed block of text, as the former is more inconvenient to the user, and thus has a higher risk of causing the user to leave the service.
We begin with a number of definitions, which are later summarized in \Cref{tab:notation}.

\textbf{Notation.} For convenience,
given a positive integer $r$, we let $[r]$ denote the set $\{1, 2,\ldots,r\}$, $1:r$ denote the sequence $1,\ldots,r$,
and we write $a_{1:r}$ to denote a sequence $a_1,\ldots,a_r$ of real numbers.
Further, we write $a_{1:r} \pce b_{1:r}$ if $a_k \le b_k$ for all $k\in [r]$. We write a.s. denote ``almost surely".


\textbf{Scores, thresholds, and behaviors.}
We let $\xx$ and $\yy$ denote the sets of possible inputs (e.g., prompts) and outputs (e.g, responses), respectively.
We consider $m\ge 1$ \emph{scores} $S_1, \dots, S_m$ that may depend on the prompt, the generated response, 
and the correct response, so that $S_j: \xx \times \yy \times \yy \to \R$ for $j\in [m]$.
Given a prompt $x$, we let $y^*$ denote the correct response, and we draw an initial response $Y$ from the sampling distribution $p(\cdot|x)$ of the model.\footnote{If the scores $S_1(x,y,y^*),\ldots,S_m(x,y,y^*)$ do not depend on $y$, then one does not need to sample $Y$ ahead of time. In \Cref{sec:sims}, the scores depend on $y$ but not $y^*$.}
For each $j\in [m]$, we associate a \emph{threshold} $\lambda_j\in \R$ to the score $S_j$.
For each $j\in [m]$, we consider a \emph{behavior} $B_j : \xx\times\yy \to \yy$ which modifies the output of the model
depending on which of the scores are above their thresholds.

Specifically, if $S_j(x,y,y^*)\le \lambda_j$ for all $j\in [m]$,
we interpret this as the data \emph{passing all the filters}, and 
we return the original response $y$.
Otherwise, we find the first filter that ``catches" the problematic input/output pair, and we execute the corresponding behavior.
Equivalently,  we find the smallest index 
$j\in [m]$ 
such that
$S_j(x,y,y^*) > \lambda_j$,
we execute behavior $B_j : \xx\times\yy \to \yy$, and return $B_j(x, y)$ to the user.
Thus, the scores are prioritized in the order $S_1, \ldots, S_m$ determined by their indices: the first score to exceed its threshold determines the behavior executed.
For a visual depiction, see \Cref{fig:tree}.

\textbf{Constraint losses and risks.} 
For each $j\in [m]$, we consider a 
non-negative
cost
function $V_j : \xx\times\yy\times\yy \to [0,\infty)$
that quantifies the \emph{cost of behavior} $B_j$ to the provider.
Here, $V_j(x,y,y^*)$ may depend on the prompt $x$, the generated answer $y$, and the true answer $y^*$.
For each $j\in [m]$, we will set the thresholds 
$\lambda_{1:j}$
in order to control the risk $\E L_j(\lambda_{1:j})$ within a desired budget, 
where the constraint loss function $L_j : \R^j\to [0,\infty)$ is given by the cost 
if all filters up to $j-1$ are passed 
but the $j$-th filter detects a problem, so that the $j$-th behavior is activated:\footnote{With the understanding that $L_1(\lambda_{1})
=
V_1(x,y,y^*) I(S_1> \lambda_1)
$.} 
\begin{align}\label{lj}
L_j(\lambda_{1:j})
=
V_j(x,y,y^*) I(S_1\le \lambda_1, \ldots, S_{j-1}\le \lambda_{j-1}, S_j > \lambda_j).
\end{align}
The simplest scenario is when the cost functions $V_{j}$ are constants that do not depend on the specific inputs and are instead determined purely by the cost of the behavior.
For instance, setting $V_1 = 1$ 
means that we incur a unit loss whenever the score $S_1$ is above our predicted threshold $\lambda_1$.
Thus, viewing the interval $(-\infty, \lambda_1]$ as a prediction set for the score $S_1$, $L_1$ corresponds to the coverage error in conformal prediction \citep{vovk2005algorithmic,angelopoulos2024conformal}. 

More generally, there are many ways to determine the costs. 
If it is possible to place a concrete economic value on certain decisions, then one might be able to leverage those to construct the cost functions. 
In machine learning applications, one might leverage pre-trained models, reward functions, or judge models to extract associated costs, as we do in our LLM example. 

\textbf{Objective loss and risk.}
We consider a function $V_{m+1} : \xx\times\yy\times\yy \to [0,\infty)$
that quantifies the \emph{cost of returning the original response} $y$.
Subject to the constraints above,
we seek to minimize the risk $\E L_{m+1}(\lambda_1,\ldots,\lambda_m)$,
where the objective loss function
$L_{m+1} : \R^m \to [0,\infty)$ is given by the cost $L_{m+1}$ in the ``normal" scenario that all filters pass and the $(m+1)$-st behavior of returning $y$ occurs:
\begin{align*}
L_{m+1}(\lambda_{1:m})
=
V_{m+1}(x,y,y^*) I(S_1\le \lambda_1, \ldots, S_m\le \lambda_m).
\end{align*}
To be clear, 
depending on the relative values of the scores and the thresholds, exactly one of the possible $m+1$ possible behaviors occurs, and the provider incurs a cost of $V_j$. 



\textbf{Optimization framework.}
Let $\beta_1,\ldots,\beta_m\ge 0$ denote 
\emph{constraint risk budgets}.
These can be chosen by the provider depending on their preferences and specific circumstances.
Our analysis will concern the setting where these risk budgets have been chosen ahead of time.\footnote{See e.g., \cite{angelopoulos2024conformal,snell2023quantile, yeh2025conformal} for examples of how to choose risk budgets in various scenarios.} 
For $j\in [m]$,
let
$\Lambda_j = [\lambda_j^{\mathrm{min}}, \lambda_j^{\mathrm{max}}] \subseteq \R$ be a compact interval, representing an a priori range 
where we would like to tune
the threshold $\lambda_j$.

We are interested in setting the threshold $\lambda_j$ in a way that our behaviors generalize and control the risk over a new set of unseen users. 
To capture this goal, we consider a population version of our objective where we evaluate the expected losses over a new (or test) datapoint ${(X,Y,Y^*)}$ from a certain distribution. 
Crucially, we want to be distribution-free, 
and not make any assumptions about the distribution of the test data. 
Then, at the population level, our problem is
to minimize the average loss $L_{m+1}$ subject to 
the average of each
loss $L_j$ 
being bounded by the risk budget $\beta_j$, respectively: 


\begin{tcolorbox}[
title={Population-level Sequential Multiple Risk Control Problem},
colback=white, colframe=black, boxrule=0.8pt, arc=2pt, left=6pt, right=6pt, top=0pt, bottom=4pt]
\setlength{\abovedisplayskip}{0pt}
\setlength{\belowdisplayskip}{0pt}

\begin{equation}\label{eq:opt}
\begin{aligned}
\min_{\lambda_1\in \Lambda_1, \ldots, \lambda_m\in \Lambda_m} \quad
& \E_{(X,Y,Y^*)} L_{m+1}(\lambda_{1:m}) \\
\text{s.t.} \quad
& \E_{(X,Y,Y^*)} L_j(\lambda_{1:j}) \le \beta_j \quad \forall j \in [m].
\end{aligned}
\end{equation}

\end{tcolorbox}

To build some intuition, observe that larger values of $\lambda_j$ 
decrease the $j$-th risk, and thus make it more likely that the $j$-th constraint is satisfied. 
In order to satisfy the risk constraints, we are required to make each $\lambda_j$ sufficiently large, given all other values of the threshold.

\textbf{Finite-sample setting.}
Of course, in reality, we do not have access to the entire population of potential datapoints where we want to control the risk. 
Instead, we consider a setting where we can collect a finite \emph{calibration set} of prompts, generated responses, and correct responses, 
which we aim to use to set our thresholds and control the risks:\footnote{As discussed above, if the scores and losses do not depend explicitly on the correct responses, we do not require $Y^{*,(1)}, \ldots, Y^{*,(n)}$.}
\begin{align*}
\mathcal{D} = \{ (X^{(1)}, Y^{(1)}, Y^{*,(1)}), 
\dots, 
(X^{(n)}, Y^{(n)}, Y^{*,(n)}) \}.
\end{align*}
We will also denote the test datapoint $(X,Y,Y^*)$ from \eqref{eq:opt}
by 
$(X^{(n+1)}, Y^{(n+1)}, Y^{*,(n+1)})$.
We want to make a decision about which behavior to execute for the new test datapoint.

\textbf{Goals.}
We aim to use the calibration data to construct thresholds achieving distribution-free risk control
on the constraints in \Cref{eq:opt}, while also efficiently minimizing the objective in \Cref{eq:opt}.

\begin{table}
\centering
\caption{Notation in \Cref{sec:prob-form}.}
\begin{tabular}{ll}
\toprule
\textbf{Symbol} & \textbf{Meaning} \\
\midrule
$\mathcal{X}$, $\mathcal{Y}$ & Sets of possible prompts and responses, respectively \\
$S_j(x, y, y^*)$ & Score function for metric $j$ (e.g., unsafety, uncertainty) \\
$\lambda_j$ & Threshold applied to score $S_j$ \\
$B_j(x, y)$ & Behavior executed when $S_j(x, y, y^*) > \lambda_j$ \\
$V_j(x, y, y^*)$ & Cost incurred under behavior $B_j$ \\
$L_j(\lambda_{1:j})$ & Constraint loss associated with $V_j$, given thresholds $\lambda_{1:j}$ \\
$\E L_j(\lambda_{1:j})$ & Constraint risk associated with $V_j$, given thresholds $\lambda_{1:j}$ \\
$\beta_j$ & Risk budget for constraint $j$ \\
$L_{m+1}(\lambda_{1:m})$ & Objective loss associated with $V_{m+1}$, given thresholds $\lambda_{1:m}$ \\
$m$ & Number of constraint metrics \\
$n$ & Number of calibration observations \\
\bottomrule
\end{tabular}
\label{tab:notation}
\end{table}

\subsection{Two-score example}



To illustrate the framework introduced above, consider a simple setting with two scores
($m = 2$) representing safety and diversity. This example clarifies how sequential constraints
and thresholds interact in practice.

Let $S_1$ denote an \textit{unsafety score}: if $S_1$ is high, the system abstains from responding.
Abstention carries a fixed cost $C_1 > 0$ to the provider, reflecting
the user inconvenience due to not receiving an answer. 
Hence the first constraint-loss function is constant, $V_1(x, y, y^*) = C_1$.
Next, let $S_2$ denote the \textit{negative of a diversity score}, 
so that larger $S_2$ means less
diverse output. 
For instance, when generating multiple images, it can be important to make them highly distinct and dissimilar from each other
in order to maximize utility. 
If $S_2$ exceeds a certain threshold, we ask the model to \textit{resample} a response, emphasizing
diversity (possibly by changing the prompt),
which incurs a constant cost $C_2 > 0$ due to additional computation or delay.
Thus the second loss function is $V_2(x, y, y^*) = C_2$.

The objective loss corresponds to the cost to the provider of generating one response. 
For simplicity, 
we normalize this cost so that the objective loss equals $V_3(x, y, y^*) = 1$.\footnote{In this example, all costs are constant for simplicity. However, the monotonicity properties of the risks generalize to arbitrary cost functions. For instance, generating longer textual answers usually incurs a higher cost, and this can be immediately included in our framework by setting $V_3$ to be the length of $y$.}
Given compact threshold domains $\Lambda_1, \Lambda_2 \subset \R$, the
population-level optimization problem becomes 
to find the thresholds $\lambda_1$, $\lambda_2$ that solve 
\begin{align*}
\min_{\lambda_1 \in \Lambda_1, \lambda_2 \in \Lambda_2}
\quad &\PP[(X,Y,Y^*)]{S_1\le \lambda_1, S_2\le \lambda_2} \\
\text{s.t.} \quad  &C_1 \PP[(X,Y,Y^*)]{S_1 > \lambda_1} \le \beta_1 \\
&C_2 \PP[(X,Y,Y^*)]{S_1\le \lambda_1, S_2 > \lambda_2} \le \beta_2.
\end{align*}
The key property of this optimization problem is its \emph{monotonicity} with respect to $\lambda_1$ and $\lambda_2$. 
Specifically, since
$\lambda_1 \mapsto \PP[(X,Y,Y^*)]{S_1 > \lambda_1}$
is monotone decreasing (technically speaking, non-increasing) in $\lambda_1$, the first constraint imposes a lower bound on $\lambda_1$. 
Next, since the objective is monotone increasing in $\lambda_1$, choosing the smallest feasible value of $\lambda_1$ also optimizes the objective. 
Once this value has been chosen, similar logic can be applied to 
$\lambda_2$. 
This crucial property is at the heart of our entire algorithmic and theoretical development. 
Nonetheless, since we do not have access to the true probabilities and instead only have a finite calibration dataset, 
significant challenges remain, as we detail below. 







\section{Related work}


Here, we list the most closely related works from the literature on risk control with finite-sample theoretical guarantees; for a more extensive literature review, see \Cref{sec:ext-rel-work}.
Early works include \cite{vovk2005algorithmic}, who propose nested prediction sets,
and \cite{angelopoulos2024conformal}, who propose conformal risk control for distribution-free control of a single constraint.
\cite{nguyen2024data} propose 
the restricted risk resampling (RRR) method for asymptotic high-probability control of a risk metric restricted to a localized set of one-dimensional hyperparameters that itself could be estimated (e.g., those achieving small population risk).
\cite{overman2025conformal} propose the distribution-free Conformal Arbitrage (CA) method to balance an objective and a constraint,
which constructs a threshold to decide which of two models to query,
and which controls a measure of disutility to within a specified budget.
In contrast to these works, our goal is to control multiple distinct constraints. 

\cite{laufer2023efficiently} 
develop the Pareto Testing method to 
optimize multiple objectives
and
control multiple risks with 
the Learn Then Test framework \citep{angelopoulos2025learn},
by identifying parameter configurations 
on the Pareto frontier
that satisfy all constraints, with a given probability.
 However, it turns out that due to the monotonicity properties of our constraints, all values of the hyperparameters $\lambda_i$ belong to the Pareto front. 
 Hence, for our particular problem, Pareto Testing is equivalent to Learn Then Test, which is a general method that applicable to any set of constraints by leveraging a multiple testing correction. 
 In contrast, our problem is highly structured due to its monotonicity properties, and hence we can develop algorithms that leverage this structure and avoid becoming overly conservative due to the multiple testing correction. 
 We will provide an experimental comparison highlighting these issues.

\cite{andeol2025conformal}
propose
the distribution-free
Sequential Conformal Risk Control (SeqCRC) method
to control three risks of the form
$L_1(\lambda_1)$, $L_2(\lambda_1, \lambda_2)$, and $L_3(\lambda_1, \lambda_3)$,
where $L_j$ is non-increasing in each argument
for each $j\in [3]$.
Although the construction
of the thresholds
is distantly similar
to our \multirisk~algorithm (\Cref{alg:multirisk}),
\multirisk~applies to risks with a different monotonicity pattern (specifically, the monotonicity of the loss differs in the two thresholds,
 which does not allow for a straightforward reduction to their setting), 
 and is designed to control any number of constraints.


\section{Algorithms}

We first summarize the intuition behind the two algorithms. 
\mrbase~(Algorithm \ref{alg:mrbase}) iteratively selects optimal thresholds that satisfy each constraint, applying dynamic programming using the empirical risks.
Next, \multirisk~(\Cref{alg:multirisk}) refines this procedure to guarantee risk control.
Together, \mrbase~and \multirisk~provide a tradeoff between computational simplicity and statistical guarantees.

\subsection{\Mrbase}

We begin
by presenting
the \mrbase~algorithm (\Cref{alg:mrbase}),
a direct dynamic programming
approach to selecting the thresholds.
The \mrbase~algorithm leverages the monotonicity of the empirical risk functions, adjusting each threshold until the corresponding empirical risk meets its target budget. This provides a simple baseline for sequential multi-risk control but lacks theoretical guarantees.

In order to select the thresholds $\lambda_{1:m}$, we only 
need access to 
the scores and losses
$(S_1^{(i)}, \ldots, S_m^{(i)})$
and
$(V_1^{(i)}, \ldots, V_m^{(i)})$
for $i\in[n]$,
defined as
\begin{align*}
S_j^{(i)} = S_j(X^{(i)}, Y^{(i)}, Y^{*,(i)}), \quad V_j^{(i)} = V_j(X^{(i)}, Y^{(i)}, Y^{*,(i)})
\end{align*}
for $i\in[n]$ and $j\in [m]$.
Given $j\in [m]$,
$\lambda_{1:j}\in \R^j$,
and $i\in [n+1]$,
let the $j$-th loss  evaluated at the $i$-th datapoint $(S_1^{(i)}, \ldots, S_m^{(i)})$ be denoted by
$L_j^{(i)}(\lambda_{1:j})$.
Given $j \in [m]$
and $\lambda_{1:(j-1)} \in \R^{j-1}$,
define the 
\emph{empirical risk function}
$g_j(\cdot; \lambda_{1:(j-1)}) : \R \to [0,\infty)$
corresponding to the $j$-th constraint
by
\begin{align*}
    g_j(\lambda_j; \lambda_{1:(j-1)}) 
    := \frac{1}{n} \sum_{i=1}^{n} L_j^{(i)}(\lambda_{1:j}).
\end{align*}
Note that $g_j(\cdot; \lambda_{1:(j-1)})$ is a non-increasing 
piecewise constant
function 
of $\lambda_j$,
for any fixed $\lambda_{1:(j-1)} \in \R^{j-1}$.
Indeed, recalling the definition from \eqref{lj},
$L_j$ as a function of $\lambda_j$ behaves as the indicator
$c I(S_j > \lambda_j)$ for some $c$ that depends on all other variables, and hence is constant except possibly with a jump at $S_j$.
Thus, $g_j(\cdot; \lambda_{1:(j-1)})$ is constant except possibly with jumps at the $j$-th scores of the datapoints.

Therefore, for each $j\in [m]$,
given $\lambda_{1:(j-1)}\in \R^{j-1}$
and $\beta_j\in \R$,
and given a compact interval $\Lambda_j$,
we may
choose the best possible setting of the hyperparameter $\lambda_j$ that satisfies the constraint for the empirical risk function
via the generalized inverse $U_j(\lambda_{1:(j-1)}; \cdot)$
of $g_j(\cdot;\lambda_{1:(j-1)})$
by
\begin{align}\label{eq:U-j-tilde}
    U_j(\lambda_{1:(j-1)}; \beta_j) = \inf\left\{ \lambda_j\in \Lambda_j : g_j(\lambda_j; \lambda_{1:(j-1)}) \le  \beta_j \right\},
\end{align}
where the infimum of the empty set is taken to be $\lambda_j^{\mathrm{max}}$.\footnote{In practice, if $\Lambda_j$ is set to be a discrete grid, then the infimum can be replaced with a minimum.}

Iterating this construction leads to the
\mrbase~given in 
\Cref{alg:mrbase} below.
The following is a description of 
the inductive step.
Suppose that for some $j\in [m]$, \mrbase~has selected the first $j-1$ thresholds $\hat \lambda_{1:(j-1)}$.
Then since the empirical risk function $g_j(\cdot; \hat \lambda_{1:(j-1)})$ corresponding to the $j$-th constraint is non-increasing, 
in order to satisfy the $j$-th constraint as tightly as possible, 
\mrbase~sets $\hat \lambda_j$ by 
computing the point
$\hat \lambda_j = U_j(\hat \lambda_{1:(j-1)}; \beta_j)$, 
where $g_j(\cdot; \hat \lambda_{1:(j-1)})$ 
crosses $\beta_j$.

\begin{algorithm}[H]
\caption{\Mrbase~algorithm}
\label{alg:mrbase}
\begin{algorithmic}[1]

\State {\bfseries Input:} number of constraints: $m$, number of observations: $n$, 
scores and cost values  $S_j^{(i)} = S_j(X^{(i)}, Y^{(i)}, Y^{*,(i)})$, $V_j^{(i)} = V_j(X^{(i)}, Y^{(i)}, Y^{*,(i)})$ of the datapoints,
risk budgets: $(\beta_j)_{j\in [m]}$, threshold domains: $(\Lambda_j)_{j\in [m]}$

\State Form the empirical risk functions
$g_j(\cdot; \lambda_{1:(j-1)}) : \R \to [0,\infty)$
corresponding to the $j$-th constraint
$g_j(\lambda_j; \lambda_{1:(j-1)})   = \frac{1}{n} \sum_{i=1}^{n} L_j^{(i)}(\lambda_{1:j})$, $j\in [m]$, with loss values $L_j^{(i)}(\lambda_{1:j})$ as per \eqref{lj}.

\State Set the tuning parameter $\hat \lambda_1 \gets U_1(\beta_1)$ \# via generalized inverse \Cref{eq:U-j-tilde} with $\Lambda_1$ and $g_1$

\For{$j = 2$ to $m$}
    \State Set the tuning parameter $\hat \lambda_j \gets U_j(\hat \lambda_{1:(j-1)}; \beta_j)$ \# via generalized inverse \Cref{eq:U-j-tilde} with $\Lambda_j$ and $g_j$
\EndFor

\State \Return $(\hat \lambda_j)_{j\in [m]}$

\end{algorithmic}
\end{algorithm}

This algorithm is valuable for several reasons, including that 
it showcases how the monotonicity structure enables an efficient hyperparameter search in our setting, and that it has a satisfactory performance when the sample size $n$ is large.
However,
a key limitation of \mrbase~is that 
its choice of thresholds is not guaranteed to control the 
average
risks for a new test data point at the nominal levels, especially when the sample size is small.
Indeed, as we show in \Cref{subsec:comparison}, the risks of the \mrbase~thresholds can be arbitrarily large, regardless of the risk budgets.

\subsection{\Multirisk}

Motivated by the need for rigorous constraint risk control,
we propose the \multirisk~algorithm (\Cref{alg:multirisk}).
Like \mrbase, \multirisk~leverages monotonicity to sequentially set the thresholds,
but refines the baseline by replacing empirical risks with conservative upper bounds and by slightly shrinking each risk budget. 
The approach is 
inspired by methods from the distribution-free uncertainty quantification literature, such as
nested prediction sets \citep{vovk2005algorithmic}
and conformal risk control \citep{angelopoulos2024conformal}.


In order to construct an upper bound on the empirical risk function,
since we do not have access to the $(n+1)$-st data point,
we replace its loss value $L_j^{(n+1)}$ by an upper bound.
For $j\in [m]$,
let the constant $V^{\mathrm{max}}_j\ge 0$ 
be an a.s. upper bound
on the cost $V_j^{(n+1)} = V_j(X^{(n+1)}, Y^{(n+1)}, Y^{*,(n+1)})$.
We define the 
\emph{bumped empirical risk function} 
$g_j^+(\cdot; \lambda_{1:(j-1)}) : \R \to [0,\infty)$
corresponding to the $j$-th constraint by
\begin{align*}
    g_j^+(\lambda_j; \lambda_{1:(j-1)}) = \frac{1}{n+1} \sum_{i=1}^{n} L_j^{(i)}(\lambda_{1:j}) + \frac{V^{\mathrm{max}}_j}{n+1}
\end{align*}
for all $\lambda_j\in \R$.
 As in prior work on conformal risk control, the interpretation of this quantity is that it upper bounds the unknown loss value of the test data point by a deterministic value.

Recall that $\Lambda_j$, $j\in[m]$,
are the compact intervals to which we restrict our search. 
As before, 
$g_j^+(\cdot; \lambda_{1:(j-1)})$ is a non-increasing step function for any $\lambda_{1:(j-1)} \in \R^{j-1}$ (\Cref{lem:emp-monot}).
Thus, as above,
for each $j\in [m]$,
given $\lambda_{1:(j-1)}\in \R^{j-1}$
and $\beta_j \in \R$,
we can define
the generalized inverse $U_j^+(\lambda_{1:(j-1)}; \cdot)$
of $g_j^+(\cdot;\lambda_{1:(j-1)})$
as the best possible setting of the hyperparameter $\lambda_j$ that satisfies the constraint for the bumped empirical risk function, namely 
\begin{align}\label{eq:U-j-plus}
    U_j^+(\lambda_{1:(j-1)}; \beta_j) = \inf\left\{ \lambda_j\in \Lambda_j : g_j^+(\lambda_j; \lambda_{1:(j-1)}) \le \beta_j \right\},
\end{align}
where the infimum of the empty set is taken to be $\lambda_j^{\mathrm{max}}$.

Then, analogously to the \mrbase~method, 
\multirisk~(\Cref{alg:multirisk}) proceeds inductively to set the thresholds.\footnote{For a depiction of the dependencies of the thresholds in the case of $m=3$ constraints, see \Cref{fig:multirisk_m3}; for an explicit derivation in the case of $m=2$ constraints, see \Cref{subsubsec:ub-result}.}
Suppose that for some $j\in [m]$, we 
have set the first $j-1$ thresholds $\hat \lambda_{1:(j-1)}$.
In order to set the $j$-th threshold $\hat \lambda_j$,
\multirisk~plugs conservative upper bounds on the previous thresholds
into the bumped empirical risk function
in a manner that guarantees risk control.
Then, since the bumped empirical risk function $g_j^+(\cdot; \tilde \lambda_{1:(j-1)})$ corresponding to the $j$-th constraint is non-increasing, \multirisk~sets $\hat \lambda_j = U_j^+(\tilde \lambda_{1:(j-1)}; \beta_j)$ to be the point at which $g_j^+(\cdot; \tilde \lambda_{1:(j-1)})$ first crosses $\beta_j$.
Note that \multirisk~is nested in the sense that adding a new constraint 
does not change the previously-computed thresholds. 

The use of the upper bounds $\tilde \lambda_{1:(j-1)}$ at step $j$
is crucial to provable risk control, and comes from the guarantees in \Cref{lem:symm-fns} and \Cref{lem:beta-trick} below.
Given thresholds $\lambda_{1:(j-1)}$ that are symmetric functions of all $n+1$ datapoints,
\Cref{lem:symm-fns} allows one to set the $j$-th threshold to a symmetric function that controls the $j$-th risk.
\Cref{lem:beta-trick} allows one to pass from an empirical choice of the $j$-th threshold
to a symmetric function of the data
by simply decreasing the effective risk budget $\beta_j$.
At step $j$,
\multirisk~combines these results as follows.
Recall that the loss $L_j$ is non-decreasing in its first $j-1$ arguments.
Thus,
given the first $j-1$ thresholds $\hat \lambda_{1:(j-1)}$,
in order to ensure that
$\E L_j(\hat \lambda_{1:(j-1)}, \lambda_j) \le \beta_j$,
it suffices to find symmetric upper bounds $\lambda_{1:(j-1)}'$ on $\hat \lambda_{1:(j-1)}$
and some $\lambda_j'\in \R$
such that
$\E L_j(\tilde \lambda_{1:(j-1)}', \lambda_j') \le \beta_j$.
The quantities $\lambda_{1:(j-1)}'$ can be obtained by means of \Cref{lem:beta-trick},
and by \Cref{lem:symm-fns}, we may set
$\lambda_j' = U_j^{\mathrm{sym}}(\lambda_{1:(j-1)}'; \beta_j)$.
Finally, since $L_j$ is non-increasing in its final argument,
if we apply \Cref{lem:beta-trick} to set $\hat \lambda_j$ equal to an empirical upper bound on $\lambda_j'$,
we deduce $\E L_j(\hat \lambda_{1:(j-1)}, \hat \lambda_j) \le \beta_j$, as desired.


Thus, although \multirisk~requires one to compute more auxiliary thresholds compared to the straightforward \mrbase~algorithm,
the above discussion illustrates that the hierarchy of upper bounds and careful budget adjustments incorporated in \multirisk~are essential for risk control,
especially in the small sample size setting.

\begin{algorithm}[H]
\caption{\Multirisk~algorithm to construct risk-controlling thresholds}
\label{alg:multirisk}
\begin{algorithmic}[1]

\State {\bfseries Input:} number of constraints: $m$, number of observations: $n$, cost lower bounds: $(V^{\mathrm{min}}_j)_{j\in [m]}$, cost upper bounds: $(V^{\mathrm{max}}_j)_{j\in [m]}$, risk budgets: $(\beta_j)_{j\in [m]}$, threshold domains: $(\Lambda_j)_{j\in [m]}$

\For{$j = 1$ to $m$}
    \State $\delta_j \gets \frac{V^{\mathrm{max}}_j-V^{\mathrm{min}}_j}{n+1}$
    \For{$k = 0$ to $m - j$}
        \State $\beta_j^{(k)} \gets \beta_j - k \delta_j$
    \EndFor
\EndFor

\For{$k = 1$ to $m$}
    \State $\hat \lambda_1^{(2k)} \gets U_1^+(\beta_1^{(k - 1)})$ \# via \Cref{eq:U-j-plus} with $\Lambda_1$
\EndFor

\For{$j = 2$ to $m$}
    \For{$k = 1$ to $m - j + 1$}
        \State $\hat \lambda_j^{(2k)} \gets U_j^+(\hat \lambda_{1:(j-1)}^{(2k + 2)}; \beta_j^{(k - 1)})$ \# via \Cref{eq:U-j-plus} with $\Lambda_j$
    \EndFor
\EndFor

\State \Return $(\hat \lambda_j^{(2)})_{j\in [m]}$
\end{algorithmic}
\end{algorithm}

\begin{figure}
\centering
\begin{tikzpicture}[
  >={Stealth},
  node style/.style={draw, rounded corners, inner sep=2.5pt, font=\large}
]

\matrix (M) [matrix of nodes,
             row sep=9mm, column sep=12mm,
             nodes={node style}]
{
  $\hat\lambda_{1}^{(2)}$ & $\hat\lambda_{1}^{(4)}$ & $\hat\lambda_{1}^{(6)}$ \\
  $\hat\lambda_{2}^{(2)}$ & $\hat\lambda_{2}^{(4)}$  \\
  $\hat\lambda_{3}^{(2)}$ \\
};

\draw[->] (M-1-2) -- (M-2-1);
\draw[->] (M-1-2) -- (M-3-1);
\draw[->] (M-1-3) -- (M-2-2);
\draw[->] (M-2-2) -- (M-3-1);

\end{tikzpicture}
\caption{Dependency graph of \multirisk~thresholds (\Cref{alg:multirisk}) for three constraints.}
\label{fig:multirisk_m3}
\end{figure}
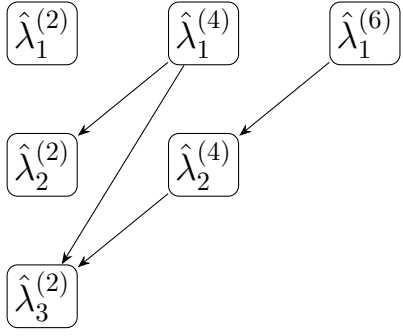

\section{Theoretical guarantees}

In this section, we establish finite-sample guarantees showing that our algorithm tightly controls risk. 
Under appropriate regularity conditions (such as feasibility and boundedness conditions on the loss, detailed in Section \ref{subsec:ub} and later),
\Cref{thm:master} states that the \multirisk~thresholds satisfy each constraint of \Cref{eq:opt}, exhausting all but $O(1/n)$ of each risk budget.
Further, we show that \multirisk~achieves a nearly optimal objective.

The next theorem establishes that our algorithm satisfies all constraints under mild exchangeability and boundedness assumptions.
This property captures the core guarantee provided by the algorithm.
Moreover, under additional continuity assumptions on the
distribution of the scores, it shows that the constraints are satisfied nearly tightly, up to $O(1/n)$ errors.

\begin{theorem}[Risk control guarantee and tightness]\label{thm:master}
Fix $m\ge 1$
and 
let $(\hat \lambda_{1:m})$ be the thresholds 
constructed by \multirisk~(\Cref{alg:multirisk}).
\begin{enumerate}
    \item [1. (Risk control)]
Under the conditions
in \Cref{subsubsec:ub-conds},
the desired constraints 
$\E L_j^{(n+1)}(\hat \lambda_{1:j})$ $ \le \beta_j$ hold
for each $j\in [m]$.
\item [2. (Tightness of risk control)]
Under 
the conditions in
\Cref{subsubsec:ub-conds}
and
\Cref{subsubsec:lb-conds},
we also control the risk close to the desired target level $\beta_j$, in the sense that
\begin{align*}
\E L_j^{(n+1)}(\hat \lambda_{1:j}) \ge \beta_j - \frac{A_j}{n+1} 
\end{align*}
for each $j\in [m]$,
where for $j\in [m]$ we define the scalar
$A_j := 2V^{\mathrm{max}}_j - V^{\mathrm{min}}_j + h_j(2)$,
where $V^{\mathrm{max}}_j$ is an a.s. upper bound on the cost $V_j$
and with $h_j(\cdot)$ defined as in \Cref{eq:h-recursion}. 
\end{enumerate}
\end{theorem}

While the term in the lower bound does not have an explicit form, we can easily compute it numerically using our recursive formula. 
For instance, in a three-constraint setting where the cost functions $V_1, V_2, V_3$ are bounded between $0.5$ and 1, we can calculate that $h_1(2)=0$, $h_2(2)=4$, and $h_3(2)=16$.


Moreover, recalling that our original goal in \eqref{eq:opt} was to minimize a certain objective given the constraints, it is also important to know that our algorithm indeed approximately achieves this minimization goal.
Let $(\lambda_{1:m}^*)$
be any population minimizer of the constrained optimization problem from \eqref{eq:opt}.
This can be viewed as an ideal oracle, because it depends on the true unknown data distribution. 
Our next result shows that our algorithm achieves an objective value that is nearly optimal compared to the oracle.

In order to bound the objective value achieved by our algorithm, 
we require that the population constraint risk functions 
decrease sufficiently quickly, in order for their inverses to be better controlled.
As we explain in more detail later,
we can quantify this via a 
standard notion of H\"older smoothness on the cumulative distribution function of the scores. 
We can obtain an even stronger result for discrete-valued scores.

\begin{theorem}[Tight control of the objective]\label{thm:master-obj}
Under the conditions in
\Cref{subsubsec:ub-conds},
under
\Cref{cond:m-positive},
and
under the conditions in
\Cref{subsubsec:cts-conc-conds},
let $\nu\ge 1$ be 
the reverse H\"older constant of the 
cumulative distribution functions 
in Condition \ref{cond:rev-lip}.
\begin{enumerate}
    \item [1. (Near-optimal objective)]
 The objective value 
$\E V_{m+1}^{(n+1)} I(S_{1:m}^{(n+1)}  \pce \hat \lambda_{1:m})$
 achieved by our algorithm
 is nearly equal to the oracle objective value 
$\E V_{m+1}^{(n+1)} I(S_{1:m}^{(n+1)} $ $ \pce \lambda_{1:m}^{*}) $
for any population minimizer $(\lambda_{1:m}^*)$, in the sense that
 \begin{align*}
    \E V_{m+1}^{(n+1)} I(S_{1:m}^{(n+1)} \pce \hat \lambda_{1:m}) \le \E V_{m+1}^{(n+1)} I(S_{1:m}^{(n+1)} \pce \lambda_{1:m}^{*}) 
    + 
    O\left( 
    \left( \frac{\log n}{n} \right)^{1/(2\nu^{m})} 
    \right).
\end{align*}
\item [1. (Improved optimality for discrete random variables)]
If $S_j$ is a discrete random variable
for each $j\in [m]$,
we have the stronger bound
\begin{align*}
    \E V_{m+1}^{(n+1)} I(S_{1:m}^{(n+1)} \pce \hat \lambda_{1:m}) \le \E V_{m+1}^{(n+1)} I(S_{1:m}^{(n+1)} \pce \lambda_{1:m}^{*}) 
    + O\left( \exp(-n^{0.99}) \right).
\end{align*}
\end{enumerate}

\end{theorem}

\begin{remark}
In the case $m=1$,
\citet[][Theorem 2]{angelopoulos2024conformal}
provides the lower bound
$\E L_1^{(n+1)}(\hat \lambda_1) \ge \beta_1 - \frac{2V^{\mathrm{max}}_1}{n+1}$
under the assumption
that $V^{\mathrm{min}}_1 \ge 0$.
(Recall that for $j\in [m]$,
$V^{\mathrm{min}}_j$ is an a.s. lower bound on the cost $V_j$.)
In comparison,
\Cref{thm:master}
provides the lower bound
$\E L_1^{(n+1)}(\hat \lambda_1) \ge \beta_1 - \frac{2V^{\mathrm{max}}_1-V^{\mathrm{min}}_1}{n+1}$
under the assumption
that $V^{\mathrm{min}}_1 > 0$,
where we used the fact that
$h_1(2)=0$.
However, 
in the special case
$m=1$,
the derivation of
\Cref{eq:m1_zero_lb} in the proof
shows that in fact
the \multirisk~threshold
$\hat \lambda_1$
obeys
$\E L_1^{(n+1)}(\hat \lambda_1) \ge \beta_1 - \frac{2V^{\mathrm{max}}_1}{n+1}$
even when $V^{\mathrm{min}}_1=0$.
Thus, the bound on
\multirisk~is consistent with
\cite[Theorem 2]{angelopoulos2024conformal}.
\end{remark}


\Cref{thm:master,thm:master-obj}
combine the results of \Cref{thm:multiple-scores}, \Cref{cor:constr-lower-bds}, \Cref{thm:cts-conc}, and \Cref{thm:disc-conc}.
We elaborate on these upper bounds, lower bounds, and concentration results in the following sections.

\subsection{Finite-sample upper bounds on constraints}\label{subsec:ub}

Here, we motivate the \multirisk~algorithm and provide intuition for the upper bound result.

\subsubsection{Upper bound conditions}\label{subsubsec:ub-conds}


In order to ensure risk control, 
we need to specify how  we can learn about the test data point from the calibration data. 
For this reason, as is standard in the area of distribution-free predictive inference and conformal prediction \citep{vovk2005algorithmic, angelopoulos2021gentle}, we will consider 
datapoints $\{ (X^{(i)}, Y^{(i)}, Y^{*,(i)}) : i\in [n+1] \}$ that are \emph{exchangeable}. Informally, this means that the data points
are equally likely to be ordered in any particular order.
This includes the common setting where the data points are independent and identically distributed from some fixed, unknown distribution. 

\begin{condition}[Exchangeable observations]\label{cond:exch-observations}
The observations $\{ (X^{(i)}, Y^{(i)}, Y^{*,(i)}) : i\in [n+1] \}$ are exchangeable.
\end{condition}

Next, we restrict attention to finite scores.

\begin{condition}[Finite scores]\label{cond:as-finite}
For $j\in [m]$,
$S_j$ is a.s. finite.
\end{condition}

 The next condition states exactly that the constraints are feasible at the population level. 

\begin{condition}[Feasibility of constraints]\label{cond:loss-inf}
For $j\in [m]$,
for $i\in [n+1]$,
we have
$L_j^{(i)}(\lambda_{1:j}^{\mathrm{max}}) \le \beta_j$ a.s.
\end{condition}

Recall that in our algorithm, we use upper and lower bounds on the cost functions $V_j$. 
The final condition for the upper bound result formalizes this step.

\begin{condition}[Bounds on $V_{1:m}$]\label{cond:loss-bds}
For $j\in [m]$,
we have
$V_j \in [V^{\mathrm{min}}_j, V^{\mathrm{max}}_j]$ a.s.
for some
finite constants
$V^{\mathrm{max}}_j \ge V^{\mathrm{min}}_j \ge 0$.
\end{condition}

\Cref{cond:exch-observations}, \Cref{cond:loss-inf}, and \Cref{cond:loss-bds} are standard in the conformal risk control literature; see for instance \citet[][Theorem 1]{angelopoulos2024conformal}.

\subsubsection{Upper bounding risks for
symmetric functions}

A crucial step in the construction of our algorithm is to define a threshold given all the previous ones. 
In order to motivate and analyze this construction,
it will be useful to consider
an (oracle) version of the empirical risk
that is symmetric in all $n+1$ datapoints. (Below, when we say that a quantity is \textit{symmetric}, we mean that it is a symmetric function of the $n+1$ calibration and test datapoints.)
Given $j \in [m]$
and
$\lambda_{1:(j-1)} \in \R^{j-1}$,
define the 
symmetric oracle empirical risk
function
$g^{\mathrm{sym}}_j(\cdot; \lambda_{1:(j-1)}) : \R \to [0,\infty)$
by
\begin{align}\label{eq:g-j-symm}
    g^{\mathrm{sym}}_j(\lambda_j; \lambda_{1:(j-1)}) = \frac{1}{n+1} \sum_{i=1}^{n+1} L_j^{(i)}(\lambda_{1:j})
\end{align}
for all $\lambda_j \in \R$.
Unlike $g_j$ and $g_j^+$, 
the function $g^{\mathrm{sym}}_j$ is a symmetric version
of the empirical risk corresponding to the $j$-th constraint,
which depends on the unobserved test loss $L_j^{(n+1)}(\lambda_{1:j})$, and it is thus purely a tool that is used in our theoretical analysis.

For each $j\in [m]$,
given $\lambda_{1:(j-1)}\in \R^{j-1}$
and $\beta_j \in \R$,
we define the generalized inverse $U^{\mathrm{sym}}_j$
of $g^{\mathrm{sym}}_j(\cdot; \lambda_{1:(j-1)})$ by
\begin{align}\label{eq:U-j-symm}
    U^{\mathrm{sym}}_j(\lambda_{1:(j-1)}; \beta_j) = \sup\left\{ \lambda_j\in \Lambda_j : g^{\mathrm{sym}}_j(\lambda_j; \lambda_{1:(j-1)}) > \beta_j \right\},
\end{align}
where the supremum of the empty set is taken to be $\lambda_j^{\mathrm{min}}$.
Since $g^{\mathrm{sym}}_j$ is symmetric, it follows that $U^{\mathrm{sym}}_j$ is symmetric.

Next, we will show that when working with the symmetric versions of the losses, the crucial step of constructing the next threshold given all the previous ones can be performed in a convenient manner with the help of the generalized inverse $U^{\mathrm{sym}}_j$.
However, since we do not have access to these symmetric risks, 
we shall introduce intermediate statistics that leverage this construction. 

For now, let us 
study symmetric choices of the first $j-1$
thresholds. 
We will define a choice of such values to be symmetric if each of the $j-1$ thresholds are individually symmetric. 
The following lemma shows 
how 
to set the $j$-th threshold $\lambda_j$ using $ U^{\mathrm{sym}}_j$
to ensure the $j$-th constraint risk is controlled, 
given a symmetric choice of the first $j-1$ thresholds.

\begin{lemma}[Controlling the risk via the symmetrized generalized inverse]\label{lem:symm-fns}
Under
the conditions in
\Cref{subsubsec:ub-conds},
for any $j\ge 1$,
and for any symmetric $\Lambda_{1:(j-1)} = \prod_{\ell=1}^{j-1}[\lambda_{\ell}^{\mathrm{min}}, \lambda_{\ell}^{\mathrm{max}}]$-valued measurable function $\Gamma_{1:(j-1)}$ of the $n+1$ datapoints,
the expected loss $L_j^{(n+1)}$ evaluated on the $(n+1)$-st data point can be  
bounded 
by setting its $j$-th argument as $U^{\mathrm{sym}}_j(\Gamma_{1:(j-1)}; \beta_j)$,
so that 
\begin{align*}
\E L_j^{(n+1)}(\Gamma_{1:(j-1)}, U^{\mathrm{sym}}_j(\Gamma_{1:(j-1)}; \beta_j)) \le \beta_j,
\end{align*}
where $U^{\mathrm{sym}}_j$ is defined in \Cref{eq:U-j-symm}.
\end{lemma}

The proof is given in \Cref{subsec:pf-symm-fns}.
If one had access to all $n+1$ datapoints,
then \Cref{lem:symm-fns} would provide a natural dynamic programming algorithm for achieving risk control:
set $\lambda_1 = U^{\mathrm{sym}}_1(\beta_1)$,
and for $j =2,\ldots,m$,
iteratively set $\lambda_j = U^{\mathrm{sym}}_j(\lambda_{1:(j-1)}; \beta_j)$.
Of course, this cannot be done in practice, because we only have access to the $n$ calibration observations.

In order to obtain a statistic that depends on the $n$ observed datapoints and serves as a valid threshold, 
our strategy below will be to 
instead \emph{bound} appropriate statistics by
symmetric functions.
In order to pass
from symmetric functions
to empirical statistics,
we will make use of the following
inequality chain.


\begin{lemma}[Sandwiching a real inverse between two oracle inverses]\label{lem:beta-trick}
Under the conditions in
\Cref{subsubsec:ub-conds},
a.s.,
for all $j\in [m]$,
for all $\lambda_{1:(j-1)}\in \R^{j-1}$,
for all $\beta\ge 0$,
 the generalized inverse associated with the bumped empirical risk function
$U_j^+$ from 
\Cref{eq:U-j-plus}
 is sandwiched between the generalized inverse associated with the oracle symmetric risk function, as follows: 
\begin{align*}
    U^{\mathrm{sym}}_j(\lambda_{1:(j-1)}; \beta) 
 \le U_j^+(\lambda_{1:(j-1)}; \beta) \le U^{\mathrm{sym}}_j(\lambda_{1:(j-1)}; \beta-\delta_j),
\end{align*}
where for $j\in [m]$
we define
 the normalized range
$\delta_j := \frac{V^{\mathrm{max}}_j-V^{\mathrm{min}}_j}{n+1}$,
where  $V^{\mathrm{max}}_j$ and $V^{\mathrm{min}}_j$ are 
upper and lower bounds on the cost, respectively, defined in \Cref{cond:loss-bds},
and where $U^{\mathrm{sym}}_j$ is
 the symmetric oracle generalized inverse 
defined in \Cref{eq:U-j-symm}.
\end{lemma}
Crucially, this lemma allows us to go back and forth between the
empirical statistics to which we have access and the oracle quantities that allow us to control the risks. 
The proof is given in \Cref{subsec:pf-beta-trick}.

\subsubsection{Upper bound result}\label{subsubsec:ub-result}

Now, we sketch the dynamic programming construction of the \multirisk~thresholds,
which relies on repeated applications of \Cref{lem:symm-fns} and \Cref{lem:beta-trick}.
Here, we illustrate the idea in the cases
$m=1$
and $m=2$ where we have one or two constraints, respectively.

\textbf{Deriving the upper bound for $m=1$.}
The analysis in this case recovers the threshold from the conformal risk control algorithm \citep{angelopoulos2024conformal}, but we find it helpful to present it here in order to explain how our algorithm generalizes these ideas to the setting of multiple risks. 
First, note that
the definition of $L_j(\lambda_{1:j})$ from \eqref{lj}
implies that
$\lambda_{1:j}\mapsto L_j(\lambda_{1:j})$
from $\R^j \to \R$
is non-decreasing in its first $j-1$ arguments
and non-increasing in its last argument.
Also,
the definition of $L_j(\lambda_{1:j})$
implies that
for $j\in [m]$,
a.s.,
for any $\lambda_{1:(j-1)}\in \R^{j-1}$,
the function
$\lambda_j\mapsto L_j(\lambda_{1:j})$
from $\R\to \R$
is right-continuous.

To set the first threshold,
we note that by \Cref{lem:symm-fns},
the symmetric threshold
$\hat \lambda_1^{(1)} = U^{\mathrm{sym}}_1(\beta_1)$
control the risk at the desired level, i.e.,
$\E L_1^{(n+1)}(\hat \lambda_1^{(1)})\le \beta_1$.
Since 
$L_1^{(n+1)}(\lambda_1)=V_1(X^{(n+1)},Y^{(n+1)},Y^{*,(n+1)}) \cdot$
$I(S_1^{(n+1)} > \lambda_1)$
is non-increasing,
it follows that if we define the first \multirisk~threshold to be
the empirical quantity
$\hat \lambda_1^{(2)} = U_1^+(\beta)$---recovering the construction from conformal risk control \citep{angelopoulos2024conformal}---then since $\hat \lambda_1^{(2)} \ge \hat \lambda_1^{(1)}$ by \Cref{lem:beta-trick},
we have
$\E L_1^{(n+1)}(\hat \lambda_1^{(2)})\le \beta_1$,
as desired.
This completes the construction of the first threshold. 
 
\textbf{Deriving the upper bound for $m=2$.}
As mentioned above, the second threshold is set after selecting the first threshold. 
In order to set the second threshold, we seek $\lambda_2$ such that $\E L_2^{(n+1)}(\hat \lambda_1^{(2)}, \lambda_2) \le \beta_2$.
As discussed above, we are only able to control expected loss values for symmetrized arguments. Therefore, we need to 
construct a symmetric upper bound $\hat \lambda_1^{(3)}$ on 
the already constructed empirical threshold $\hat \lambda_1^{(2)}$.
To do so, we note that by \Cref{lem:beta-trick}, we have the symmetric upper bound $\hat \lambda_1^{(2)} = U_1^+(\beta_1) \le U^{\mathrm{sym}}_1(\beta_1^{(1)})$, hence we may set $\hat \lambda_1^{(3)} = U^{\mathrm{sym}}_1(\beta_1^{(1)})$. 

Next, we need to switch from the infeasible oracle threshold to an empirical quantity.
Therefore, we construct an empirical upper bound $\hat \lambda_1^{(4)}$ on $\hat \lambda_1^{(3)}$.
By another application of \Cref{lem:beta-trick}, we have the empirical upper bound $\hat \lambda_1^{(4)} = U^{\mathrm{sym}}_1(\beta_1^{(1)}) \le U_1^+(\beta_1^{(1)})$, hence we may set $\hat \lambda_1^{(4)} = U_1^+(\beta_1^{(1)})$, 

Now, since $\hat \lambda_1^{(3)}$ is symmetric,
\Cref{lem:symm-fns} implies that $\E L_2^{(n+1)}(\hat \lambda_1^{(3)}, U^{\mathrm{sym}}_2(\hat \lambda_1^{(3)}; \beta_2)) \le \beta_2$.
Since $U^{\mathrm{sym}}_2$ is non-decreasing in its first argument (by \Cref{lem:emp-monot} in the appendix), 
we have that $U^{\mathrm{sym}}_2(\hat \lambda_1^{(3)}; \beta_2)$ is bounded above by the empirical quantity $U_2^+(\hat \lambda_1^{(4)}; \beta_2)$.

Thus, since 
$$L_2^{(n+1)}(\lambda_1,\lambda_2)=V_2(X^{(n+1)},Y^{(n+1)},Y^{*,(n+1)}) 
I(S_1^{(n+1)}\le \lambda_1, 
S_2^{(n+1)} > \lambda_2)$$ 
is non-decreasing in its first argument and non-increasing in its second argument, we deduce
$\E L_2^{(n+1)}(\hat \lambda_1^{(2)},  U_2^+(\hat \lambda_1^{(4)}; \beta_2)) \le \beta_2$.
Therefore, 
if we define the second \multirisk~threshold to be $\hat \lambda_2^{(2)} = U_2^+(\hat \lambda_1^{(4)}; \beta_2)$, then
 it controls the expected loss
$\E L_2^{(n+1)}(\hat \lambda_1^{(2)}, \hat \lambda_2^{(2)}) \le \beta_2$, as desired.
This finishes the construction of the second threshold.

\textbf{General result.}
In general, to construct the \multirisk~threshold $\hat \lambda_{m}^{(2)}$ for arbitrary $m$, we must construct $2(m-j)$ auxiliary thresholds for each constraint $j\in [m-1]$. 
The resulting general risk control result, along with the required properties of the intermediate constructions,
are formalized in \Cref{thm:multiple-scores} below.

\begin{theorem}[Upper bound on \multirisk~constraint risks]\label{thm:multiple-scores}
Under the conditions
in 
\Cref{subsubsec:ub-conds},
fix $m\ge 1$.
Given $j\in [m]$
and $k\ge 0$,
 defined the tightened thresholds
$\beta_j^{(k)} = \beta_j -k\delta_j$,
where $\delta_j$ is
 the normalized range
defined in \Cref{lem:beta-trick}.
Given $j\in [m]$,
define\footnote{To be clear, for $j=1$, we define 
$\hat \lambda_1^{(2k-1)} = U^{\mathrm{sym}}_1(\beta_1^{(k-1)})$
and
$\hat \lambda_1^{(2k)} = U_1^+(\beta_1^{(k-1)})$ for $k\in [m]$.}
iteratively for $k=1, \ldots, m-j+1$,
$\hat \lambda_j^{(2k-1)} = U^{\mathrm{sym}}_j(\hat \lambda_{1:(j-1)}^{(2k+1)}; \beta_j^{(k-1)})
\in \Lambda_j$
and
$\hat \lambda_j^{(2k)} = U_j^+(\hat \lambda_{1:(j-1)}^{(2k+2)}; \beta_j^{(k-1)})
\in \Lambda_j$
where $U^{\mathrm{sym}}_j$ 
symmetric generalized inverse
defined in \Cref{eq:U-j-symm}.
Then:
\begin{enumerate}
    \item the threshold $\hat \lambda_j^{(2k-1)}$ is a symmetric function of the $n+1$ datapoints
and
$\hat \lambda_j^{(2k)}$ is a function of the observed data
for $j\in [m]$ and $k\in [m-j+1]$;
    \item we have the inequality chain
$\hat \lambda_j^{(1)} \le \hat \lambda_j^{(2)} \le \ldots \le \hat \lambda_j^{(2(m-j+1))}$ for $j\in [m]$ a.s.; and

    \item the \multirisk~thresholds $(\hat \lambda_{1:m}^{(2)})$ obey 
    the desired risk constraints
    $\E L_j^{(n+1)}(\hat \lambda_{1:j}^{(2)}) \le \beta_j$ for $j\in [m]$.
\end{enumerate}

\end{theorem}

The proof is given in \Cref{subsec:pf-multiple-scores}.

\subsection{Finite-sample lower bounds on constraints}\label{subsec:lb}

Here, we present conditions under which the \multirisk~thresholds achieve near-tight control of the constraint risks.

\subsubsection{Lower bound conditions}\label{subsubsec:lb-conds}

In order to ensure that the constraint risk functions a.s. exhibit small discontinuities, we assume that the scores are continuously distributed random variables. 
Such conditions are standard in the area when establishing lower bounds,
see for instance
\cite{vovk2005algorithmic,angelopoulos2021gentle} and
\cite[Theorem 2]{angelopoulos2024conformal}.

\begin{condition}[Continuous scores]\label{cond:cts-scores}
For $j\in [m]$,
$S_j$ is a continuous random variable.
\end{condition}

Next, we assume the costs $V_{1:m}$ are strictly positive, which is crucial to a counting argument in our proof of \Cref{cor:constr-lower-bds}.
This is a mild condition that holds in many cases of interest. 
For instance, for the case of conformal prediction, 
as explained above, 
the coverage error corresponds to a constant cost $V_1=1$; and our condition clearly holds. 
More generally, if some costs $V_j$ can be equal to zero, we can run our algorithm with a shifted cost $\tilde V_j = V_j+\epsilon$ that adds some small $\epsilon > 0$ to the cost.
Then it will follow from our result that the original
risk (with the cost $V_j$)
is upper bounded by the desired level $\beta_j$
and lower bounded by $\beta_j - \epsilon$,
up to the associated slack in the lower bound of \Cref{thm:master-obj}.


\begin{condition}[$V_{1:m}$ positive]\label{cond:m-positive}
We have $V_j \ge V^{\mathrm{min}}_j > 0$ a.s. for each $j\in [m]$.
\end{condition}


Next, we need to impose a condition on the loss functions relative to the risk budget. 
To see why an additional condition might be needed for tightness, observe that our conditions so far allow for a scenario where,
for any fixed $\lambda_{1:(j-1)}$,
even the worst-case
setting of $\lambda_j$, namely $\lambda_{j}^{\mathrm{min}}$, 
loosely satisfies the risk control property, in the sense that 
$L_j^{(i)}(\lambda_{1:(j-1)},\lambda_{j}^{\mathrm{min}}) < \beta_j$.
If this happens, then in general, risk control may only be achieved loosely, and there is no hope to ensure tight risk control. 
Therefore, it is clear that an additional condition is needed. 

For our algorithm, it turns out that due to technical reasons, the appropriate condition to impose is that 
within the parameter region $\Lambda_{1:m}$, the loss $L_j^{(i)}(\cdot)$ a.s. exceeds the threshold $\beta_j$. This condition is used only for the lower bound in this section.
It amounts to setting the budgets $\beta_{1:m}$ sufficiently small and the left endpoints $\lambda_{1:m}^{\mathrm{min}}$ sufficiently large.

\begin{condition}[Lower bound on maximum loss in $\Lambda_{1:j}$]\label{cond:loss-sup}
For $j\in [m]$,
for $i\in [n+1]$,
we have
$L_j^{(i)}(\lambda_{1:j}^{\mathrm{min}}) > \beta_j$ a.s.

\end{condition}

Finally, we assume a strengthened form of
\Cref{cond:loss-inf} (feasibility)
that is needed due to the recursive nature of our lower bound argument.
Specifically, 
in the inductive step of the proof of \Cref{thm:reverse-ineq}, for each $\ell\in [j-1]$, we must count the number of jumps in the step function $g^{\mathrm{sym}}_{\ell}(\cdot; \hat \lambda_{1:(\ell-1)}^{(2s+3)})$ between the heights $\beta_{\ell}^{(s)}$ and $\beta_{\ell}^{(k)} - \frac{h_{\ell}(k-s)}{n+1}$,
where $h_{\ell}(\cdot)$ is defined in \Cref{thm:reverse-ineq}.
In order to apply a specific counting lemma (\Cref{lem:count-scores}), we must ensure that the heights lie within the range of our risk function $L_j^{(i)}(\cdot)$, which 
can be ensured given the strengthened feasibility condition below.
 To be clear, this condition ensures that the problem is feasible even for slightly decreased risk budgets. 

\begin{condition}[Strong feasibility]\label{cond:loss-inf-strong}
For $j\in [m-1]$, and 
for $i\in [n+1]$,
we have
$L_j^{(i)}(\lambda_{1:j}^{\mathrm{max}}) \le 
\beta_j^{(m+2-j)} - \frac{h_j(2)}{n+1}$
a.s., where $h_j(\cdot)$ is defined in \Cref{eq:h-recursion}.
\end{condition}

\subsubsection{Lower bounding risks for symmetric functions}

Next, we turn to presenting our lower bounds. 
We begin with a generalization of
\citet[Lemma 1]{angelopoulos2024conformal}
that bounds jumps in the step function $g^{\mathrm{sym}}_j$.
Specifically, recall that $U^{\mathrm{sym}}_j(\Gamma_{1:(j-1)}; \beta)$ is defined as the generalized inverse of $g_j^{\mathrm{sym}}(\cdot; \Gamma_{1:(j-1)})$.
If this function were continuous at $\beta$, then we would have that
$g^{\mathrm{sym}}_j(U^{\mathrm{sym}}_j(\Gamma_{1:(j-1)}; \beta); \Gamma_{1:(j-1)}) = \beta$.
However, since this is a step function, 
the presence of jumps may instead yield 
$g^{\mathrm{sym}}_j(U^{\mathrm{sym}}_j(\Gamma_{1:(j-1)}; \beta); \Gamma_{1:(j-1)}) < \beta$.
Our result bounds the size of the jump by lower bounding
$g^{\mathrm{sym}}_j(U^{\mathrm{sym}}_j(\Gamma_{1:(j-1)}; \beta); \Gamma_{1:(j-1)})$.

\begin{lemma}[Bounded jumps]\label{lem:lower-bd}
Assume that the conditions
in 
\Cref{subsubsec:ub-conds}
and
\Cref{cond:cts-scores}
hold.
Fix $j\in [m]$.
Suppose $\beta\ge 0$
is such that
for $i\in [n+1]$,
we have
$L_j^{(i)}(\lambda_{1:j}^{\mathrm{min}}) > \beta$ a.s.
Then for any $\Lambda_{1:(j-1)}$-valued symmetric functions $\Gamma_{1:(j-1)}$ of the $n+1$ datapoints, we have
\begin{align}\label{eq:lower-bd}
g^{\mathrm{sym}}_j(U^{\mathrm{sym}}_j(\Gamma_{1:(j-1)}; \beta); \Gamma_{1:(j-1)}) \ge \beta - \frac{V^{\mathrm{max}}_j}{n+1},
\end{align}
where $V^{\mathrm{max}}_j$ is the upper bound defined in \Cref{cond:loss-bds},
$g^{\mathrm{sym}}_j$ is the symmetric risk defined in \Cref{eq:g-j-symm},
and $U^{\mathrm{sym}}_j$ is the associated symmetric generalized inverse function defined in \Cref{eq:U-j-symm}.
Consequently, 
the risk of the $n+1$-st data point is tightly lower bounded around the target value $\beta$
as
\begin{align}\label{eq:lower-bd-expec}
    \E L_j^{(n+1)}( \Gamma_{1:(j-1)}, U^{\mathrm{sym}}_j(\Gamma_{1:(j-1)}; \beta) ) \ge \beta - \frac{V^{\mathrm{max}}_j}{n+1}.
\end{align}
\end{lemma}

The proof is given in \Cref{subsec:pf-lower-bd}.

\subsubsection{Lower bound result}

Now, we 
provide the arguments for
deriving lower bounds on the constraints
evaluated at 
the \multirisk~thresholds
$(\hat \lambda_{1:m}^{(2)})$
by passing to symmetric lower bounds and using \Cref{lem:lower-bd}.
The result is given in
\Cref{cor:constr-lower-bds}
below.
Here, we sketch the derivation in the cases $m=1$ and $m=2$.

\textbf{Deriving the lower bound for $m=1$.}
Consider lower bounding
the first constraint risk in the case $j=1$.
 This matches the setting of conformal risk control \citep{angelopoulos2024conformal}, 
 and our analysis and result are essentially equivalent. 
By \Cref{lem:symm-fns}, we have $\E L_1^{(n+1)}( U^{\mathrm{sym}}_1(\beta_1) ) \le \beta_1$,
which by \Cref{lem:beta-trick} and the fact that $L_1^{(n+1)}$ is non-increasing implies that $\E L_1^{(n+1)}( U_1^+(\beta_1) ) \le \beta_1$.
Since
by \Cref{lem:beta-trick}
we have
$\E L_1^{(n+1)}( U_1^+(\beta_1) ) \ge \E L_1^{(n+1)}( U^{\mathrm{sym}}_1(\beta_1^{(1)}) )$,
and since $U^{\mathrm{sym}}_1(\beta_1^{(1)})$ is symmetric,
\Cref{lem:lower-bd} (bounded jumps) implies
\begin{align}\label{eq:m1_zero_lb}
\E L_1^{(n+1)}( U_1^+(\beta_1) ) 
\ge \E L_1^{(n+1)}( U^{\mathrm{sym}}_1(\beta_1^{(1)}) ) 
\ge \beta_1^{(1)} - \frac{V^{\mathrm{max}}_1}{n+1}
= \beta_1 - \frac{2V^{\mathrm{max}}_1-V^{\mathrm{min}}_1}{n+1},
\end{align}
which gives us the desired risk lower bound.

\textbf{Deriving the lower bound for $m=2$.}
The $j=2$ case is more tricky, since $L_2^{(n+1)}$ has two arguments
with different monotonicities.
By \Cref{lem:symm-fns}, we have the initial bound
\begin{align*}
    \E L_2^{(n+1)}(U^{\mathrm{sym}}_1(\beta_1^{(1)}), U^{\mathrm{sym}}_2(U^{\mathrm{sym}}_1(\beta_1^{(1)}); \beta_2)) \le \beta_2.
\end{align*}
By \Cref{lem:beta-trick}, the fact that $L_2^{(n+1)}$ is non-decreasing in its first argument and non-increasing in its second argument, and the fact that $U_2^+$ is non-decreasing in its first argument (\Cref{lem:emp-monot}), this implies
\begin{align*}
    \E L_2^{(n+1)}( U_1^+(\beta_1), U_2^+(U_1^+(\beta_1^{(1)}); \beta_2) ) \le \beta_2.
\end{align*}
Thus, our choice of \multirisk~thresholds $(\hat \lambda_1^{(2)}, \hat \lambda_2^{(2)})$ controls the second constraint risk.
By another application of \Cref{lem:beta-trick},
and again using the monotonicity properties of $L_2^{(n+1)}$, 
we obtain the lower bound
\begin{align*}
    \E L_2^{(n+1)}( U^{\mathrm{sym}}_1(\beta_1), U^{\mathrm{sym}}_2( U^{\mathrm{sym}}_1(\beta_1^{(2)}) ; \beta_2^{(1)}) ),
\end{align*}
which is symmetric.

At this stage, we would like
to cite \Cref{lem:lower-bd}
to lower bound this expression by
$\beta_2^{(1)} - \frac{V^{\mathrm{max}}_2}{n+1}$.
However, we cannot, because this expression is not of the form
$\E L_2^{(n+1)}( \Gamma_1, U^{\mathrm{sym}}_2(\Gamma_1; \tilde \beta) )$
for some symmetric $\Gamma_1$ and some $\tilde \beta \in \R$.
In order to compare this expression
to a quantity of the form $\E L_2^{(n+1)}(\Gamma_1, U^{\mathrm{sym}}_2(\Gamma_1; \tilde \beta))$,
we need to reverse an inequality.
Specifically,
if we could find $\tilde \beta\in \R$ such that
$U^{\mathrm{sym}}_2(U^{\mathrm{sym}}_1(\beta_1^{(2)}); \beta_2^{(1)}) \le U^{\mathrm{sym}}_2( U^{\mathrm{sym}}_1(\beta_1); \tilde \beta)$,
then we would obtain the lower bound $\tilde \beta - \frac{V^{\mathrm{max}}_2}{n+1}$.

In general, given $\lambda_1, \tilde \lambda_1\in \Lambda_1$ with $\tilde \lambda_1 \le \lambda_1$
and $\beta\in \R$,
we must determine $\tilde \beta\le \beta$ to ensure
$U^{\mathrm{sym}}_2(\lambda_1; \beta) \le U^{\mathrm{sym}}_2(\tilde \lambda_1; \tilde \beta)$.
By the definition of $U^{\mathrm{sym}}_2$ and \Cref{lem:crossing-points},
it suffices to bound the difference between the step functions
\begin{align*}
    g^{\mathrm{sym}}_2(\lambda_2; \lambda_1) = \frac{1}{n+1} \sum_{i=1}^{n+1} L_2^{(i)}(\lambda_1, \lambda_2)
\end{align*}
and
\begin{align*}
    g^{\mathrm{sym}}_2(\lambda_2; \tilde \lambda_1) = \frac{1}{n+1} \sum_{i=1}^{n+1} L_2^{(i)}(\tilde \lambda_1, \lambda_2).
\end{align*}
Note that by \Cref{lem:emp-monot},
$g^{\mathrm{sym}}_2(\cdot; \lambda_1) \ge g^{\mathrm{sym}}_2(\cdot; \tilde \lambda_1)$.
Thus,
by the definition of $L_2^{(i)}$,
the difference
$g^{\mathrm{sym}}_2(\lambda_2; \lambda_1) - g^{\mathrm{sym}}_2(\lambda_2; \tilde \lambda_1)$
equals the non-negative function
\begin{align*}
    \Delta(\lambda_2) = \frac{1}{n+1} \sum_{i \in \ii} V_2^{(i)} I(S_2^{(i)} > \lambda_2),
\end{align*}
where the index set $\ii$ is defined as
$\ii = \{ i \in [n+1] : S_1^{(i)} \in (\tilde \lambda_1, \lambda_1] \}$.

Since $|V_2^{(i)}|\le V^{\mathrm{max}}_2$ a.s. by \Cref{cond:loss-bds}, we have the uniform bound $\|\Delta\|_{L^{\infty}(\R)} \le \frac{1}{n+1} V^{\mathrm{max}}_2 |\ii|$.
If we could establish the bound
$\frac{1}{n+1} V^{\mathrm{max}}_2 |\ii| \le C$
for some nonrandom scalar $C$,
then by 
\Cref{eq:one-sided-crossing}
from
\Cref{lem:crossing-points} (which shows that bounded perturbations of monotone functions lead to bounded generalized inverses),
we would obtain
$U^{\mathrm{sym}}_2(\lambda_1; \beta) \le U^{\mathrm{sym}}_2(\tilde \lambda_1; \tilde \beta)$ with
$\tilde \beta = \beta - C$,
as desired.

{\bf Bounding the number of scores in an interval via the number of jumps of the empirical risk.}
Therefore, it suffices to bound $|\ii|$. In order to do so, we leverage the specific forms of $\lambda_1$ and $\tilde \lambda_1$, namely
$\lambda_1 = U^{\mathrm{sym}}_1(\beta_1^{(2)})$
and
$\tilde \lambda_1 = U^{\mathrm{sym}}_1(\beta_1)$.
We claim that the number of indices $i\in [n+1]$ satisfying $U^{\mathrm{sym}}_1(\beta_1) < S_1^{(i)} \le U^{\mathrm{sym}}_1(\beta_1^{(2)})$
is bounded by
\begin{align*}
\frac{\beta_1 - \beta_1^{(2)} + \frac{V^{\mathrm{max}}_1}{n+1}}{V^{\mathrm{min}}_1/(n+1)}.
\end{align*}
To see this, note that by the definition of $L_1^{(i)}(\lambda_1)$,
each such index $i\in [n+1]$
corresponds to a downwards jump
in the function $g^{\mathrm{sym}}_1(\lambda_1) = \frac{1}{n+1} \sum_{i=1}^{n+1} L_1^{(i)}(\lambda_1)$ of at least $\frac{V^{\mathrm{min}}_1}{n+1}$.
Further, by the definition of $U^{\mathrm{sym}}_1$, between $U^{\mathrm{sym}}_1(\beta_1)$ and $U^{\mathrm{sym}}_1(\beta_1^{(2)})$
the value of $g^{\mathrm{sym}}_1$ changes by at most
$\beta_1 - \beta_1^{(2)} + \frac{V^{\mathrm{max}}_1}{n+1}$.
Thus,
if $V^{\mathrm{min}}_1 > 0$,
we may bound the number of downwards jumps in $g^{\mathrm{sym}}_1$ by
\begin{align*}
|\ii| \le \frac{\beta_1 - \beta_1^{(2)} + \frac{V^{\mathrm{max}}_1}{n+1}}{V^{\mathrm{min}}_1/(n+1)}.
\end{align*}
Plugging this into our bound on $\|\Delta\|_{L^{\infty}(\R)}$, we find
\begin{align*}
\|\Delta\|_{L^{\infty}(\R)} \le \frac{1}{n+1} V^{\mathrm{max}}_2 \frac{\beta_1 - \beta_1^{(2)} + \frac{V^{\mathrm{max}}_1}{n+1}}{V^{\mathrm{min}}_1/(n+1)} = V^{\mathrm{max}}_2 \frac{\beta_1 - \beta_1^{(2)} + \frac{V^{\mathrm{max}}_1}{n+1}}{V^{\mathrm{min}}_1},
\end{align*}
so that we may set
\begin{align*}
\tilde \beta = \beta_2^{(1)} - V^{\mathrm{max}}_2 \frac{\beta_1 - \beta_1^{(2)} + \frac{V^{\mathrm{max}}_1}{n+1}}{V^{\mathrm{min}}_1}
\end{align*}
to ensure that
$U^{\mathrm{sym}}_2(U^{\mathrm{sym}}_1(\beta_1^{(2)}); \beta_2^{(1)}) \le U^{\mathrm{sym}}_2( U^{\mathrm{sym}}_1(\beta_1); \tilde \beta)$.
The final risk lower bound reads 
\begin{align*}
&\E L_2^{(n+1)}( U_1^+(\beta_1), U_2^+(U_1^+(\beta_1^{(1)}); \beta_2) )
\ge \E L_2^{(n+1)}( U^{\mathrm{sym}}_1(\beta_1), U^{\mathrm{sym}}_2( U^{\mathrm{sym}}_1(\beta_1) ; \tilde \beta ) ) \\
&\ge \tilde \beta - \frac{V^{\mathrm{max}}_2}{n+1}
= \beta_2^{(1)} - V^{\mathrm{max}}_2 \frac{\beta_1 - \beta_1^{(2)} + \frac{V^{\mathrm{max}}_1}{n+1}}{V^{\mathrm{min}}_1} - \frac{V^{\mathrm{max}}_2}{n+1},
\end{align*}
as desired.

\textbf{General lower bound.} This argument generalizes to all $m\ge 1$. In \Cref{thm:reverse-ineq}, we present a generalization of the reversed inequality, and in \Cref{cor:constr-lower-bds}, we apply this to derive a lower bound on the constraint risks of the \multirisk~thresholds.
 In order to state this result, we need to define a quantity that captures by how much the above argument iteratively decreases the risk budgets in aggregate. 
For this purpose, 
for all integers $t\ge 0$,
let $h_1(t) = 0$.
Sequentially for $j=2,\ldots,m$
define the non-negative quantities
\begin{align}\label{eq:h-recursion}
    h_j(t) = V^{\mathrm{max}}_j \sum_{\ell=1}^{j-1} \frac{t (V^{\mathrm{max}}_{\ell} - V^{\mathrm{min}}_{\ell}) + V^{\mathrm{max}}_{\ell} + h_{\ell}(t)}{V^{\mathrm{min}}_{\ell}},
\end{align}
where $V^{\mathrm{max}}_{\ell}$ and $V^{\mathrm{min}}_{\ell}$ are defined in \Cref{cond:loss-bds}.
Then we have the following statement about the values of the symmetric generalized inverse function, illustrated in \Cref{fig:reverse-ineq}: 
\begin{theorem}[Reversed inequality for $U^{\mathrm{sym}}_j$]\label{thm:reverse-ineq}
Assume that the conditions
in 
\Cref{subsubsec:ub-conds},
\Cref{cond:cts-scores},
and
\Cref{cond:m-positive}
hold.
Fix $j\in [m]$
and integers $k\ge s\ge 0$.
Suppose that
for $\ell\in [j-1]$,
for $i\in [n+1]$,
we have the bounds 
$L_{\ell}^{(i)}(\lambda_{1:\ell}^{\mathrm{min}}) > \beta_{\ell}^{(s)}$
a.s.
and
$L_{\ell}^{(i)}(\lambda_{1:\ell}^{\mathrm{max}}) \le \beta_{\ell}^{(k+j-\ell)} - \frac{h_{\ell}(k-s)}{n+1}$
a.s.
Then
for each $\beta\ge 0$,
a.s., we have the bound
\begin{align*}
    U^{\mathrm{sym}}_j(\hat \lambda_{1:(j-1)}^{(2k+1)}; \beta)
\le U^{\mathrm{sym}}_j\left(\hat \lambda_{1:(j-1)}^{(2s+1)}; \beta - \frac{h_j(k-s)}{n+1}\right),
\end{align*}
where $U^{\mathrm{sym}}_j$ is defined in \Cref{eq:U-j-symm}.
\end{theorem}

\begin{figure}[H]
    \centering
    \includegraphics[scale=0.33]{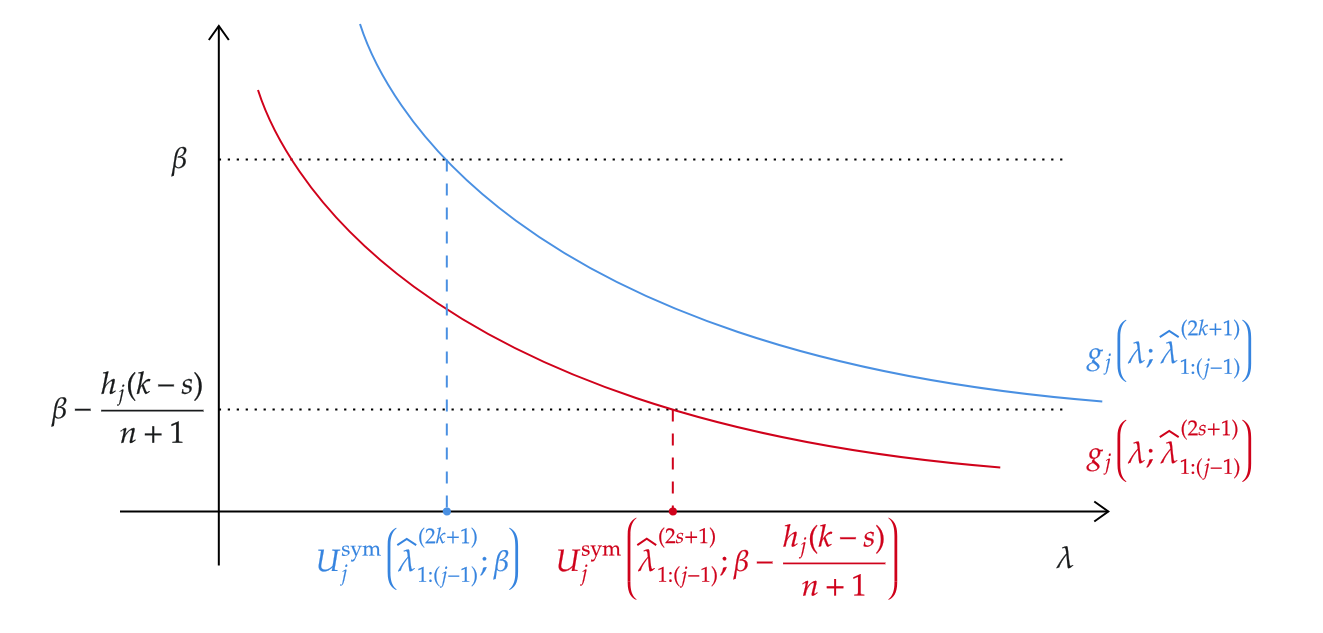}
    \caption{Illustration of the bound in \Cref{thm:reverse-ineq}. } 
    \label{fig:reverse-ineq}
\end{figure}

Observe that due to the monotonicity properties discussed above, we have 
the inequality
$U^{\mathrm{sym}}_j(\hat \lambda_{1:(j-1)}^{(2k+1)}; \beta)
\ge U^{\mathrm{sym}}_j\left(\hat \lambda_{1:(j-1)}^{(2s+1)}; \beta\right)$.
The above result shows that if we decrease $\beta$ by a small amount, we can also obtain an inequality in the reverse direction. 
The proof is given in \Cref{subsec:pf-reverse-ineq}.
Our next results state the implication of this theorem for 
tightness of the risk control of the \multirisk~algorithm. 

\begin{corollary}[Lower bound on \multirisk~constraint risks]\label{cor:constr-lower-bds}

Assume that the conditions
in \Cref{subsubsec:ub-conds}
and in \Cref{subsubsec:lb-conds} hold.
Let $(\hat \lambda_{1:m}^{(2)})$ denote 
the \multirisk~thresholds
constructed
in \Cref{thm:multiple-scores}.
Then for each $j\in [m]$,
we have the risk lower bound
\begin{align*}
\E L_j^{(n+1)}(\hat \lambda_{1:j}^{(2)}) \ge \beta_j - \frac{2V^{\mathrm{max}}_j - V^{\mathrm{min}}_j + h_j(2)}{n+1},
\end{align*}
where $V^{\mathrm{max}}_j$ and $V^{\mathrm{min}}_j$ are defined as in \Cref{cond:loss-bds},
and where the functions $h_j(\cdot)$
are defined as in \Cref{eq:h-recursion}.

\end{corollary}

The proof is given in \Cref{subsec:pf-constr-lower-bds}.


\subsection{Near-optimality of the objective}\label{subsec:cts-conc}

Finally, we show that
under regularity conditions,
the \multirisk~thresholds
not only satisfy the constraints,
but also approximately minimize the population-level objective in \Cref{eq:opt}.

First, note that the population-level objective given in \Cref{eq:opt} has a solution
with an iterative structure.
Given $j \in [m]$
and $\lambda_{1:(j-1)} \in \R^{j-1}$,
define the population risk function
$g_j^*(\cdot; \lambda_{1:(j-1)}) : \R \to [0,\infty)$
corresponding to the $j$-th constraint
by
\begin{align}\label{gjs}
    g_j^*(\lambda_j; \lambda_{1:(j-1)}) = \E L_j(\lambda_{1:j})
\end{align}
for all $\lambda_j \in \R$.\footnote{When $j=1$, we define $g_1^*(\lambda_1) = \E L_1(\lambda_1)$ for $\lambda_1\in \R$.}
Also, given $\beta_j\ge 0$, define the population-level generalized inverse
\begin{align*}
    U_j^*(\lambda_{1:(j-1)}; \beta_j) = \max\{ \lambda_j\in \Lambda_j : g_j^*(\lambda_j; \lambda_{1:(j-1)}) > \beta_j \},
\end{align*}
where the supremum of the empty set is taken to be $\lambda_j^{\mathrm{min}}$.
 Due to the right continuity of $g_j^*$ that
 we establish in the 
 Appendix, 
 if $g_j^*(\lambda_j^{\mathrm{max}}; \lambda_{1:(j-1)}) \le \beta_j$,
 it follows that 
$g_j^*(U_j^*(\lambda_{1:(j-1)}; \beta_j); \lambda_{1:(j-1)}) \le \beta_j$.
 
Let $\lambda_1^* = U_1^*(\beta_1)$,
and for $j\in \{2,\ldots,m\}$,
define $\lambda_j^*$ sequentially by
\begin{align*}
\lambda_j^* = U_j^*(\lambda_{1:(j-1)}^*; \beta_j).
\end{align*}
Then, we show in the appendix  
(\Cref{lem:pop-minimizer}) that  
in the non-trivial case where the risk $g_j^*(\cdot; \lambda_{1:(j-1)}^*)$ crosses 
$\beta_j$---which turns out to amount to 
 $\lambda_j^{*}\in 
(\lambda_j^{\mathrm{min}}, \lambda_j^{\mathrm{max}})$---for all $j\in [m]$, 
$\lambda_{1:m}^*$ is a minimizer of the objective of \Cref{eq:opt}.

\subsubsection{Concentration conditions}\label{subsubsec:cts-conc-conds}

We define
$\mathcal{I}_{\mathrm{discrete}} := \{ j\in [m] : S_j \text{ is a discrete random variable} \}$.
We impose different conditions
on the discrete
$S_j$ for $j\in \mathcal{I}_{\mathrm{discrete}}$
and
on the continuous
$S_j$ for $j\in \mathcal{I}_{\mathrm{cts}}:=  [m]\setminus \mathcal{I}_{\mathrm{discrete}}$.

First, in order to study the concentration of the objective, we introduce bounds on the objective cost.
 These are the analogs of \Cref{cond:loss-bds} for the objective cost instead of the constraint costs. 

\begin{condition}[Objective loss bounds]\label{cond:obj-loss-bds}
    We have $V_{m+1}\in [V^{\mathrm{min}}_{m+1}, V^{\mathrm{max}}_{m+1}]$ a.s. for some finite constants $V^{\mathrm{max}}_{m+1}\ge V^{\mathrm{min}}_{m+1} \ge 0$.
\end{condition}

Whereas in \Cref{subsec:ub} and \Cref{subsec:lb}, we relied only on the exchangeability of the observations, in this section we impose an i.i.d. assumption in order to apply standard concentration inequalities.

\begin{condition}[I.I.D. observations]\label{cond:iid-observations}
The observations $\{ (X^{(i)}, Y^{(i)}, Y^{*,(i)}) : i\in [n+1] \}$ are i.i.d.
\end{condition}

We also assume that the scores are compactly supported.

\begin{condition}[Compactly supported scores]\label{cond:cpt-supp}
The joint distribution of $S_{1:m}$
has support contained
in the box
$\mathcal{B} = \prod_{\ell=1}^m (0, B_{\ell})$ 
for some finite $B_{1:m} > 0$.
\end{condition}

In practice, this condition requires that the scores can be normalized and truncated into a finite range, which can then be shifted to be contained in an appropriate box.\footnote{We note that the specific box chosen in this condition is merely for convenience, and the same results would hold for any other choice of the box.} 
We then let
$\Lambda_{1:m} = \prod_{\ell=1}^m [0, B_{\ell}]$,
where $B_{1:m}$ are as in
\Cref{cond:cpt-supp}.\footnote{Here, it is important that the hyperparameter space includes the entire range of the scores, which can ensure that for the extreme setting $\lambda_j=0$, $j\in[m]$, of the hyperparameters we obtain zero risk.}
Having defined
$\Lambda_{1:m}$,
given a positive integer $k$,
we define
$(\lambda_{1:m}^{*, (2k)})$
recursively via similar logic as 
$(\lambda_{1:m}^{*})$
above but with the adjusted budgets $\beta_1^{(k-1)}$ instead of the original budgets $\beta$ by 
$\lambda_1^{*,(2k)} = U_1^*(\beta_1^{(k-1)})$
and
\begin{align}\label{eq:lambd-star-2k}
\lambda_j^{*, (2k)} = U_j^*(\lambda_{1:(j-1)}^{*, (2k)}; \beta_j^{(k-1)})
\end{align}
for $j\in \{2,\ldots,m\}$.
As noted above,
if
$\lambda_j^{*,(2k)}\in 
(\lambda_j^{\mathrm{min}}, \lambda_j^{\mathrm{max}})$
for all $j\in [m]$,
then
$\lambda_{1:m}^{*,(2k)}$
is a minimizer
of the population-level problem
with risk upper bounds given by
$(\beta_{1:m}^{(k-1)})$
(\Cref{lem:pop-minimizer}).

For each of the continuous scores,
we impose a Lipschitz condition of the cumulative distribution function (c.d.f.).
Such conditions are often needed to establish concentration conditions \citep[e.g.,][etc]{sesia2021conformal,jung2023batch,kiyani2024length} in the literature on conformal prediction.

\begin{condition}[Lipschitz c.d.f.s of scores]\label{cond:cdf-lip}
For each $j\in \mathcal{I}_{\mathrm{cts}}$,
the c.d.f. of $S_j$ is $K_j$-Lipschitz
for some constant $K_j>0$.
\end{condition}

In order show that our algorithm achieves an objective that tightly concentrates around the population optimum, 
for $j\in \ii_{\mathrm{cts}}$ we require that the population constraint risk function $g_j^*(\lambda_j; \lambda_{1:(j-1)}) = \E L_j(\lambda_{1:j})$ decreases sufficiently fast in $\lambda_j$.
Here we opt for quantifying this via a polynomial-type growth condition on the c.d.f. of the scores. 

\begin{condition}[Reverse H\"older c.d.f. of scores]\label{cond:rev-lip}
There is some $\nu\ge 1$,
such that for each $j\in \mathcal{I}_{\mathrm{cts}}$,
there exists a constant $\widetilde K_j > 0$
for which
the joint c.d.f.~
$F : \R^m \to [0,1]$
of $S_{1:m}$, given by 
$F(s_{1:m}) := \PP{S_{1:m} \pce s_{1:m}}$
for all $s_{1:m}\in \R^m$,
obeys the reverse $\nu$-Hölder condition 
\begin{align*}
    F(s_{1:(j-1)}, s_j', s_{(j+1):m}) - F(s_{1:(j-1)}, s_j, s_{(j+1):m}) \ge \widetilde K_j (s_j' - s_j)^{\nu}
\end{align*}
for all $s_j\le s_j'$ with $s_j, s_j' \in (0,B_j)$
and for all other coordinates $s_{\ell}\in (0,B_{\ell})$
with $\ell\in [m]\setminus\{j\}$.
(Here, $B_j$ is defined as in \Cref{cond:cpt-supp}.)
\end{condition}

For the discrete scores,
we introduce the following notation.

\begin{condition}[Finite supports of discrete scores]\label{cond:disc-supp}
For $j\in \mathcal{I}_{\mathrm{discrete}}$,
there is a finite set $\mathcal{S}_j \subseteq \R$
to which $S_j$ belongs a.s.
\end{condition}

For each $j\in \mathcal{I}_{\mathrm{discrete}}$,
for each $k\in \{0, \ldots, m-j+1\}$, 
let\footnote{For two sets $A$ and $B$ and a function $f: A \to B$, $f(A)$ is the image of $A$ under $f$.}
$\vv_j^{(k)} := g_j^*(\R; \lambda_{1:(j-1)}^{*, (2k+2)})$
denote the range of the population risk function 
$g_j^*(\cdot; \lambda_{1:(j-1)}^{*, (2k+2)})$ from \eqref{gjs} evaluated at $\lambda_{1:(j-1)}^{*, (2k+2)}$ from \Cref{eq:lambd-star-2k}.
 Observe that 
due to \Cref{cond:disc-supp}
 for all $j\in\mathcal{I}_{\mathrm{discrete}}, k\in \{0,\ldots,m-j+1\}$, $\vv_j^{(k)}$ are finite sets. 
If $j=1$ is in $\mathcal{I}_{\mathrm{discrete}}$,
then we have that
$\vv_1^{(k)}=
g_1^*(\R)$
for all $k\in \{0,\ldots,m\}$, which does not depend on $k$,
so we drop the superscript
and write
$\vv_1 = \vv_1^{(k)}$
for all $k\in \{0,\ldots,m\}$.

For discrete scores, the functions
$g_j^*(\cdot; \lambda_{1:(j-1)})$
are step functions with a discrete range.
It turns out that controlling them at levels that are equal to the values in their image poses some technical challenges due to the nonzero probabilities of these functions taking those exact values.
Therefore, we will require the corresponding values of
$\beta_j$ to lie outside of these finite sets.
In practice, this is not a restriction, as it holds for almost every $\beta_j$ with respect to the Lebesgue measure, and moreover can be easily achieved by perturbing the levels $\beta_j$ with small continuous-valued noise.

\begin{condition}[Non-degenerate risk levels]\label{cond:no-bad-beta}
For each $j\in \mathcal{I}_{\mathrm{discrete}}$,
suppose that
$\beta_j\not\in g_j^*(\R; \lambda_{1:(j-1)}^{*, (2)})$, where $\lambda_{1:(j-1)}^{*, (2)}$ is defined in \Cref{eq:lambd-star-2k}. 
\end{condition}

Finally, we impose a condition
to ensure that
$\lambda_{1:m}^{*,(2)}$
are minimizers of
the population-level problem from
\Cref{eq:opt}.
 As previously discussed, this is a condition that ensures that 
 the risk function $g_j^*(\cdot; \lambda_{1:(j-1)}^*)$ crosses 
$\beta_j$, so that tightness is possible in our case. 

\begin{condition}[Sufficient condition for population minimizer]\label{cond:beta-tilde-exists}
For all $j\in [m]$, 
suppose that
$\lambda_j^{*,(2)}\in (0, B_j)$,
where $\lambda_{j}^{*, (2)}$ is defined in \Cref{eq:lambd-star-2k},
and $B_j$ is defined in \Cref{cond:cpt-supp}.
\end{condition}



Under the above conditions, 
the following result 
shows that the 
\multirisk~thresholds
 are near-optimal
for the objective in \Cref{eq:opt}.

\begin{theorem}[Bound on \multirisk~objective value]\label{thm:cts-conc}
Under the conditions
in \Cref{subsubsec:ub-conds},
\Cref{cond:m-positive},
and the conditions in \Cref{subsubsec:cts-conc-conds},
if $(\hat \lambda_{1:m}^{(2)})$ denote the \multirisk~thresholds constructed in
\Cref{thm:multiple-scores},
and if
$(\lambda_{1:m}^{*})$
is any population minimizer of the objective from \eqref{eq:opt},
then we have
\begin{align*}
    \E V_{m+1}^{(n+1)} I(S_{1:m}^{(n+1)} \pce \hat \lambda_{1:m}^{(2)}) \le \E V_{m+1}^{(n+1)} I(S_{1:m}^{(n+1)} \pce \lambda_{1:m}^{*}) 
    + O\left( 
    \left( \frac{\log n}{n} \right)^{1/(2\nu^{m})} 
    \right).
\end{align*}
\end{theorem}


\begin{remark}
It is natural to ask
whether
\Cref{thm:cts-conc}
 holds for any threshold satisfying the 
upper and lower bounds
in 
\Cref{thm:master}, 
not just for those constructed by our algorithm. 
This is not the case,
as the following counterexample shows; and the specific construction of our algorithm is used beyond the result from \Cref{thm:master} to show near-optimality.

Suppose that
$S_j \sim \mathrm{Unif}([0,1])$
are i.i.d.
for $j\in [m]$,
and suppose that
$V_j \equiv 1$
for $j\in [m]$.
Then it can be verified that 
$\lambda_{1:m}^* = \beta_{1:m}$.
For each $j\in [m]$,
fix a constant 
$c_j \in (0, \min\{ \beta_j, 1-\beta_j \})$,
and let the random variable
$\hat \lambda_j$ take 
the value $\beta_j + c_j$ with probability $1/2$
and
the value $\beta_j - c_j$ with probability $1/2$,
independently of all else.
Then
for $j\in [m]$,
we have
$\E L_j^{(n+1)}(\hat \lambda_{1:j}) = \beta_j$,
while
$|\hat \lambda_j - \lambda_j^*| = c_j$
is of constant order
and does not tend to zero
as $n\to \infty$.
Consequently,
\Cref{thm:cts-conc}
does not hold.
\end{remark}

Here we present the main ideas behind the proof of \Cref{thm:cts-conc}. First, by means of the Lipschitz assumption in \Cref{cond:cdf-lip}, we reduce bounding
\begin{align*}
\E V_{m+1}^{(n+1)} I(S_{1:m}^{(n+1)} \pce \hat \lambda_{1:m}^{(2)}) - \E V_{m+1}^{(n+1)} I(S_{1:m}^{(n+1)} \pce \lambda_{1:m}^{*,(2)})
\end{align*}
to controlling the differences
$|\hat \lambda_{j}^{(2)} - \lambda_{j}^{*,(2)}|$
for $j\in [m]$
with high probability.
 Due to the recursive nature of the construction of the thresholds, we achieve this through showing 
by induction on $j\in [m]$ and $k\in [m-j+1]$ that the differences $|\hat \lambda_j^{(2k)} - \lambda_j^{*,(2k)}|$ can be bounded with high probability.

Given certain $\lambda_{1:(j-1)}\in \R^{j-1}$ and $\beta,\beta'\in \R$,
the inductive argument involves controlling differences of the form
\begin{align*}
(I): \quad |U_j^*(\lambda_{1:(j-1)}; \beta) - U_j^*(\lambda_{1:(j-1)}; \beta')|
\end{align*}
and
\begin{align*}
(II): \quad |U_j^*(\lambda_{1:(j-1)}; \beta) - U_j^+(\lambda_{1:(j-1)}; \beta)|.
\end{align*}

To control expressions of type (I), we recall that $U_j^*(\lambda_{1:(j-1)}; \cdot)$ is defined as the generalized inverse of $g_j^*(\cdot; \lambda_{1:(j-1)})$. The reverse H\"older assumption in \Cref{cond:rev-lip} implies that $g_j^*(\cdot; \lambda_{1:(j-1)})$ decreases rapidly (\Cref{lem:condl-rev-lip-bds}), which implies that $U_j^*(\lambda_{1:(j-1)}; \cdot)$ decreases slowly. Hence, 
we can bound a difference of type (I) 
by $|\beta-\beta'|^{1/\nu}$ up to constants.

For expressions of type (II), we first bound the uniform norm
\begin{align*}
\|g_j^+(\cdot; \lambda_{1:(j-1)}) - g_j^*(\cdot; \lambda_{1:(j-1)})\|_{L^{\infty}(\R)}
\end{align*}
with high probability (\Cref{lem:emp-proc-bd}).
To do so, we rewrite this quantity in terms of an empirical process over a function class $\ff_j$ that can be expressed as the product $\tilde f_j\cdot \mathcal{G}_j$, where the function $\tilde f_j$ is uniformly bounded and the function class $\mathcal{G}_j$ has VC-dimension unity. 
Applying the Ledoux-Talagrand contraction lemma to $\ff_j$ \citep[Theorem 4.12]{ledoux2013probability}, we obtain a bound on the Rademacher complexity of $\ff_j$ of size $O(\sqrt{\log n/n})$, which in turn provides a high probability bound on the uniform norm.
Equipped with this bound, by \Cref{lem:crossing-points}, if
\begin{align*}
\|g_j^+(\cdot; \lambda_{1:(j-1)}) - g_j^*(\cdot; \lambda_{1:(j-1)})\|_{L^{\infty}(\R)} \le \ep
\end{align*}
for some $\ep > 0$, then by the definitions of $U_j^+$ and $U_j^*$, we necessarily have
\begin{align*}
    U_j^*(\lambda_{1:(j-1)}; \beta+\ep) \le U_j^+(\lambda_{1:(j-1)}; \beta) \le U_j^*(\lambda_{1:(j-1)}; \beta-\ep),
\end{align*}
so that $ |U_j^+(\lambda_{1:(j-1)}; \beta) - U_j^*(\lambda_{1:(j-1)}; \beta)|$ is bounded by 
\begin{align*}
    \max\big\{|U_j^*(\lambda_{1:(j-1)}; \beta+\ep) - U_j^*(\lambda_{1:(j-1)}; \beta)|,
     |U_j^*(\lambda_{1:(j-1)}; \beta-\ep) - U_j^*(\lambda_{1:(j-1)}; \beta)|\big\}.
\end{align*}
It follows that bounding expressions of type (II) reduces to bounding expressions of type (I).
The proof of \Cref{thm:cts-conc} is given in \Cref{subsec:pf-cts-conc}.

\subsection{Concentration for discrete scores}\label{subsec:disc-conc}

The results of
\Cref{subsec:ub},
\Cref{subsec:lb},
and
\Cref{subsec:cts-conc}
apply
in the case of discrete scores.
However,
when $S_j$ is discrete
for all $j\in [m]$
(that is,
when $\mathcal{I}_{\mathrm{discrete}} = [m]$),
\Cref{thm:cts-conc}
can be improved,
as we show here.

\subsubsection{Concentration conditions}\label{subsubsec:disc-conc-conds}

The following two conditions are \Cref{cond:disc-supp} and \Cref{cond:no-bad-beta} specialized to the case $\mathcal{I}_{\mathrm{discrete}} = [m]$.
 We state them here for clarity. 

\begin{condition}[Finite supports of all scores]\label{cond:disc-supp-strong}
For $j\in [m]$,
there is a finite set $\mathcal{S}_j \subseteq \R$
to which $S_j$ belongs a.s.
\end{condition}

\begin{condition}[Non-degenerate risk levels for all scores]\label{cond:no-bad-beta-strong}
For each $j\in [m]$,
suppose that
$\beta_j\not\in g_j^*(\R; \lambda_{1:(j-1)}^{*,(2)})$,
where $\lambda_{1:(j-1)}^{*, (2)}$ is defined in \Cref{eq:lambd-star-2k}.
\end{condition}


The following is the strengthened form of \Cref{thm:cts-conc} in the case that all scores are discrete.

\begin{theorem}[Bound on \multirisk~objective value for discrete scores]\label{thm:disc-conc}
Under the conditions
in \Cref{subsubsec:ub-conds},
under
\Cref{cond:iid-observations},
and under the conditions
in \Cref{subsubsec:disc-conc-conds},
if $(\hat \lambda_{1:m}^{(2)})$ denote the \multirisk~thresholds constructed in
\Cref{thm:multiple-scores},
and if
$(\lambda_{1:m}^{*})$
is any population minimizer,
then we have
\begin{align*}
    \E V_{m+1}^{(n+1)} I(S_{1:m}^{(n+1)} \pce \hat \lambda_{1:m}^{(2)}) \le \E V_{m+1}^{(n+1)} I(S_{1:m}^{(n+1)} \pce \lambda_{1:m}^{*}) 
    + O\left( \exp(-n^{0.99}) \right).
\end{align*}
\end{theorem}

\begin{remark}
The exponent $0.99$ appearing in \Cref{thm:disc-conc} can  be replaced with any constant $\zeta\in (0,1)$, as can be seen from the proof in \Cref{subsec:pf-disc-conc}.
\end{remark}

The proof of \Cref{thm:disc-conc}, given in \Cref{subsec:pf-disc-conc}, proceeds along the same lines as the proof of \Cref{thm:cts-conc}.
However, it inductively proves the stronger statement
$|\hat \lambda_j^{(2k)} - \lambda_j^{*,(2k)}| = 0$
for all $j\in [m]$ and $k\in [m-j+1]$ high probability,
and also improves the high probability guarantee to 
$O(\exp(-n^{0.99}))$.
This is achieved by noticing that in the case of discrete scores, bounding the difference of step functions $g_j^+(\cdot; \lambda_{1:(j-1)}) - g_j^*(\cdot; \lambda_{1:(j-1)})$ reduces to bounding the maximum of finitely many bounded mean zero random variables, corresponding to the range of $g_j^+(\cdot; \lambda_{1:(j-1)}) - g_j^*(\cdot; \lambda_{1:(j-1)})$. 
As a result, Hoeffding's inequality \citep{hoeffding1963probability} can be used in place of the
empirical process bounds (\Cref{lem:emp-proc-bd}) used in the proof of \Cref{thm:cts-conc}.
Consequently, although the strategies behind the proofs of \Cref{thm:cts-conc} and \Cref{thm:disc-conc} are similar, the two arguments cannot easily be unified.

\section{Illustration: LLM alignment}\label{sec:sims}

\textbf{Overview of results.} In this section, we study a three-constraint Large Language Model (LLM) alignment problem involving two safety risks and one uncertainty-based risk.
Across a wide range of risk budgets, \multirisk~and \mrbase~attain lower expected objective cost than the Learn then Test (LTT) baseline from \cite{angelopoulos2025learn},
while also controlling all three constraints.
The resulting empirical Pareto frontiers illustrate smooth and interpretable trade-offs between helpfulness and safety,
highlighting the practical advantages of multi-constraint optimization.

\subsection{Experimental setup}\label{subsec:hh}

{\bf Problem and dataset.}
In this example, we aim to maximize helpfulness
subject to constraints on two harmfulness scores and one uncertainty metric.
We use the PKU-SafeRLHF-30K dataset \citep{ji2023beavertails}, which contains approximately 27000 training examples and 3000 test examples,
from which we form calibration and test sets of size $n_{\text{cal}} = n_{\text{test}} = 500$.
For each question $x$,
we obtain an answer $y$
from the \textsc{Alpaca-7b-reproduced} model from 
\cite{dai2023safe}
using greedy decoding.

{\bf Models, scores, and losses.}
We compute the helpfulness
of a given completion
using the \textsc{beaver-7b-v1.0-reward} model from
\cite{dai2023safe}.
We compute two harmfulness scores (or, negative rewards)
using the \textsc{Llama-Guard-3-8B model} from
\cite{dubey2024llama3herdmodels}.
By design, Llama-Guard-3-8B
detects unsafe content
across 14 categories,
labeled ``S1"
through ``S14".
The first score
$S_1(x,y)$
equals the average of
the probabilities that
the completion $y$
contains material
belonging to the
``S1: Violent Crimes"
category,
``S9: Indiscriminate Weapons",
and
``S11: Self-Harm"
categories.
The second score
$S_2(x,y)$
equals the average of
the probabilities that
the completion $y$
contains material
belonging to the 
``S2: Non-Violent Crimes",
``S5: Defamation",
and
``S7: Privacy"
categories.
The third score
$S_3(x,y)$
equals the perplexity \citep{jelinek1977perplexity} of
the tokens generated by \textsc{Alpaca-7b-reproduced}.\footnote{For an input $x$ and an output $y$ of length $k$ given by a model $p(\cdot|x)$, the perplexity equals $1/p(y|x)^{1/k} \ge 1$. This can be interpreted as one over the length-normalized probability of generating the output, and is a standard measure of the informal uncertainty of the model in its generation. For instance, for a generation that is completely deterministic and where the model has no other choice but the returned value $y$, the perplexity equals unity. If the probability of the returned answer is small, then instead the perplexity will be high.}
 This can be viewed as a measure of the uncertainty of the answer. 

Define the
abstention response
by the string
$y_{\text{abstain}} = \text{``Sorry, I can't answer that."}$.
If any one of the scores exceeds its threshold,
 the associated behavior is to 
return $y_{\text{abstain}}$;
otherwise, it is to return the original completion $y$.

Define 
the cost
$\tilde V_{\text{obj}}(x, y) = -(\text{\textsc{beaver-7b-v1.0-reward}})(x,y)$
to be the unhelpfulness
of the original completion.
Define
the costs
\begin{align*}
\tilde V_1(x,y) = \tilde V_2(x,y) = \tilde V_3(x,y) = -(\text{\textsc{beaver-7b-v1.0-reward}})(x, y_{\text{abstain}})
\end{align*}
to be
the unhelpfulness 
of the abstention response.

In order to run \multirisk~and \mrbase,
we require that all costs are non-negative.
However,
the \textsc{beaver-7b-v1.0-reward} model 
returns both negative and positive rewards.
In order to ensure
that all costs are non-negative,
we shift all costs by their population-level essential minima.
Specifically, for each constraint cost $\tilde V_j$ and the objective cost $\tilde V_{\text{obj}}$, we subtract the corresponding essential minimum, yielding shifted costs $V_j$ and $V_{\text{obj}}$ that are almost surely non-negative.

Given risk budgets $\beta_1, \beta_2, \beta_3$, 
we aim to solve the optimization problem 
\begin{equation}\label{eq:three-constr-opt}
\begin{aligned}
\min_{\lambda_1, \lambda_2, \lambda_3} \quad 
& \E[ V_{\text{obj}} \, I[S_1 \le \lambda_1, S_2 \le \lambda_2, S_3 \le \lambda_3] ] \\
\text{s.t.} \quad 
& \E[ V_1 I[S_1 > \lambda_1] ] \le \beta_1, \\
& \E[ V_2 I[S_1 \le \lambda_1, S_2 > \lambda_2] ] \le \beta_2, \\
& \E[ V_3 I[S_1 \le \lambda_1, S_2 \le \lambda_2, S_3 > \lambda_3] ] \le \beta_3.
\end{aligned}
\end{equation}

In practice, we must estimate the amount by which we shift the costs. 
To do so, we compute empirical minima over the calibration data, and apply these shifts to the calibration and test costs.\footnote{To obtain a precise empirical version of the optimization problem in \Cref{eq:three-constr-opt}, one could estimate the essential minima using an additional holdout set. However, we note that in practice, due to the large number of calibration and test observations, our choice of shifts does not significantly affect the results.}
Let $\{ \tilde V_j^{(i)} : i\in [n_{\text{cal}}] \}$ denote the calibration values of the $j$-th unshifted cost,
and let $\{ \tilde V_j^{(i)} : i\in \{ n_{\text{cal}}+1, \ldots, n_{\text{cal}} + n_{\text{test}} \} \}$ denote the test values of the $j$-th unshifted cost.
For $j\in [3]$, 
let $\hat c_j = \min_{i\in [n_{\text{cal}}]} \tilde V_j^{(i)}$.
We define the $j$-th shifted 
calibration costs as $V_j^{(i)} = \tilde V_j^{(i)} - \hat c_j^{(i)}$ for $i\in [n_{\text{cal}}]$,
and the $j$-th shifted
test costs as $V_j^{(i)} = \tilde V_j^{(i)} - \hat c_j^{(i)}$ for $i\in \{ n_{\text{cal}}+1, \ldots, n_{\text{cal}} + n_{\text{test}} \}$.
We define the objective costs $\{ V_{\text{obj}}^{(i)} : i\in [n_{\text{cal}} + n_{\text{test}}] \}$ similarly.
Further,
to run the \multirisk~algorithm,
for each $j\in [3]$,
we set the upper bound as $V^{\mathrm{max}}_j = \max_{i\in [n_{\text{cal}}]} V_j^{(i)}$,
and we set the lower bound to be $V^{\mathrm{min}}_j = 0$.

We use the Learn then Test algorithm \citep{angelopoulos2025learn} as a baseline
using CLT-based p-values and the Bonferroni multiple testing procedure.
In \Cref{sec:addl-exp-dets}, we clarify the manner in which we set the LTT budgets.

{\bf Visualizing the results.}
Since our problem concerns a three-dimensional constrained optimization problem, 
we choose to visualize the results by showing how the objective function varies as the constraints change.
This is a practically relevant metric as it characterizes how much our test time expected loss changes when we are willing to make trade-offs on the test time performance constraints. 
Visually, this approach leads to two-dimensional empirical Pareto frontiers achieved by various algorithms. 

\subsection{Results}

\Cref{fig:hh_three_constraint/bonferroni/clt/tradeoff_constraint_1_grid_31_budget_101_shuffles_10,fig:hh_three_constraint/bonferroni/clt/tradeoff_constraint_2_grid_31_budget_101_shuffles_10,fig:hh_three_constraint/bonferroni/clt/tradeoff_constraint_3_grid_31_budget_101_shuffles_10}
show that our algorithm tightly controls the three risks at the desired levels across a wide range of budgets, 
in close agreement with our finite-sample theory.

{\bf Effect of varying the budgets.} 
Specifically. 
we illustrate the effect
of varying
risk budgets 1, 2, and 3 on the objective and constraint test risks.
Given
$n_{\text{budgets}} = 101$, 
for \multirisk~and \mrbase,
for each $j\in [m] = [3]$,
in plot $j$
we vary the $j$-th risk budget 
$\beta_j$
within the set
\begin{align*}
    \left\{ \left(1 + 4 \frac{k-1}{n_{\text{budgets}} - 1} \right)
    \cdot 
    0.10 \widehat V^{\mathrm{max}}_i : k \in [n_{\text{budgets}}] \right\},
\end{align*}
while holding
the remaining risk budgets
fixed at
$\beta_{\ell} = 0.10 \widehat V^{\mathrm{max}}_{\ell}$
for $\ell\in [m]\setminus j$.
Given $\beta_{1:m}$
and a confidence level $\delta$,
we compute the LTT risk budgets
$\tilde \beta_{1:m}$
according to
\Cref{eq:heuristic-ltt-budgets}.
We average the risks
over $n_{\text{splits}} = 10$
random
calibration and test splits,
obtained by 
randomly shuffling
the union of
the original
calibration and test sets.
For each plot
associated to
varying risk budget $j\in \{1,2,3\}$,
we plot 
the \multirisk~budget
on the $x$-axis.
In \Cref{fig:hh_three_constraint/bonferroni/clt/surface_plot_grid_31_budget_21_shuffles_0.png},
we further illustrate the 
landscape 
of the \multirisk~objective test risk
as a function of the first two risk budgets
in a surface plot.

{\bf Smaller objective vs. the baseline.}
From
\Cref{fig:hh_three_constraint/bonferroni/clt/tradeoff_constraint_1_grid_31_budget_101_shuffles_10},
\Cref{fig:hh_three_constraint/bonferroni/clt/tradeoff_constraint_2_grid_31_budget_101_shuffles_10},
and
\Cref{fig:hh_three_constraint/bonferroni/clt/tradeoff_constraint_3_grid_31_budget_101_shuffles_10},
we see that
\multirisk~and \mrbase~behave
very similarly,
controlling the risks at their desired values
while achieving a smaller expected objective value than the three versions of LTT considered.
(For a discussion of when \multirisk~and \mrbase~differ, see \Cref{subsec:comparison}.)


\begin{figure}[H]
    \centering
    \includegraphics[scale=0.38]{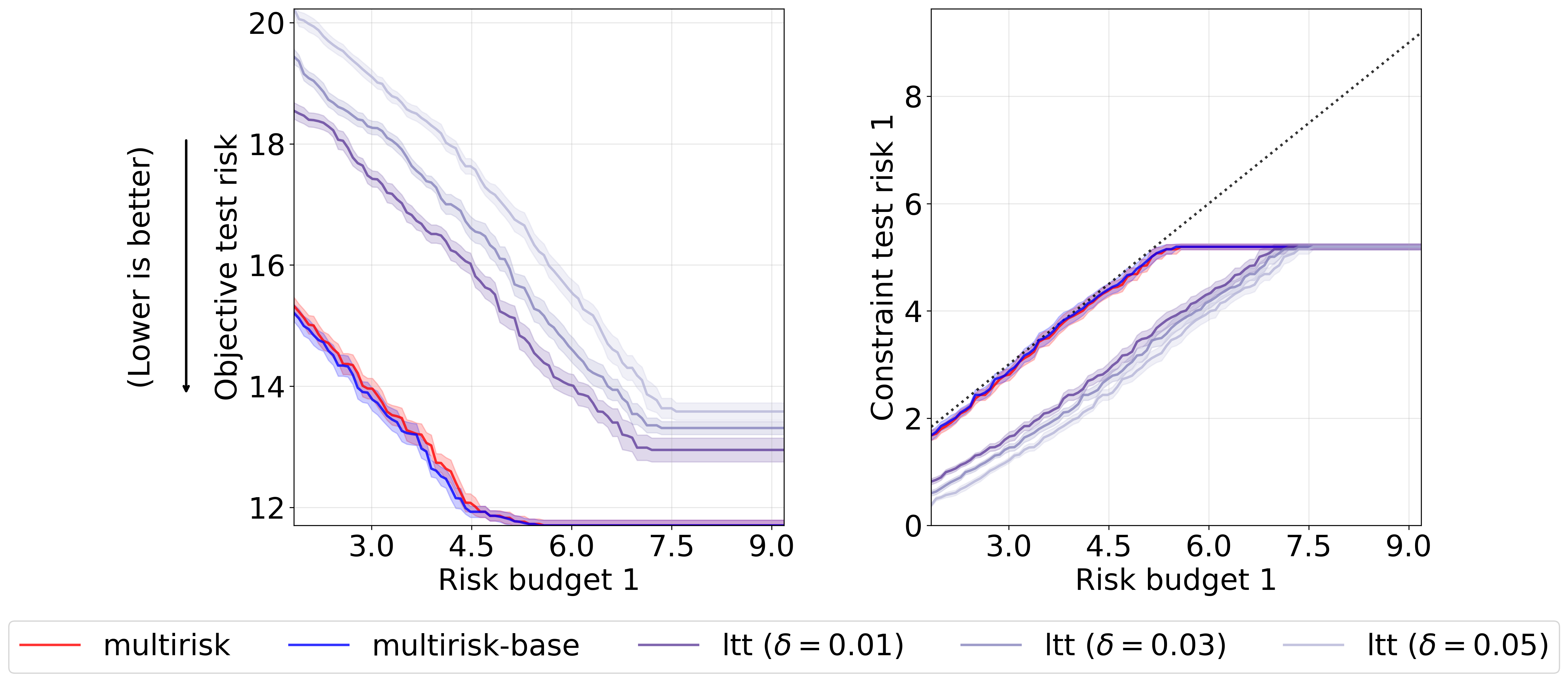}
    \caption{Tradeoff for constraint one for three-constraint example.
    For LTT, for a given $\delta$,
    we set the risk budgets according to
    \Cref{eq:heuristic-ltt-budgets}.
    LTT is run with CLT p-values and the Bonferroni multiple testing procedure.
    Averaged over 10 random calibration-test splits, with an error band of one standard error.}
    \label{fig:hh_three_constraint/bonferroni/clt/tradeoff_constraint_1_grid_31_budget_101_shuffles_10}
\end{figure}

\begin{figure}[H]
    \centering
    \includegraphics[scale=0.4]{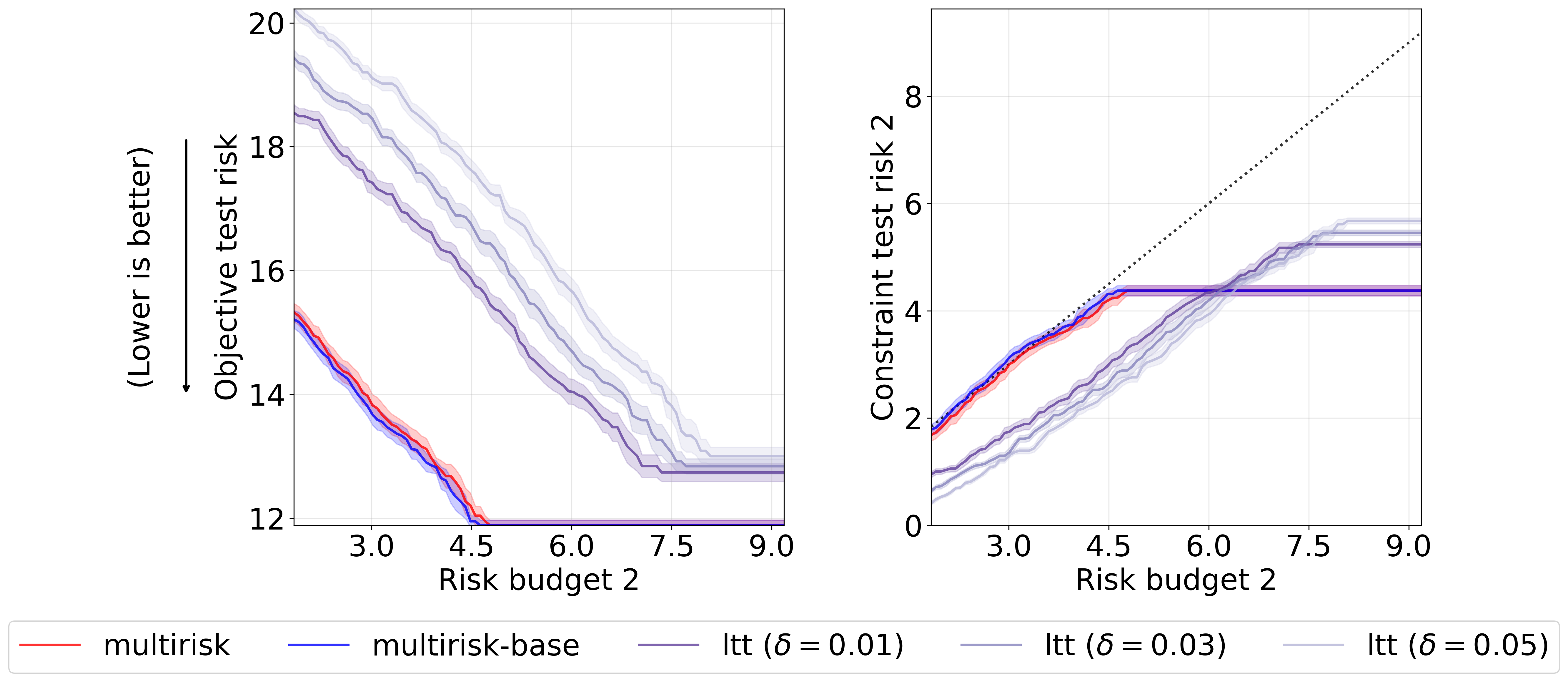}
    \caption{Tradeoff for constraint two for three-constraint example. 
    For LTT, for a given $\delta$,
    we set the risk budgets according to
    \Cref{eq:heuristic-ltt-budgets}.
    LTT is run with CLT p-values and the Bonferroni multiple testing procedure.
    Averaged over 10 random calibration-test splits, with an error band of one standard error.}
    \label{fig:hh_three_constraint/bonferroni/clt/tradeoff_constraint_2_grid_31_budget_101_shuffles_10}
\end{figure}

\begin{figure}[H]
    \centering
    \includegraphics[scale=0.4]{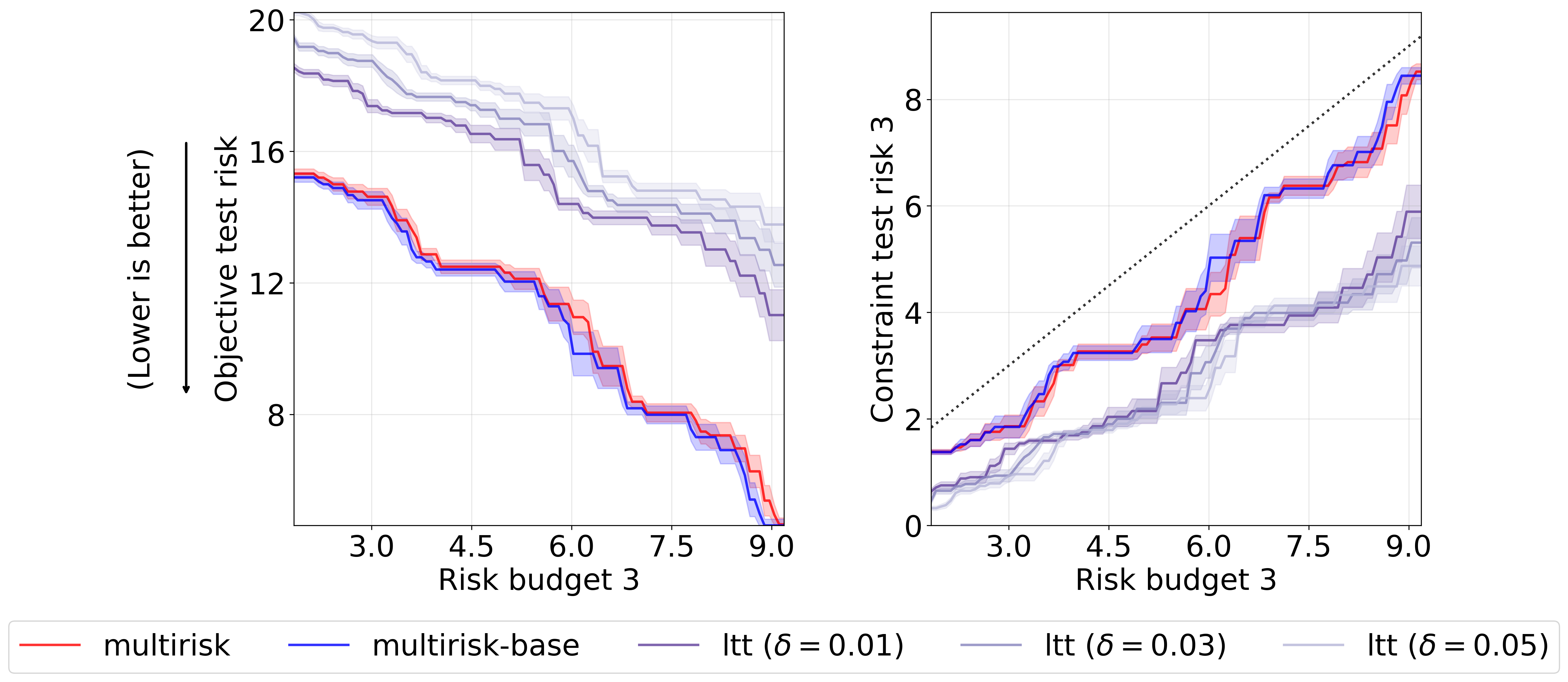}
    \caption{Tradeoff for constraint three for three-constraint example. 
    For LTT, for a given $\delta$,
    we set the risk budgets according to
    \Cref{eq:heuristic-ltt-budgets}.
    LTT is run with CLT p-values and the Bonferroni multiple testing procedure.
    Averaged over 10 random calibration-test splits, with an error band of one standard error.}
    \label{fig:hh_three_constraint/bonferroni/clt/tradeoff_constraint_3_grid_31_budget_101_shuffles_10}
\end{figure}

\begin{figure}[H]
    \centering
    \includegraphics[scale=0.5]{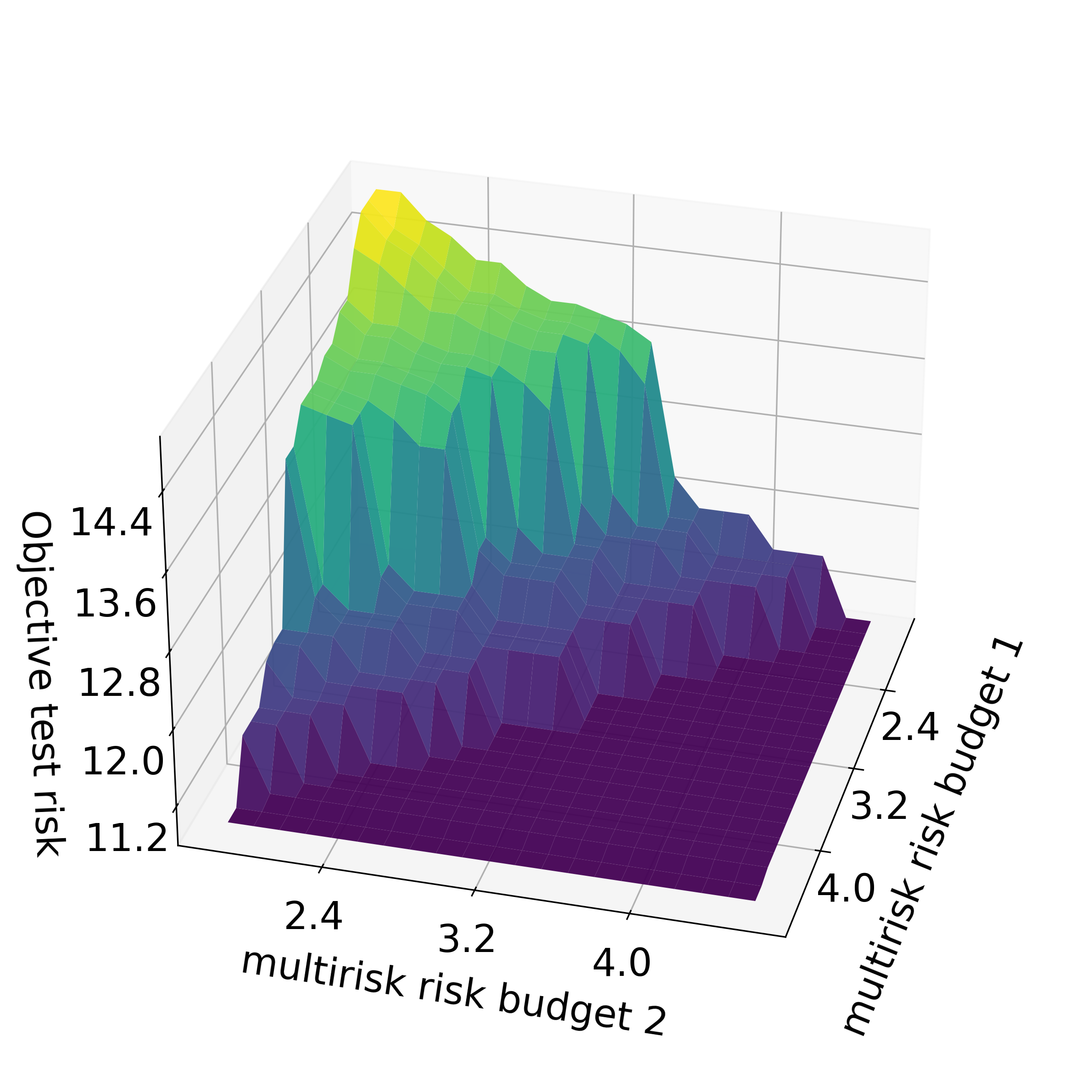}
    \caption{Surface plot of \multirisk~objective test risk as constraint risk budgets 1 and 2 are varied for three-constraint example.}
    \label{fig:hh_three_constraint/bonferroni/clt/surface_plot_grid_31_budget_21_shuffles_0.png}
\end{figure}

\section{Conclusion}

This paper introduced a general framework for test-time filtering in generative models, formulated as optimizing an objective subject to multiple risk constraints with a sequential structure. Building on a natural dynamic programming baseline, \mrbase, we proposed \multirisk, a computationally efficient algorithm that provides distribution-free control of the constraint risks. Theoretically, we established conditions under which the \multirisk~thresholds are near-optimal, and empirically demonstrated that both \mrbase~and \multirisk~effectively regulate harmfulness and uncertainty in Large Language Model outputs.

Beyond this setting, the proposed framework suggests new ways to integrate statistical guarantees into AI deployment. Future work could extend \multirisk~to handle more general graph-structured dependencies among risks.

\section*{Acknowledgements}
This work was supported in part by the US NSF, ARO, AFOSR, ONR, the Simons Foundation and the Sloan Foundation.

{\small
\setlength{\bibsep}{0.2pt plus 0.3ex}
\bibliographystyle{plainnat-abbrev}
\bibliography{ref}
}

\newpage
\appendix 

\section{Additional related work}\label{sec:ext-rel-work}

Existing methods for
multi-objective control and alignment
either
directly modify the weights of the model,
or learn a low-dimensional set of parameters
to perform post-hoc modifications to the output.

\textbf{Algorithms with finite-sample theoretical guarantees.} The following papers
offer finite-sample
theoretical guarantees.

\cite{laufer2023efficiently} 
develop the Pareto Testing method to 
optimize multiple objectives
and
control multiple risks with 
the Learn Then Test framework \citep{angelopoulos2025learn}.
This method identifies parameter configurations on the Pareto frontier that satisfy all constraints, with a given probability.
In this work, we focus on optimizing a single objective and restrict to constraints with a sequential structure.

\cite{nguyen2024data} propose 
the restricted risk resampling (RRR) method for asymptotic high-probability control of a risk metric restricted to a localized set of one-dimensional hyperparameters that itself could be estimated (e.g., those achieving small population risk).
The authors also propose
a method for high-probability control of arbitrary functions of monotone loss functions; which is extended to non-monotone losses via monotonization.
This method first requires uniform confidence bands for the loss functions over a fixed grid of hyperparameters.
In contrast, our focus is multi-dimensional hyperparameters and on finite-sample bounds.

\cite{overman2025conformal} balance an objective and a constraint.
The authors propose
the Conformal Arbitrage (CA) method,
which constructs a threshold
that measures confidence
and that decides which of two models
to query.
Their distribution-free construction
controls a measure of disutility
to within a specified budget.

\cite{andeol2025conformal}
propose
the distribution-free
Sequential Conformal Risk Control (SeqCRC) method
to control three risks of the form
$L_1(\lambda_1)$, $L_2(\lambda_1, \lambda_2)$, and $L_3(\lambda_1, \lambda_3)$,
where $L_j$ is non-increasing in each argument
for each $j\in [3]$.
Although the construction
of the thresholds
is similar
to our \multirisk~algorithm,
\multirisk~
applies to risks
with a different monotonicity pattern,
and is designed to control
any number
of constraints.

\textbf{Empirical algorithms for multi-objective alignment.} The following papers offer
empirical approaches
that modify the weights of the model,
but do not provide finite-sample theoretical guarantees.
Much of this work builds on
the multi-objective reinforcement learning (MORL) literature
\citep{barrett2008learning, roijers2013survey, van2014multi, li2020deep, hayes2022practical}.

\cite{dai2024safe}
highlights the challenge
of disentangling 
conflicting objectives
in preference data
for LLM alignment,
such as helpfulness and harmlessness.
The authors train separate reward and cost models
to capture helpfulness and harmfulness,
respectively.
The authors then formulate the Safe RLHF algorithm
to maximize expected helpfulness
subject to a constraint
on the expected harmfulness,
and solve this optimization problem
using the Lagrangian method.
\cite{williams2024multi}
proposes
Multi-Objective Reinforcement Learning from AI Feedback
(MORLAIF),
which,
in a similar vein
to \cite{dai2024safe},
trains separate reward models
to capture individual principles,
this time using the
preferences of a stronger LLM,
and then performs standard PPO-based
RLHF training with respect
to a scalar function of these rewards.

\cite{guo2024controllable}
augments SFT and preference optimization data
for LLMs
by placing priorities
on the various objectives.
The authors introduce 
controllable preference optimization (CPO),
which consists of
controllable preference supervised fine-tuning (CPSFT)
and
controllable direct preference optimization
(CDPO).
Through CPO, the model learns 
to take priorities into account during generation.
In order to combat
the instability of RL training,
\cite{zhou2024beyond}
propose 
Multi-Objective Direct Preference Optimization (MODPO), 
an extension of Direct Preference Optimization (DPO)
for multiple alignment objectives.

\cite{li2025self}
point out that
preference data
that include 
conflicting ratings
with respect to different objectives
can cause issues for
DPO-based alignment methods.
To address this,
the authors propose the
Self-Improvement DPO framework 
towards Pareto Optimality
(SIPO)
to improve preference data
by sampling new
nonconflicting responses
from the LLM
after an initial alignment phase.
This can be used 
in conjunction with 
any multi-objective 
DPO-based alignment method.
\cite{rame2023rewarded}
propose the
rewarded soup method
for multi-objective alignment
of generative models,
an alternative to scalarization.
Taking inspiration
from the literature on the
linear mode connectivity phenomenon
\citep{frankle2020linear, neyshabur2020being},
starting from a single pretrained model,
the authors specialize 
one copy of the model to each reward,
and then combine these fine-tuned models
by taking a linear combination
of their weights,
where the specific linear combination
depends on the preferences
of the user.

\cite{tamboli2025balanceddpo}
study the setting of
multi-objective alignment of
text-to-image (T2I) diffusion models.
The authors point out the limitations
of methods that rely on scalarization
of learned reward models,
including 
reward rescaling issues,
conflicting gradients,
and pipeline complexity.
As an alternative,
the authors propose the BalancedDPO method,
which aggregates preferences
by performing a majority vote,
avoiding the need individual reward models.
The resulting aggregated preferences
are then used to perform standard 
DPO-based alignment.
\cite{barker2025faster}
study the problem of
selecting hyperparameters
to optimize multiple objectives
in complex LLM-based
retrieval-augmented generation (RAG)
systems,
such as cost, latency, safety, and alignment.
The authors utilize
Bayesian optimization
to compute a Pareto-optimal set
of hyperparameter configurations.

\section{Additional experimental details}\label{sec:addl-exp-dets}

\textbf{Decoding details.}
Note that during
greedy decoding
from \textsc{Alpaca-7b-reproduced},
we stop the generations
at 512 tokens.
In 11 of the 1000 generations, we hit the token limit. These are instances where the model starts repeating itself indefinitely.

\textbf{Definition of CLT p-values.}
Let $\Phi$ denote the
standard normal c.d.f.
Given calibration losses $\{ L_j^{(i)} : i\in [n_{\text{cal}}] \}$,
the CLT-based p-value
for the $j$-th LTT null hypothesis
$H_j :
\E L_j(\lambda_{1:j}) 
> \beta_j$
is given by
$p_j = 
1 - 
\Phi\left( \frac{\beta_j - \widehat L_j}{\hat \sigma_j / \sqrt n} \right)$,
where
$\widehat L_j = 
\frac{1}{n_{\text{cal}}} 
\sum_{i=1}^{n_{\text{cal}}}
L_j^{(i)}(\lambda_{1:j})$
and is the $j$-th empirical calibration risk,
and where
$\hat \sigma_j = 
\frac{1}{n_{\text{cal}}-1} 
\sum_{i=1}^{n_{\text{cal}}}
(L_j^{(i)}(\lambda_{1:j}) - \widehat L_j)^2$
is the $j$-th empirical calibration standard deviation.
(For a proof 
of asymptotic validity, 
see \citet[Proposition B.1]{angelopoulos2025learn}.)
 We remark that it is possible to 
have $\hat \sigma_j = 0$, 
especially for extreme values of thresholds. 
We handle this edge case
as follows.
If $\hat \sigma_j = 0$
and 
$\widehat L_j > \beta_j$,
we set
$p_j = 1$.
If $\hat \sigma_j = 0$
and
$\widehat L_j = \beta_j$,
we set
$p_j = 0.5$.
If $\hat \sigma_j = 0$
and
$\widehat L_j > \beta_j$,
we set
$p_j = 0$
(and hence, we necessarily reject $H_j$).
This can be justified
by considering the limit
$\hat \sigma_j \searrow 0$.

\textbf{Setting LTT constraint risk budgets.}
Recall that the LTT algorithm takes as input a confidence level $\delta\in (0,1)$ and risk budgets $\beta_{1:m}$, and returns thresholds 
that simultaneously obey all risk constraints with probability at least $1-\delta$.
In order to compare \multirisk~(which provides risk bounds in expectation)
with LTT (which provides high probability risk bounds),
we set the LTT risk budgets according to the following heuristic.
Note that 
if we run the LTT algorithm
with risk budgets $\tilde \beta_{1:m}$
and with confidence level $\delta\in (0,1)$
to obtain thresholds $\hat\lambda_{1:m}^{\text{LTT}}$,
then we may convert the LTT high probability guarantee
to an in-expectation guarantee
via the law of total expectation:
\begin{align*}
    \E[L_j^{(n+1)}(\hat \lambda_{1:j}^{\text{LTT}})] 
&= \E[L_j^{(n+1)}(\hat \lambda_{1:j}^{\text{LTT}}) 
I(L_j^{(n+1)}(\hat \lambda_{1:j}^{\text{LTT}}) \le \tilde \beta_j)]
+ 
\E[L_j^{(n+1)}(\hat \lambda_{1:j}^{\text{LTT}}) 
I(L_j^{(n+1)}(\hat \lambda_{1:j}^{\text{LTT}}) > \tilde \beta_j)] \\
&\approx 
(1-\delta) \tilde \beta_j + \delta V^{\mathrm{max}}_j
\approx
(1-\delta) \tilde \beta_j + \delta \widehat V^{\mathrm{max}}_j,
\end{align*}
where $\widehat V^{\mathrm{max}}_j$
denotes the maximum value 
of the loss $V_j^{(i)}$
in the calibration set.
Our heuristic is to set 
the $j$-th \multirisk~risk budget
equal to this quantity,
so that
\begin{align}\label{eq:heuristic-ltt-budgets}
    \beta_j = (1-\delta) \tilde \beta_j + \delta \widehat V^{\mathrm{max}}_j \text{ for } j\in [m].
\end{align}
Equivalently,
fixing the \multirisk~risk budgets $\beta_{1:m}$
and the LTT confidence level $\delta$,
we set the $j$-th LTT risk budget equal to
$\tilde \beta_j = (\beta_j - \delta \widehat V^{\mathrm{max}}_j) / (1-\delta)$
for each $j\in [m]$.

\textbf{Comparing amount of constraint budget exhausted.}
We note that
for the second constraint
in
\Cref{fig:hh_three_constraint/bonferroni/clt/tradeoff_constraint_2_grid_31_budget_101_shuffles_10},
for large values of the risk budget
the three LTT variants exhaust more of the constraint risk budget than \multirisk~and \mrbase,
even though the objective test risks of \multirisk~and \mrbase~are lower.
This is likely due
to the sequential nature of
\multirisk~and \mrbase,
as the following simplified argument shows.
In \mrbase, 
the first threshold $\hat \lambda_1^{\Mrbase}$ is greedily selected to exhaust as much of the first constraint risk budget as possible,
and since the first constraint loss $L_1(\cdot)$ is non-increasing,
this suggests that
$\hat \lambda_1^{\Mrbase} \le
\hat \lambda_1^{\text{LTT}}$.
Consequently,
since the second constraint loss $L_2(\cdot)$ is non-decreasing in its first argument,
we have
\begin{align*}
L_2(\hat \lambda_1^{\Mrbase}, \lambda_2)
\le 
 L_2(\hat \lambda_1^{\text{LTT}}, \lambda_2)
\end{align*}
for all $\lambda_2\in \R$.
It follows that if
$\hat \lambda_2^{\Mrbase}$ and $\hat \lambda_2^{\text{LTT}}$ are close,
we expect LTT to have a higher constraint two risk.
Indeed, since for large $\beta_2$ the constraint risks plateau, both \mrbase~and LTT set $\hat \lambda_2^{\Mrbase} = \hat \lambda_2^{\text{LTT}} = \lambda_2^{\mathrm{min}}$,
and LTT achieves higher constraint two risk.

\section{Comparing
\multirisk~and \mrbase}\label{subsec:comparison}

Here,
we compare \multirisk~with the heuristic \mrbase~algorithm.

As seen in
\Cref{subsec:hh},
when $n$ is large
relative to $V^{\mathrm{max}}_{1:m}$,
\multirisk~and \mrbase~behave similarly.
In this case,
the conformal corrections
$\frac{V^{\mathrm{max}}_j}{n+1}$
and
$\delta_j = \frac{V^{\mathrm{max}}_j-V^{\mathrm{min}}_j}{n+1}$
made by \multirisk~are small.
However,
if $V^{\mathrm{max}}_{1:m}$ are large compared to $n$,
then \mrbase~can substantially exceed the risk budgets.

Consider the following scenario.
Given $m\ge 1$ constraints,
we construct i.i.d. scores
$S_{1:m}$ as follows.
Fix constants $V^{\mathrm{max}}_j > 1$
for $j\in [m]$
and probabilities $p_j\in (0,1)$
for $j\in [m]$.
For $j\in [m]$,
independently of the other scores,
let the $j$-th score $S_j$ equal
$V^{\mathrm{max}}_j$ with probability $p_j \in (0,1)$,
and
let $S_j$ be sampled from $\mathrm{Unif}([0,1])$
otherwise.
For $j\in [m]$,
let the $j$-th loss be $V_j = S_j$.

Here,
we set 
$m=2$,
$V^{\mathrm{max}}_1=4.6$,
$p_1=0.055$,
$V^{\mathrm{max}}_2=90$,
and
$p_2=0.01$,
and we use calibration sets of size
$n_{\text{cal}} = 20$.
We aim ensure
$\E V_1 I(S_1 > \lambda_1) \le \beta_1$
and
$\E V_2 I(S_1\le \lambda_1, S_2 > \lambda_2) \le \beta_2$
for risk budgets 
$(\beta_1, \beta_2) = (0.23, 0.23)$.
We compute
the population test risk
for each constraint
averaged over
$n_{\text{batches}} = 5000$ i.i.d. calibration sets.
The results are displayed in
\Cref{fig:simple-risk-table}.
We see that
\multirisk~satisfies both constraints,
whereas \mrbase~exceeds each budget by at least one standard error.
In the case of constraint 1,
\mrbase~only exceeds the budget by
approximately $0.01$,
whereas \multirisk~is conservative
by approximately $0.14$.
However, for constraint 2,
due to the large value of $V^{\mathrm{max}}_2$,
\mrbase~violates the constraint
by more than $0.42$,
whereas \multirisk~sets the threshold
to ensure zero risk.

\begin{table}[ht]
\centering
\begin{tabular}{lcc}
\hline
Algorithm    & Constraint 1: $\mathbb{E} R_{1} \le 0.23$             & Constraint 2: $\mathbb{E} R_{2} \le 0.23$             \\
\hline
\textsc{multiperf}~  & $0.085087 \pm 0.001783$ & $0.000000 \pm 0.000000$ \\
\textsc{multiperf-base}~  & $0.242805 \pm 0.002582$ & $0.665129 \pm 0.004424$ \\
\hline
\end{tabular}
\caption{Test risks for each algorithm (mean $\pm$ standard error) with budgets $(\beta_1, \beta_2) = (0.23, 0.23)$.}
\label{fig:simple-risk-table}
\end{table}

Theoretically,
we can understand
the behavior of \mrbase~for constraint 2
in \Cref{fig:simple-risk-table}
as follows.
Consider the special case when $m=1$.
Note that if $M$ is allowed to be made arbitrarily large,
then the risk of \mrbase~can arbitrarily exceed the risk budget $\beta_1$.
Indeed, with probability $(1-p)^{n_{\text{cal}}}$
the calibration set only contains scores $S_1^{(i)}$ in $[0,1]$,
in which case the threshold $\hat \lambda_1$
selected by \mrbase~must lie in $[0,1]$.
If the test score $S_1^{(n+1)}$ equals $M$,
then the threshold $\hat \lambda_1$ incurs a loss of
$V_1^{(n+1)} I(S_1^{(n+1)} > \hat \lambda_1) = M$.
Since $S_1^{(n+1)} = M$ with probability $p$,
it follows that the test risk of the \mrbase~threshold
$\E[L_1^{(n+1)}(\hat \lambda_1)]$
is bounded below by
$(1-p)^{n_{\text{cal}}} \cdot pM$,
which can grow arbitrarily large.

Based on these results,
we provide the following recommendations.
If in a particular application,
$n$ is large
relative to the sizes of the losses,
one can either use \multirisk~or \mrbase.
If $n$ is small,
then if the user
is in a setting
in which rigorous constraint risk control is imperative,
and if the user does not mind
a slightly conservative threshold,
one should use \multirisk.
Otherwise,
\mrbase~offers a less conservative threshold,
but one which may potentially violate the constraint
by arbitrarily large amounts in expectation.

\section{Proofs}

\subsection{Notation}\label{subsec:notation}

We write a.s. to denote ``almost surely".
For a positive integer $r$,
we use $[r]$
or $1:r$
to denote
$\{1,2,\ldots,r\}$.
For a finite set $S\subseteq \NN$
of positive integers,
with elements given by $s_1<\ldots<s_r$,
we write $a_{S}$ to denote
a sequence $a_{s_1}, \ldots, a_{s_r}$
of real numbers.
Given $c\in \R$,
we write
$ca_{S}$ to denote
the sequence
$ca_{s_1},\ldots, ca_{s_r}$.
We write
$a_{S} \pm b_{S}$ to denote
the sequence
$(a_{s_1}\pm b_{s_1}), \ldots, (a_{s_r}\pm b_{s_r})$.
We write $a_{S} \pce b_{S}$ if
$a_{s_k} \le b_{s_k}$ for $k\in [r]$.
We write $a_{S} \pcn b_{S}$ if
$a_{s_k} < b_{s_k}$ for $k\in [r]$.
We write $a_{S} \sce b_{S}$ if
$a_{s_k} \ge b_{s_k}$ for $k\in [r]$.
We write $a_{S} \scn b_{S}$ if
$a_{s_k} > b_{s_k}$ for $k\in [r]$.
Given $c\in \R\cup\{\pm\infty\}$,
we write $a_{S} \pce c$ if
$a_{s_k} \le c$ for $k\in [r]$.
We define $a_{S} \pcn c$, $a_{S} \sce c$, and $a_{S} \scn c$ similarly.
We write $\nea$ to indicate non-decreasing,
and we write $\sea$ to indicate non-increasing.
Given $f : \R\to \R$,
we let $f(\R)$
denote the range of $f$.
We let
$\|\cdot\|_1$
denote the
$\ell^1$-norm.
Given compacts
$K,K'\subseteq \R$,
we let
$\dist(K,K') 
:= 
\max\{ 
\max_{y\in K} \min_{y'\in K'} |y-y'|,
\max_{y'\in K'} \min_{y\in K} |y-y'|
\}$
denote the Hausdorff distance
between $K$ and $K'$.
Given $p\in \R$,
we write
$\dist(p,K)$
for $\dist(\{p\},K)$.
We use Landau notation $o, O$,
where constants are allowed to
depend on the number of constraints $m$.
We assume that $m$ does not grow with $n$.
 Moreover, we will refer to functions that depend only on the observed data (i.e., to statistics) as measurable functions. 
 Hence, we will say, for instance, that the composition of two measurable functions is measurable, etc. 

\subsection{Monotonicity and right-continuity of empirical risks}\label{subsubsec:prelims}

\begin{lemma}\label{lem:emp-monot}
Assume that
the conditions
in
\Cref{subsubsec:ub-conds}
hold.
Then a.s.,
each of the functions
$\lambda_{1:j}\mapsto g^{\mathrm{sym}}_j(\lambda_j; \lambda_{1:(j-1)})$
and
$\lambda_{1:j}\mapsto g_j^+(\lambda_j; \lambda_{1:(j-1)})$
is $\nea$ in its last $j-1$ arguments
and $\sea$ in its first argument.
Further,
a.s.,
for any $\lambda_{1:(j-1)}\in \Lambda_{1:(j-1)}$, 
the functions
$g^{\mathrm{sym}}_j(\cdot; \lambda_{1:(j-1)})$
and
$g_j^+(\cdot; \lambda_{1:(j-1)})$
are right-continuous.
Finally, note that
a.s.,
for any $\lambda_{1:(j-1)}\in \Lambda_{1:(j-1)}$,
the functions
$U^{\mathrm{sym}}_j(\lambda_{1:(j-1)}; \cdot)$
and
$U_j^+(\lambda_{1:(j-1)}; \cdot)$
are $\sea$,
while
a.s.,
for any $\beta_j\in \R$,
the functions
$\lambda_{1:(j-1)}\mapsto U^{\mathrm{sym}}_j(\lambda_{1:(j-1)}; \beta_j)$
and
$\lambda_{1:(j-1)}\mapsto U_j^+(\lambda_{1:(j-1)}; \beta_j)$
are $\nea$ in each argument.
\end{lemma}

\begin{proof}
Note that
$L_j(\cdot)$ is $\nea$ in its first $j-1$ arguments
and $\sea$ in its last argument,
which implies the monotonicity of
$g_j^+(\cdot; \cdot)$
and
$g^{\mathrm{sym}}_j(\cdot;\cdot)$.
Also note that
$L_j(\cdot)$ is right-continuous in its last argument,
which implies the right-continuity of
$g_j^+(\cdot; \lambda_{1:(j-1)})$
and
$g^{\mathrm{sym}}_j(\cdot; \lambda_{1:(j-1)})$.
By
\Cref{cond:loss-inf},
a.s.,
we have the equality
$\lim_{\lambda_j\to \infty} 
g_j^+(\lambda_j; \lambda_{1:(j-1)})
= 0$,
hence
$U_j^+(\lambda_{1:(j-1)}; \beta_j)$
exists
for all $\lambda_{1:(j-1)}\in \Lambda_{1:(j-1)}$
and $\beta_j > 0$.
Similarly,
a.s.,
$U^{\mathrm{sym}}_j(\lambda_{1:(j-1)}; \beta_j)$
exists
for all $\lambda_{1:(j-1)}\in \Lambda_{1:(j-1)}$
and $\beta_j > 0$.
The definition of $U_j^+(\cdot;\cdot)$
and the monotonicity of $g_j^+(\cdot;\cdot)$
imply the monotonicity of $U_j^+(\cdot;\cdot)$,
and similarly for
$U^{\mathrm{sym}}_j(\cdot;\cdot)$.
\end{proof}

\subsection{Proof of \Cref{lem:symm-fns}}\label{subsec:pf-symm-fns}

\begin{proof}
Fix nonrandom $\lambda_{1:(j-1)} \in \Lambda_{1:(j-1)}$.
By
\Cref{lem:emp-monot}
and
\Cref{cond:loss-inf},
a.s., we have
\begin{align*}
g^{\mathrm{sym}}_j(\lambda_j^{\mathrm{max}}; \lambda_{1:(j-1)})
\le g^{\mathrm{sym}}_j(\lambda_j^{\mathrm{max}}; \lambda_{1:(j-1)}^{\mathrm{max}})
= 
\frac{1}{n+1} \sum_{i=1}^{n+1} L_j^{(i)}(\lambda_{1:j}^{\mathrm{max}})
\le \beta_j.
\end{align*}
By the definition of $U^{\mathrm{sym}}_j(\lambda_{1:(j-1)}; \beta_j)$,
it follows that a.s.,
$g^{\mathrm{sym}}_j(U^{\mathrm{sym}}_j(\lambda_{1:(j-1)}; \beta_j); \lambda_{1:(j-1)}) \le \beta_j$.
By exchangeability from \Cref{cond:exch-observations},
and since $U^{\mathrm{sym}}_j(\lambda_{1:(j-1)}; \beta_j)$ is a symmetric function of the $n+1$ datapoints,
we have 
\begin{align*}
\E L_j^{(n+1)}(\lambda_{1:(j-1)}, U^{\mathrm{sym}}_j(\lambda_{1:(j-1)}; \beta_j)) 
&= \E \frac{1}{n+1} \sum_{i=1}^{n+1} L_j^{(i)}(\lambda_{1:(j-1)}, U^{\mathrm{sym}}_j(\lambda_{1:(j-1)}; \beta_j)) \\
&= \E g^{\mathrm{sym}}_j(U^{\mathrm{sym}}_j(\lambda_{1:(j-1)}; \beta_j); \lambda_{1:(j-1)}).
\end{align*}
Since the integrand
is bounded by $\beta_j$ a.s.,
we deduce
\begin{align}\label{eq:fix-lambda-ineq}
    \E L_j^{(n+1)}(\lambda_{1:(j-1)}, U^{\mathrm{sym}}_j(\lambda_{1:(j-1)}; \beta_j)) \le \beta_j.
\end{align}
For any symmetric function $\Gamma_{1:(j-1)}$ of the data,
the datapoints are conditionally exchangeable given $\Gamma_{1:(j-1)}$.
Since \Cref{eq:fix-lambda-ineq} applies for a fixed $\lambda_{1:(j-1)}\in \Lambda_{1:(j-1)}$
whenever the datapoints are exchangeable, we have
$\E[L_j^{(n+1)}(\Gamma_{1:(j-1)}, U^{\mathrm{sym}}_j(\Gamma_{1:(j-1)}; \beta_j)) | \Gamma_{1:(j-1)}] \le \beta_j$.
Therefore,
\begin{align*}
\E L_j^{(n+1)}(\Gamma_{1:(j-1)}, U^{\mathrm{sym}}_j(\Gamma_{1:(j-1)}; \beta_j)) \le \beta_j,
\end{align*}
as desired.
\end{proof}

\subsection{Comparing generalized inverses of non-increasing functions}

\begin{lemma}[Bounded perturbations imply bounded generalized inverses for non-increasing functions]\label{lem:crossing-points}
Let $I\subseteq \R$ be an interval. 
Suppose that $f_1, f_2 : I \to [0, \infty)$
are non-increasing functions such that
$f_1 \ge f_2 \ge f_1 - C$
for some constant $C>0$. Then for all $\gamma\in \R$,
\begin{align}\label{eq:one-sided-crossing}
    \sup\{ x\in I : f_2(x) > \gamma - C \} \ge \sup\{ x\in I : f_1(x) > \gamma \}.
\end{align}
Similarly,
suppose that $\tilde f_1, \tilde f_2 : I \to [0, \infty)$
are non-increasing functions such that
$\| \tilde f_1 - \tilde f_2 \|_{L^{\infty}(I)} \le \widetilde C$
for some constant $\widetilde C>0$.
Then for all $\gamma\in \R$,
\begin{align}\label{eq:two-sided-crossing}
    \sup\{ x\in I : \tilde f_2(x) > \gamma - \widetilde C \} \ge \sup\{ x\in I : \tilde f_1(x) > \gamma \} \ge \sup\{ x\in I : \tilde f_2(x) > \gamma + \widetilde C \}.
\end{align}
Here, we define
the supremum of an empty subset of $I$ to be $\inf I$. 
\end{lemma}

\begin{proof}

Since we define
the supremum of an empty subset of $I$ to be $\inf I$, 
note that
given subsets
$A,B \subseteq I$
with
$A\subseteq B$,
we necessarily have
$\sup\{A\} \le \sup\{B\}$.
Indeed,
if $A$ is nonempty,
then $B$ is necessarily nonempty,
and the inequality follows.
Otherwise,
if $A$ is empty,
then the inequality vacuously holds.

\textit{First claim:}
If $f_1(x) > \gamma$ for some $x\in I$,
then by assumption
$f_2(x) \ge f_1(x) - C > \gamma - C$.
Thus, we have the inclusion
\begin{align*}
    \{ x\in I : f_1(x) > \gamma \} \subseteq \{ x\in I : f_2(x) > \gamma - C \}.
\end{align*}
The result follows.

\textit{Second claim:}
If $ \tilde f_2(x) > \gamma + \widetilde C $,
then by assumption
and the triangle inequality,
we have
$\tilde f_1(x) \ge \tilde f_2(x) - \widetilde C > \gamma$,
so we have the inclusion
\begin{align*}
    \{ x\in I : \tilde f_2(x) > \gamma + \widetilde C \} \subseteq \{ x\in I : \tilde f_1(x) > \gamma \}.
\end{align*}
Similarly,
if $\tilde f_1(x) > \gamma$,
then by assumption and the triangle inequality,
we have
$\tilde f_2(x) \ge \tilde f_1(x) - \widetilde C > \gamma - \widetilde C$,
so we have the inclusion
\begin{align*}
    \{ x\in I : \tilde f_1(x) > \gamma \} \subseteq \{ x\in I : \tilde f_2(x) > \gamma - \widetilde C \}.
\end{align*}
Putting these together, we have
\begin{align*}
    \{ x\in I : \tilde f_2(x) > \gamma + \widetilde C \} \subseteq \{ x\in I : \tilde f_1(x) > \gamma \} \subseteq \{ x\in I : \tilde f_2(x) > \gamma - \widetilde C \},
\end{align*}
which implies the result. 
\end{proof}

\subsection{Proof of \Cref{lem:beta-trick}}\label{subsec:pf-beta-trick}

\begin{proof}
By 
\Cref{cond:loss-bds},
a.s.,
$L_j^{(n+1)}(\lambda_{1:j}) \in [V^{\mathrm{min}}_j, V^{\mathrm{max}}_j]$ for all $\lambda_{1:j}\in \R^j$.
Thus, a.s., we have the uniform control
\begin{align*}
    0 \le g_j^+(\lambda_j; \lambda_{1:(j-1)}) - g^{\mathrm{sym}}_j(\lambda_j; \lambda_{1:(j-1)}) = \frac{V^{\mathrm{max}}_j - L_j^{(n+1)}(\lambda_{1:j})}{n+1} \le \frac{V^{\mathrm{max}}_j - V^{\mathrm{min}}_j}{n+1} =: \delta_j
\end{align*}
for all $\lambda_j\in \R$.
Applying
\Cref{eq:one-sided-crossing}
from
\Cref{lem:crossing-points}
to the functions
$g_j^+(\cdot;\lambda_{1:(j-1)})$
and
$g^{\mathrm{sym}}_j(\cdot;\lambda_{1:(j-1)})$,
and using
the definitions of $U^{\mathrm{sym}}_j$ and $U_j^+$,
we obtain the result.
\end{proof}

\subsection{Proof of \Cref{thm:multiple-scores}}\label{subsec:pf-multiple-scores}

\begin{proof}
\textit{Item 1:} First, given $m \ge 1$,
we check that
$\hat \lambda_j^{(2k-1)}$ is symmetric
and
$\hat \lambda_j^{(2k)}$ is measurable
for $j\in [m]$
and $k\in [m-j+1]$.
We proceed by induction on $j \ge 1$.
The base case is $j = 1$.
By the definition of $U^{\mathrm{sym}}_1(\beta_1)$,
$\hat \lambda_1^{(1)} = U^{\mathrm{sym}}_1(\beta_1)$ is symmetric.
By the definition of $U_1^+(\beta_1)$,
$\hat \lambda_1^{(2)} = U_1^+(\beta_1)$ is measurable.
For the inductive step, suppose that the result holds for some $j-1 \in [m-1]$.
We show the result for $j$.
Given $k \in [m-j+1]$,
since $U^{\mathrm{sym}}_j(\lambda_{1:(j-1)}; \beta_j^{(k-1)})$ is symmetric for any fixed $\lambda_{1:(j-1)}$,
since $\hat \lambda_{1:(j-1)}^{(2k+1)}$ is symmetric by the inductive hypothesis,
and since the composition of symmetric functions is symmetric,
we have that $\hat \lambda_j^{(2k-1)} = U^{\mathrm{sym}}_j(\hat \lambda_{1:(j-1)}^{(2k+1)}; \beta_j^{(k-1)})$ is symmetric.
Similarly,
since $U_j^+(\lambda_{1:(j-1)}; \beta_j^{(k-1)})$ is measurable for any fixed $\lambda_{1:(j-1)}$,
since $\hat \lambda_{1:(j-1)}^{(2k+2)}$ is measurable by the inductive hypothesis,
and since the composition of measurable functions is measurable,
we have that $\hat \lambda_j^{(2k)} = U_j^+(\hat \lambda_{1:(j-1)}^{(2k+2)}; \beta_j^{(k-1)})$ is measurable.
This completes the induction.

\textit{Item 2:} Next, we check that the inequality chain holds.
We proceed by induction on $j \ge 1$.
The base case is $j = 1$.
By the definition of $U^{\mathrm{sym}}_1(\beta)$ and $U_1^+(\beta)$,
we have $U^{\mathrm{sym}}_1(\beta) \le U_1^+(\beta)$ for all $\beta$ a.s.,
which implies that $\hat \lambda_1^{(2k-1)} \le \hat \lambda_1^{(2k)}$ for $k \in [m]$ a.s.
Further, since
by \Cref{lem:beta-trick},
we have
$U_1^+(\beta^{(k-1)}) \le U^{\mathrm{sym}}_1(\beta^{(k)})$ for all $\beta$ a.s.,
it follows that we have $U_1^+(\beta_1^{(k-1)}) \le U^{\mathrm{sym}}_1(\beta_1^{(k)})$ a.s.,
which implies that $\hat \lambda_1^{(2k)} \le \hat \lambda_1^{(2k+1)}$ for $k\in [m-1]$ a.s.
For the inductive step, suppose that the result holds for some $j-1 \in [m-1]$.
We show the result for $j$.
Given $k \in [m-j+1]$,
by the definition of $U^{\mathrm{sym}}_j(\lambda_{1:(j-1)}; \beta)$ and $U_j^+(\lambda_{1:(j-1)}; \beta)$,
we have $U^{\mathrm{sym}}_j(\lambda_{1:(j-1)}; \beta) \le U_j^+(\lambda_{1:(j-1)}; \beta)$ for all $\lambda_{1:(j-1)}$ and $\beta$ a.s.,
which implies that $\hat \lambda_j^{(2k-1)} \le \hat \lambda_j^{(2k)}$ for $k\in [m-j+1]$ a.s.
Further, since
by \Cref{lem:beta-trick}
we have
$U_j^+(\lambda_{1:(j-1)}; \beta^{(k-1)}) \le U^{\mathrm{sym}}_j(\lambda_{1:(j-1)}; \beta^{(k)})$ for all $\lambda_{1:(j-1)}$, $\beta$, and $k$ a.s.,
we have $U_j^+(\hat \lambda_{1:(j-1)}^{(2k+2)}; \beta_j^{(k-1)}) \le U^{\mathrm{sym}}_j(\hat \lambda_{1:(j-1)}^{(2k+2)}; \beta_j^{(k)})$ for $k \in [m-j]$ a.s.
Since by the inductive hypothesis we have
$\hat \lambda_{\ell}^{(2k+2)} \le \hat \lambda_{\ell}^{(2k+3)}$ for $\ell \in [j-1]$ and $k\in [m-j]$ a.s.,
and since by
the conditions in
\Cref{subsubsec:ub-conds}
we may apply
\Cref{lem:emp-monot}
to deduce that
$U^{\mathrm{sym}}_j(\lambda_{1:(j-1)}; \beta)$ is $\nea$ in $\lambda_{1:(j-1)}$ a.s.,
it follows that
we have $U^{\mathrm{sym}}_j(\hat \lambda_{1:(j-1)}^{(2k+2)}; \beta_j^{(k)}) \le U^{\mathrm{sym}}_j(\hat \lambda_{1:(j-1)}^{(2k+3)}; \beta_j^{(k)})$ a.s.
Putting these inequalities together, we have 
$U_j^+(\hat \lambda_{1:(j-1)}^{(2k+2)}; \beta_j^{(k-1)}) \le U^{\mathrm{sym}}_j(\hat \lambda_{1:(j-1)}^{(2k+3)}; \beta_j^{(k)})$ a.s.,
hence $\hat \lambda_j^{(2k)} \le \hat \lambda_j^{(2k+1)}$ for $k\in [m-j]$ a.s.
This completes the induction.

\textit{Item 3:} Finally, to prove that $(\hat \lambda_{1:j}^{(2)})$ satisfy the first $j$ constraints
for $j\in [m]$,
we proceed by induction on $j$.

The base case is $j=1$.
By
\Cref{lem:symm-fns},
we have
$\E L_1^{(n+1)}(U^{\mathrm{sym}}_1(\beta_1)) \le \beta_1$,
which is equivalent to
$\E L_1^{(n+1)}(\hat \lambda_1^{(1)}) \le \beta_1$.
Since $L_1^{(n+1)}(\lambda_1)$ is $\sea$ in $\lambda_1$ a.s., 
we have
$\E L_1^{(n+1)}(\hat \lambda_1^{(2)}) \le \beta_1$,
hence $\hat \lambda_1^{(2)}$ is a valid threshold for the first $j=1$ constraints.

For the inductive step,
suppose that the result holds for some $j-1\ge 1$ with $j-1\in [m-1]$.
We show the result for $j$.
By the inductive hypothesis,
$(\hat \lambda_{1:(j-1)}^{(2)})$ satisfy the first $j-1$ constraints.
By
\Cref{lem:symm-fns},
we have
that for any symmetric $\Lambda_{1:(j-1)}$-valued functions $\Gamma_{1:(j-1)}$, 
we have
$\E L_j^{(n+1)}(\Gamma_{1:(j-1)}, U^{\mathrm{sym}}_j(\Gamma_{1:(j-1)}; \beta_j)) \le \beta_j$.
Setting $\Gamma_{1:(j-1)} = \hat \lambda_{1:(j-1)}^{(3)}$
and recalling that $\hat \lambda_j^{(1)} = U^{\mathrm{sym}}_j(\hat \lambda_{1:(j-1)}^{(3)}; \beta_j^{(0)})$, we obtain
$\E L_j^{(n+1)}(\hat \lambda_{1:(j-1)}^{(3)}, \hat \lambda_j^{(1)}) \le \beta_j$.
Since $L_j^{(n+1)}$ is $\nea$ in its first $j-1$ arguments
and $\sea$ in its last argument,
since $\hat \lambda_{\ell}^{(2)} \le \hat \lambda_{\ell}^{(3)}$ for $\ell \in [j-1]$ a.s.,
and since $\hat \lambda_j^{(2)} \ge \hat \lambda_j^{(1)}$ a.s.,
we deduce
$\E L_j^{(n+1)}(\hat \lambda_{1:(j-1)}^{(2)}, \hat \lambda_j^{(2)}) \le \beta_j$.
It follows that $(\hat \lambda_{1:j}^{(2)})$ satisfy the $j$ constraints, completing the induction.

\end{proof}


\subsection{Proof of \Cref{lem:lower-bd}}\label{subsec:pf-lower-bd}

Here we give the proof of
\Cref{lem:lower-bd}. It makes use of \Cref{lem:bdd-jumps}, given at the end of this section.

\begin{proof}
Since the assumptions in
\Cref{subsubsec:ub-conds}
and
\Cref{cond:cts-scores}
hold,
we may apply
\Cref{lem:emp-monot}
and
\Cref{lem:bdd-jumps}
to deduce that
$g^{\mathrm{sym}}_j(\cdot; \Gamma_{1:(j-1)})$
is a $\sea$ right-continuous step function
whose jumps are bounded by $\frac{V^{\mathrm{max}}_j}{n+1}$,
and that
$g^{\mathrm{sym}}_j(\lambda_j; \cdot)$ is
$\nea$ in each argument
for each $\lambda_j\in \Lambda_j$.
Since
$g^{\mathrm{sym}}_j(\lambda_j; \cdot)$ is
$\nea$ in each argument,
and
by the assumption that
$L_j^{(i)}(\lambda_{1:j}^{\mathrm{min}}) > \beta$ a.s. for $i\in [n+1]$,
we may bound
\begin{align*}
g^{\mathrm{sym}}_j(\lambda_j^{\mathrm{min}}; \Gamma_{1:(j-1)})
\ge g^{\mathrm{sym}}_j(\lambda_j^{\mathrm{min}}; \lambda_{1:(j-1)}^{\mathrm{min}})
= 
\frac{1}{n+1} \sum_{i=1}^{n+1} L_j^{(i)}(\lambda_{1:j}^{\mathrm{min}})
> \beta.
\end{align*}
Since $g^{\mathrm{sym}}_j(\cdot; \Gamma_{1:(j-1)})$ is right-continuous, 
it follows that $U^{\mathrm{sym}}_j(\Gamma_{1:(j-1)}; \beta) \in (\lambda_j^{\mathrm{min}}, \lambda_j^{\mathrm{max}}]$.
By
\Cref{lem:bdd-jumps}
and
the right-continuity of $g^{\mathrm{sym}}_j(\cdot; \Gamma_{1:(j-1)})$,
for any $c\in (\lambda_j^{\mathrm{min}}, \lambda_j^{\mathrm{max}}]$,
we have
\begin{align*}
    \lim_{\lambda_j \to c^-} g^{\mathrm{sym}}_j(\lambda_j; \Gamma_{1:(j-1)}) - \lim_{\lambda_j \to c} g^{\mathrm{sym}}_j(\lambda_j; \Gamma_{1:(j-1)}) \le \frac{V^{\mathrm{max}}_j}{n+1}.
\end{align*}
Setting $c = U^{\mathrm{sym}}_j(\Gamma_{1:(j-1)}; \beta)$,
we find
\begin{align*}
    \lim_{\lambda_j \to c^-} g^{\mathrm{sym}}_j(\lambda_j; \Gamma_{1:(j-1)}) - g^{\mathrm{sym}}_j(c; \Gamma_{1:(j-1)}) \le \frac{V^{\mathrm{max}}_j}{n+1},
\end{align*}
hence
\begin{align*}
    g^{\mathrm{sym}}_j(c; \Gamma_{1:(j-1)}) \ge \lim_{\lambda_j \to c^-} g^{\mathrm{sym}}_j(\lambda_j; \Gamma_{1:(j-1)}) - \frac{V^{\mathrm{max}}_j}{n+1}.
\end{align*}
By the definition of $U^{\mathrm{sym}}_j$,
for all sufficiently small $\eta > 0$, we have 
$g^{\mathrm{sym}}_j(c - \eta; \Gamma_{1:(j-1)}) > \beta$,
which implies
$\lim_{\lambda_j \to c^-} g^{\mathrm{sym}}_j(\lambda_j; \Gamma_{1:(j-1)}) \ge \beta$.
Combining this with the lower bound on $g^{\mathrm{sym}}_j(c; \Gamma_{1:(j-1)})$,
we deduce
\begin{align*}
    g^{\mathrm{sym}}_j(c; \Gamma_{1:(j-1)}) \ge \beta - \frac{V^{\mathrm{max}}_j}{n+1},
\end{align*}
which proves
\Cref{eq:lower-bd}.
Taking expectations, we obtain
\Cref{eq:lower-bd-expec},
as desired.
\end{proof}

\begin{lemma}\label{lem:bdd-jumps}
Assume that
the conditions
in
\Cref{subsubsec:ub-conds}
and
\Cref{cond:cts-scores}
hold.
Then for $j\in [m]$,
a.s.,
for any $\lambda_{1:(j-1)}\in \R^{j-1}$,
the functions
$g^{\mathrm{sym}}_j(\cdot; \lambda_{1:(j-1)})$
and
$g_j^+(\cdot; \lambda_{1:(j-1)})$
have jumps
bounded above by $\frac{V^{\mathrm{max}}_j}{n+1}$.
\end{lemma}

\begin{proof}
Fix $j\in [m]$.
Since the conditions in
\Cref{subsubsec:ub-conds}
hold,
we may apply
\Cref{lem:emp-monot}
to deduce that
a.s.,
for any 
$\lambda_{1:(j-1)}\in \R^{j-1}$,
$g^{\mathrm{sym}}_j(\cdot; \lambda_{1:(j-1)})$
and
$g_j^+(\cdot; \lambda_{1:(j-1)})$
are right-continuous
and non-increasing.
By
\Cref{cond:cts-scores},
$\{ S_j^{(i)} : i\in [n+1] \}$
are a.s. distinct.
By
\Cref{cond:loss-bds},
for $i\in [n+1]$,
$\lambda_j \mapsto L_j^{(i)}(\lambda_{1:j})$
has a single jump
at $\lambda_j = S_j^{(i)}$
of size bounded by $V^{\mathrm{max}}_j$.
Thus,
a.s.,
\begin{align*}
    g^{\mathrm{sym}}_j(\lambda_j; \lambda_{1:(j-1)}) = \frac{1}{n+1} \sum_{i=1}^{n+1} L_j^{(i)}(\lambda_{1:j})
\end{align*}
has $n+1$ distinct jumps
of sizes bounded by $\frac{V^{\mathrm{max}}_j}{n+1}$.
Similarly,
a.s.,
\begin{align*}
    g_j^+(\lambda_j; \lambda_{1:(j-1)}) = \frac{1}{n+1} \sum_{i=1}^{n} L_j^{(i)}(\lambda_{1:j}) + \frac{V^{\mathrm{max}}_j}{n+1}
\end{align*}
has $n$ distinct jumps
of sizes bounded by $\frac{V^{\mathrm{max}}_j}{n+1}$,
as claimed.
\end{proof}


\subsection{Proof of \Cref{thm:reverse-ineq}}\label{subsec:pf-reverse-ineq}

Here we give the proof of \Cref{thm:reverse-ineq}. It makes use of the counting result \Cref{lem:count-scores} given at the end of this section.

\begin{proof}
Denote the statement to be proved by
$P(j,k,s)$
for $j\in [m]$
and $k\ge s\ge 0$.
We proceed by strong induction on $j$.

\textit{Base case:}
For $j=1$, it suffices to check that
$U^{\mathrm{sym}}_1(\beta) \le U^{\mathrm{sym}}_1\left( \beta - \frac{h_1(k-s)}{n+1} \right)$ a.s.
for all $\beta\ge 0$,
which holds with equality.

\textit{Inductive step:}
Fix $j\in \{2,\ldots, m\}$
and
$\beta\ge 0$.
Assume that $P(\ell,k,s)$ holds
for all
$\ell\in [j-1]$.
Fix integers $k\ge s\ge 0$
such that
for $\ell\in [j-1]$,
for $i\in [n+1]$,
we have the bounds
$L_{\ell}^{(i)}(\lambda_{1:\ell}^{\mathrm{min}}) > \beta_{\ell}^{(s)}$
a.s.
and
$L_{\ell}^{(i)}(\lambda_{1:\ell}^{\mathrm{max}}) \le \beta_{\ell}^{(k+j-\ell)} - \frac{h_{\ell}(k-s)}{n+1}$
a.s.
We shall prove the statement $P(j,k,s)$.

Consider the step functions
$g^{\mathrm{sym}}_j(\cdot; \hat \lambda_{1:(j-1)}^{(2k+1)})$
and
$g^{\mathrm{sym}}_j(\cdot; \hat \lambda_{1:(j-1)}^{(2s+1)})$.
Since the conditions in
\Cref{subsubsec:ub-conds} hold,
\Cref{thm:multiple-scores}
implies that
$\hat \lambda_{\ell}^{(2k+1)} \ge \hat \lambda_{\ell}^{(2s+1)}$ a.s. for all $\ell \in [j-1]$.
Since the conditions in
\Cref{subsubsec:ub-conds} hold,
\Cref{lem:emp-monot} holds
and implies that
$g^{\mathrm{sym}}_j(\lambda_j; \cdot)$
is $\nea$
in each argument
for each $\lambda_j\in \Lambda_j$.
Combining these two facts,
we deduce that
$g^{\mathrm{sym}}_j(\cdot; \hat \lambda_{1:(j-1)}^{(2k+1)}) \ge g^{\mathrm{sym}}_j(\cdot; \hat \lambda_{1:(j-1)}^{(2s+1)})$ a.s.
The a.s. non-negative difference $\Delta(\cdot) = g^{\mathrm{sym}}_j(\cdot; \hat \lambda_{1:(j-1)}^{(2k+1)}) 
- g^{\mathrm{sym}}_j(\cdot; \hat \lambda_{1:(j-1)}^{(2s+1)})$ is given by
\begin{align*}
    \Delta(\lambda_j) = \frac{1}{n+1} \sum_{i \in \ii} V_j^{(i)} I(S_j^{(i)} > \lambda_j)
\end{align*}
for all $\lambda_j\in \Lambda_j$.
Here, the index set $\ii$ is defined as 
\begin{align*}
    \ii = 
    \{ i\in [n+1] : S_{1:(j-1)}^{(i)}\pce \hat \lambda_{1:(j-1)}^{(2k+1)} \} 
    \setminus 
    \{ i\in [n+1] : S_{1:(j-1)}^{(i)}\pce \hat \lambda_{1:(j-1)}^{(2s+1)} \}.
\end{align*}
Note that
$\ii \subseteq \bigsqcup_{\ell=1}^{j-1} \ii_{\ell}$,
where for $\ell \in [j-1]$ we define
\begin{align*}
    \ii_{\ell} = \{ i \in [n+1] : 
    S_{1:(\ell-1)}^{(i)} \pce \hat \lambda_{1:(\ell-1)}^{(2s+1)}, 
    S_{\ell}^{(i)} \in ( \hat \lambda_{\ell}^{(2s+1)}, \hat \lambda_{\ell}^{(2k+1)} ] \}.
\end{align*}
Since by
\Cref{cond:loss-bds}
we have
$|V_j^{(i)}| \le V^{\mathrm{max}}_j$
a.s.
for $i\in [n+1]$,
the triangle inequality implies the uniform bound
$\|\Delta\|_{L^{\infty}(\R)} \le \frac{1}{n+1} V^{\mathrm{max}}_j |\ii|$ a.s.
If we can bound the right-hand side of this inequality
by a nonrandom constant $C$ a.s.,
then by
\Cref{lem:emp-monot}
and by 
\Cref{eq:one-sided-crossing}
from
\Cref{lem:crossing-points},
we have
\begin{align*}
U^{\mathrm{sym}}_j( \hat \lambda_{1:(j-1)}^{(2k+1)}; \beta )
\le U^{\mathrm{sym}}_j( \hat \lambda_{1:(j-1)}^{(2s+1)}; \beta - C ),
\end{align*}
so that we may set $h_j(k-s) = (n+1) C$
and conclude.
It therefore suffices to bound $|\ii|$.
By the union bound, we have $|\ii| \le \sum_{\ell=1}^{j-1} |\ii_{\ell}|$.
Fix $\ell \in [j-1]$.
Recall that
$\hat \lambda_{\ell}^{(2s+1)} 
= U^{\mathrm{sym}}_{\ell}(\hat \lambda_{1:(\ell-1)}^{(2s+3)}; \beta_{\ell}^{(s)})$
and
$\hat \lambda_{\ell}^{(2k+1)} 
= U^{\mathrm{sym}}_{\ell}(\hat \lambda_{1:(\ell-1)}^{(2k+3)}; \beta_{\ell}^{(k)})$.

By assumption,
for
$\ell'\in [\ell-1]$,
for $i\in [n+1]$,
we have the bounds
$L_{\ell'}^{(i)}(\lambda_{1:\ell'}^{\mathrm{min}})
> \beta_{\ell'}^{(s)}$ a.s.
and
$L_{\ell'}^{(i)}(\lambda_{1:\ell'}^{\mathrm{max}})
\le \beta_{\ell'}^{(k+j-\ell')} - \frac{h_{\ell'}(k-s)}{n+1}$ a.s.
Note that
$\beta_{\ell'}^{(s)} \ge \beta_{\ell'}^{(s+1)}$,
note that $\ell\le j-1$ implies that
$\beta_{\ell'}^{(k+j-\ell')} \le \beta_{\ell'}^{((k+1)+\ell-\ell')}$,
and note that
$\frac{h_{\ell'}(k-s)}{n+1}
= \frac{h_{\ell'}((k+1)-(s+1))}{n+1}$.
Thus,
for $\ell'\in [\ell-1]$,
for $i\in [n+1]$,
we have the bounds
$L_{\ell'}^{(i)}(\lambda_{1:\ell'}^{\mathrm{min}})
> \beta_{\ell'}^{(s+1)}$ a.s.
and
$L_{\ell'}^{(i)}(\lambda_{1:\ell'}^{\mathrm{max}})
\le \beta_{\ell'}^{((k+1)+\ell-\ell')} - \frac{h_{\ell'}((k+1)-(s+1))}{n+1}$ a.s.
Consequently,
we may apply the inductive hypothesis
to deduce the bound
\begin{align*}
    \hat \lambda_{\ell}^{(2k+1)} 
    = U^{\mathrm{sym}}_{\ell}(\hat \lambda_{1:(\ell-1)}^{(2k+3)}; \beta_{\ell}^{(k)}) &\le U^{\mathrm{sym}}_{\ell}\left( \hat \lambda_{1:(\ell-1)}^{(2s+3)}; \beta_{\ell}^{(k)} - \frac{h_{\ell}((k+1) - (s+1))}{n+1} \right) \\
    &= U^{\mathrm{sym}}_{\ell}\left( \hat \lambda_{1:(\ell-1)}^{(2s+3)}; \beta_{\ell}^{(k)} - \frac{h_{\ell}(k-s)}{n+1} \right).
\end{align*}
Thus, we have the inclusion
\begin{align*}
    ( \hat \lambda_{\ell}^{(2s+1)}, \hat \lambda_{\ell}^{(2k+1)} ]
    \subseteq
    \left(
    U^{\mathrm{sym}}_{\ell}(\hat \lambda_{1:(\ell-1)}^{(2s+3)}; \beta_{\ell}^{(s)}),
    U^{\mathrm{sym}}_{\ell}\left( \hat \lambda_{1:(\ell-1)}^{(2s+3)}; \beta_{\ell}^{(k)} - \frac{h_{\ell}(k-s)}{n+1} \right) \right].
\end{align*}
Since
\Cref{thm:multiple-scores}
implies $\hat \lambda_{1:(\ell-1)}^{(2s+1)}
\pce
\hat \lambda_{1:(\ell-1)}^{(2s+3)}$,
we also have that 
$S_{1:(\ell-1)}^{(i)} \pce \hat \lambda_{1:(\ell-1)}^{(2s+1)}$
implies
$S_{1:(\ell-1)}^{(i)} \pce \hat \lambda_{1:(\ell-1)}^{(2s+3)}$. 
Combining the previous two observations,
we obtain the bound
\begin{align*}
    |\ii_{\ell}| \le \left| \left\{ i \in [n+1] : 
    S_{1:(\ell-1)}^{(i)} \pce \hat \lambda_{1:(\ell-1)}^{(2s+3)},
    S_{\ell}^{(i)} \in \left(
    U^{\mathrm{sym}}_{\ell}(\hat \lambda_{1:(\ell-1)}^{(2s+3)}; \beta_{\ell}^{(s)}),
    U^{\mathrm{sym}}_{\ell}\left( \hat \lambda_{1:(\ell-1)}^{(2s+3)}; \beta_{\ell}^{(k)} - \frac{h_{\ell}(k-s)}{n+1} \right)
    \right] \right\} \right|.
\end{align*}
By assumption,
for $\ell\in [j-1]$,
for $i\in [n+1]$,
we have the bounds
$L_{\ell}^{(i)}(\lambda_{1:\ell}^{\mathrm{min}}) > \beta_{\ell}^{(s)}$
a.s.
and
$L_{\ell}^{(i)}(\lambda_{1:\ell}^{\mathrm{max}}) \le \beta_{\ell}^{(k+j-\ell)} - \frac{h_{\ell}(k-s)}{n+1}$
a.s.
Thus,
since the conditions in
\Cref{subsubsec:ub-conds},
\Cref{cond:cts-scores},
and
\Cref{cond:m-positive}
hold,
we may apply
\Cref{lem:count-scores}
to obtain the bound
\begin{align*}
    |\ii_{\ell}| &\le \frac{\beta_{\ell}^{(s)} - \left( \beta_{\ell}^{(k)} - \frac{h_{\ell}(k-s)}{n+1} \right) + \frac{V^{\mathrm{max}}_{\ell}}{n+1}}{V^{\mathrm{min}}_{\ell}/(n+1)} \\
    &= \frac{ \frac{(k-s) (V^{\mathrm{max}}_{\ell} - V^{\mathrm{min}}_{\ell})}{n+1} + \frac{h_{\ell}(k-s)}{n+1} + \frac{V^{\mathrm{max}}_{\ell}}{n+1}}{V^{\mathrm{min}}_{\ell}/(n+1)} \\
    &= \frac{(k-s) (V^{\mathrm{max}}_{\ell} - V^{\mathrm{min}}_{\ell}) + h_{\ell}(k-s) + V^{\mathrm{max}}_{\ell}}{V^{\mathrm{min}}_{\ell}},
\end{align*}
where in the second step
we recalled that
$\beta_{\ell}^{(s)} - \beta_{\ell}^{(k)} = (k-s) \delta_{\ell}
= \frac{(k-s) (V^{\mathrm{max}}_{\ell} - V^{\mathrm{min}}_{\ell})}{n+1}$. 
Plugging this into the union bound, we find
\begin{align*}
    |\ii| \le \sum_{\ell=1}^{j-1} \frac{(k-s) (V^{\mathrm{max}}_{\ell} - V^{\mathrm{min}}_{\ell}) + h_{\ell}(k-s) + V^{\mathrm{max}}_{\ell}}{V^{\mathrm{min}}_{\ell}}.
\end{align*}
This implies the nonrandom uniform bound
\begin{align*}
    \|\Delta\|_{L^{\infty}(\R)} \le \frac{1}{n+1} V^{\mathrm{max}}_j \sum_{\ell=1}^{j-1} \frac{(k-s) (V^{\mathrm{max}}_{\ell} - V^{\mathrm{min}}_{\ell}) + h_{\ell}(k-s) + V^{\mathrm{max}}_{\ell}}{V^{\mathrm{min}}_{\ell}}.
\end{align*}
We may apply 
\Cref{eq:one-sided-crossing}
from
\Cref{lem:crossing-points}
to set
\begin{align*}
    h_j(k-s) = V^{\mathrm{max}}_j \sum_{\ell=1}^{j-1} \frac{(k-s) (V^{\mathrm{max}}_{\ell} - V^{\mathrm{min}}_{\ell}) + h_{\ell}(k-s) + V^{\mathrm{max}}_{\ell}}{V^{\mathrm{min}}_{\ell}},
\end{align*}
which completes the induction.

\end{proof}

\begin{lemma}\label{lem:count-scores}

Assume that the conditions
in \Cref{subsubsec:ub-conds},
\Cref{cond:cts-scores},
and
\Cref{cond:m-positive}
hold.
Fix $\ell \in [m]$.
Suppose that
$\beta\ge \beta'\ge 0$
are such that
for $i\in [n+1]$,
we have
$L_{\ell}^{(i)}(\lambda_{1:\ell}^{\mathrm{min}}) > \beta$ a.s.
and
$L_{\ell}^{(i)}(\lambda_{1:\ell}^{\mathrm{max}}) \le \beta'$ a.s.
Then given $\lambda_{1:(\ell-1)} \in \Lambda_{1:(\ell-1)}$,
we have the bound
\begin{align*}
    |\{ i \in [n+1] : 
    S_{1:(\ell-1)}^{(i)} \pce \lambda_{1:(\ell-1)},
    S_{\ell}^{(i)} \in ( U^{\mathrm{sym}}_{\ell}(\lambda_{1:(\ell-1)}; \beta), U^{\mathrm{sym}}_{\ell}(\lambda_{1:(\ell-1)}; \beta') ] \}| \le \frac{\beta - \beta' + \frac{V^{\mathrm{max}}_{\ell}}{n+1}}{V^{\mathrm{min}}_{\ell} / (n+1)}.
\end{align*}


\end{lemma}

\begin{proof}

By the definition of $L_{\ell}^{(i)}(\lambda_{1:\ell})$,
each $i\in [n+1]$ satisfying
$S_{1:(\ell-1)}^{(i)} \pce \lambda_{1:(\ell-1)}$
and
$S_{\ell}^{(i)} \in ( U^{\mathrm{sym}}_{\ell}(\lambda_{1:(\ell-1)}; \beta), U^{\mathrm{sym}}_{\ell}(\lambda_{1:(\ell-1)}; \beta') ]$
corresponds to a downwards jump
in the function
$g^{\mathrm{sym}}_{\ell}(\cdot; \lambda_{1:(\ell-1)})$
of size at least $\frac{V^{\mathrm{min}}_{\ell}}{n+1}$
on the interval
$( U^{\mathrm{sym}}_{\ell}(\lambda_{1:(\ell-1)}; \beta), U^{\mathrm{sym}}_{\ell}(\lambda_{1:(\ell-1)}; \beta') ]$.
Note that since
for $i\in [n+1]$,
we have
$L_j^{(i)}(\lambda_{1:j}^{\mathrm{min}}) > \beta$ a.s. 
and
$L_j^{(i)}(\lambda_{1:j}^{\mathrm{max}}) \le \beta'$ a.s.,
and since the conditions in
\Cref{subsubsec:ub-conds}
and
\Cref{cond:cts-scores}
hold,
the definition of
$U^{\mathrm{sym}}_{\ell}$
and
\Cref{eq:lower-bd}
from
\Cref{lem:lower-bd}
imply that a.s.,
\begin{align*}
     g^{\mathrm{sym}}_{\ell}(U^{\mathrm{sym}}_{\ell}(\lambda_{1:(\ell-1)}; \beta); \lambda_{1:(\ell-1)}) \le \beta
\end{align*}
and
\begin{align*}
     g^{\mathrm{sym}}_{\ell}(U^{\mathrm{sym}}_{\ell}(\lambda_{1:(\ell-1)}; \beta'); \lambda_{1:(\ell-1)}) \ge \beta' - \frac{V^{\mathrm{max}}_{\ell}}{n+1},
\end{align*}
respectively.
Since the conditions in
\Cref{subsubsec:ub-conds} hold,
\Cref{lem:emp-monot}
implies that a.s.
$g^{\mathrm{sym}}_{\ell}(\cdot;\lambda_{1:(\ell-1)})$
is $\sea$.
Thus,
for all $\lambda \le \lambda'$ with
\begin{align*}
    \lambda, \lambda' \in ( U^{\mathrm{sym}}_{\ell}(\lambda_{1:(\ell-1)}; \beta), U^{\mathrm{sym}}_{\ell}(\lambda_{1:(\ell-1)}; \beta') ],
\end{align*}
we have
\begin{align*}
    0 &\le g^{\mathrm{sym}}_{\ell}(\lambda; \lambda_{1:(\ell-1)}) - g^{\mathrm{sym}}_{\ell}(\lambda'; \lambda_{1:(\ell-1)}) \\
    &\le g^{\mathrm{sym}}_{\ell}(U^{\mathrm{sym}}_{\ell}(\lambda_{1:(\ell-1)}; \beta); \lambda_{1:(\ell-1)}) 
    - g^{\mathrm{sym}}_{\ell}(U^{\mathrm{sym}}_{\ell}(\lambda_{1:(\ell-1)}; \beta'); \lambda_{1:(\ell-1)}) \\
    &\le \beta - \beta' + \frac{V^{\mathrm{max}}_{\ell}}{n+1}.
\end{align*}
By
\Cref{cond:m-positive},
we have
$\frac{V^{\mathrm{min}}_{\ell}}{n+1} > 0$,
so the number of downwards jumps in 
$g^{\mathrm{sym}}_{\ell}(\cdot; \lambda_{1:(\ell-1)})$
on this interval
is bounded by
$\frac{\beta - \beta' + \frac{V^{\mathrm{max}}_{\ell}}{n+1}}{V^{\mathrm{min}}_{\ell} / (n+1)}$,
as desired.
\end{proof}


\subsection{Proof of \Cref{cor:constr-lower-bds}}\label{subsec:pf-constr-lower-bds}

\begin{proof}

Fix $j \in [m]$,
and consider constraint $j$.
Our chosen thresholds $(\hat \lambda_{1:j}^{(2)})$ obey
\begin{align*}
    \E L_j^{(n+1)}(\hat \lambda_{1:(j-1)}^{(2)}, \hat \lambda_j^{(2)}) \le \beta_j.
\end{align*}
Since 
$L_j^{(n+1)}$ is $\nea$ in its first $j-1$ arguments
and $\sea$ in its last argument
a.s.,
and since $\hat \lambda_{\ell}^{(k)}$ is non-decreasing in $k$
a.s.,
the left-hand side is bounded below by
$\E L_j^{(n+1)}(\hat \lambda_{1:(j-1)}^{(1)}, \hat \lambda_j^{(3)})$.
We seek $\tilde \beta_j\ge 0$ such that
$\hat \lambda_j^{(3)} \le U^{\mathrm{sym}}_j(\hat \lambda_{1:(j-1)}^{(1)}; \tilde \beta_j)$
a.s.
Since by definition 
$\hat \lambda_j^{(3)} = U^{\mathrm{sym}}_j(\hat \lambda_{1:(j-1)}^{(5)}; \beta_j^{(1)})$,
this is equivalent to
\begin{align*}
    U^{\mathrm{sym}}_j(\hat \lambda_{1:(j-1)}^{(5)}; \beta_j^{(1)}) \le U^{\mathrm{sym}}_j(\hat \lambda_{1:(j-1)}^{(1)}; \tilde \beta_j).
\end{align*}
Since the conditions in
\Cref{subsubsec:ub-conds}
and
\Cref{subsubsec:lb-conds}
hold,
we may apply
\Cref{thm:reverse-ineq}
to deduce that this holds a.s.
when $\tilde \beta_j = \beta_j^{(1)} - \frac{h_j(2)}{n+1}$.
Since $\hat \lambda_j^{(3)} \le U_j(\hat \lambda_{1:(j-1)}^{(1)}; \tilde \beta_j)$ a.s.,
and since 
$L_j^{(n+1)}$ is $\sea$ in its last argument a.s.,
we deduce
\begin{align*}
     \E L_j^{(n+1)}(\hat \lambda_{1:(j-1)}^{(2)}, \hat \lambda_j^{(2)}) \ge \E L_j^{(n+1)}(\hat \lambda_{1:(j-1)}^{(1)}, \hat \lambda_j^{(3)}) \ge \E L_j^{(n+1)}(\hat \lambda_{1:(j-1)}^{(1)}, U^{\mathrm{sym}}_j(\hat \lambda_{1:(j-1)}^{(1)}; \tilde \beta_j)).
\end{align*}
Since the conditions in
\Cref{subsubsec:ub-conds},
\Cref{cond:cts-scores},
and
\Cref{cond:loss-sup} hold,
we may apply
\Cref{lem:lower-bd}
to deduce the lower bound
\begin{align*}
    \tilde \beta_j - \frac{V^{\mathrm{max}}_j}{n+1} 
    &= \beta_j^{(1)} - \frac{h_j(2)}{n+1} - \frac{V^{\mathrm{max}}_j}{n+1} 
    = \beta_j - \frac{V^{\mathrm{max}}_j - V^{\mathrm{min}}_j}{n+1} - \frac{h_j(2)}{n+1} - \frac{V^{\mathrm{max}}_j}{n+1} \\
    &= \beta_j - \frac{2V^{\mathrm{max}}_j - V^{\mathrm{min}}_j + h_j(2)}{n+1},
\end{align*}
where in the second step we used the definition $\beta_j^{(1)} = \beta_j - \delta_j$, as desired.
\end{proof}

\subsection{Population-level analysis}

Here, we study the population-level problem \Cref{eq:opt}.

\begin{lemma}\label{lem:pop-monot}
Assume that
the conditions in
\Cref{subsubsec:ub-conds}
hold.
Then
$\lambda_{1:j}\mapsto 
g_j^*(\lambda_j; \lambda_{1:(j-1)})$
is $\nea$ in its last $j-1$ arguments
and $\sea$ in its first argument,
and for any $\lambda_{1:(j-1)}\in \R^{j-1}$,
the function
$g_j^*(\cdot; \lambda_{1:(j-1)})$
is right-continuous.
Also,
for any $\lambda_{1:(j-1)}\in \R^{j-1}$,
the function
$U_j^*(\lambda_{1:(j-1)}; \cdot)$
is $\sea$,
while
for any $\beta_j\in \R$,
the function
$\lambda_{1:(j-1)}\mapsto U_j^*(\lambda_{1:(j-1)}; \beta_j)$
is $\nea$ in each argument.
\end{lemma}

\begin{proof}
Note that $L_j(\cdot)$ is $\nea$ in its first $j-1$ arguments
and $\sea$ in its last argument,
so the definition of
$g_j^*(\cdot; \cdot)$
implies the monotonicity
of $g_j^*(\cdot;\cdot)$.
Also, $L_j(\cdot)$ is right-continuous in its last argument,
so the definition of
$g_j^*(\cdot; \cdot)$
implies the right-continuity
of $g_j^*(\cdot; \lambda_{1:(j-1)})$.
By
\Cref{cond:loss-inf},
\Cref{cond:loss-bds},
and the dominated convergence theorem,
we have
$\lim_{\lambda_j\to \infty} g_j^*(\lambda_j; \lambda_{1:(j-1)})
= 0$,
hence
$U_j^*(\lambda_{1:(j-1)}; \beta_j)$
exists
for all $\lambda_{1:(j-1)}\in \R^{j-1}$
and all $\beta_j > 0$.
The definition of $U_j^*(\cdot;\cdot)$
and the monotonicity of $g_j^*(\cdot;\cdot)$
imply the monotonicity of $U_j^*(\cdot;\cdot)$.
\end{proof}


\begin{lemma}[Minimizer of population-level problem]\label{lem:pop-minimizer}
Assume that
the conditions in
\Cref{subsubsec:ub-conds}
hold.
Then
\Cref{eq:opt} is feasible.
Further,
let $\lambda_1^* = U_1^*(\beta_1)$,
and for $j\in \{2,\ldots,m\}$,
define $\lambda_j^*$ recursively by
$\lambda_j^* = U_j^*(\lambda_{1:(j-1)}^*; \beta_j)$.
Suppose that 
$\lambda_j^* \in (\lambda_j^{\mathrm{min}}, \lambda_j^{\mathrm{max}})$ 
for all $j\in [m]$.
Then $\lambda_{1:m}^*$ is a minimizer of
\Cref{eq:opt}.
\end{lemma}

\begin{proof}
To see that
\Cref{eq:opt}
is feasible,
note that by
\Cref{cond:loss-inf},
$\lambda_{1:m}^{\mathrm{max}}\in \Lambda_{1:m}$
obeys
$\E L_j(\lambda_{1:j})
\le \beta_j$
for all $j\in [m]$,
hence
$\lambda_{1:m}^{\mathrm{max}}$ is feasible.

Next, we show that
for any feasible $\lambda_{1:m}'\in \Lambda_{1:m}$,
we have $\lambda_j \ge \lambda_j^*$
for $j\in [m]$.
We proceed by induction on $j$.

\textit{Base case:} 
Consider the case $j=1$.
Since $U_1^*(\beta_1) \in (\lambda_1^{\mathrm{min}}, \lambda_1^{\mathrm{max}})$,
and since by
\Cref{lem:pop-monot}
$g_1^*(\cdot)$
is $\sea$,
the first constraint reads
$\lambda_1 \ge U_1^*(\beta_1) =: \lambda_1^*$.
In particular,
if $\lambda_{1:m}'$ is feasible,
we must have $\lambda_1' \ge \lambda_1^*$,
as claimed.

\textit{Inductive step:}
Suppose that we have proved the claim
for all $\ell\in [j-1]$
for some $j\in \{2,\ldots,m\}$.
We show the claim for $j$.
Since $\lambda_{1:m}'$ is feasible,
we have
$g_j^*(\lambda_j'; \lambda_{1:(j-1)}') \le \beta_j$.
Since by
\Cref{lem:pop-monot}
$g_j^*(\lambda_j';\cdot)$ is $\nea$ in each argument,
and since by the inductive hypothesis
we have $\lambda_{1:(j-1)}' \sce \lambda_{1:(j-1)}^*$,
we deduce that
$g_j^*(\lambda_j'; \lambda_{1:(j-1)}^*) \le \beta_j$.
Since by assumption,
$\lambda_j^* := U_j^*(\lambda_{1:(j-1)}^*; \beta_j)
\in (\lambda_j^{\mathrm{min}}, \lambda_j^{\mathrm{max}})$,
and since
by \Cref{lem:pop-monot}
$g_j^*(\cdot; \lambda_{1:(j-1)}^*)$ is $\nea$,
it follows that
$\lambda_j' \ge \lambda_j^*$,
completing the induction.

Since the objective is $\nea$
in $\lambda_j$ for all $j\in [m]$,
it follows that $\lambda_{1:m}^*$
is a minimizer.
\end{proof}

\subsection{Proof of \Cref{thm:cts-conc}}\label{subsec:pf-cts-conc}

In what follows
in this section,
when we write ``whp", we mean with probability $1-O(1/\sqrt n)$.

\subsubsection{Reduction to consistency of thresholds}\label{subsubsec:reduction}

We begin by reducing the result to establishing consistency of the estimated thresholds. Specifically, we perform two inductive arguments:
\begin{enumerate}
    \item \textbf{Preliminary induction:} for each $j\in [m]$,
    for each $k\in [m-j+1]$,
    we show the deterministic bound
    \begin{align}
        |\lambda_j^{*,(2k)} - \lambda_j^{*,(2k+2)}| \le b_j^{(2k)}/n^{1/\nu^j}
    \end{align}
    for some constants $b_j^{(2k)} \ge 0$ defined recursively
    satisfying
    $b_j^{(2k)} = 0$ for $j\in \mathcal{I}_{\mathrm{discrete}}$.
    This proceeds via induction on $j$.

    \item \textbf{Main induction:} For each $j\in [m]$, for each $k\in [m-j+1]$, we show that
    \begin{align}
    |\hat \lambda_j^{(2k)} - \lambda_j^{*,(2k)}| 
    \le 
    c_j^{(2k)} (\log n / n)^{1/(2\nu^j)}
    \end{align}
    whp
    for some constants $c_j^{(2k)} \ge 0$ defined recursively
    satisfying
    $c_j^{(2k)} = 0$ for $j\in \mathcal{I}_{\mathrm{discrete}}$.
    This also proceeds by induction on $j$.
    \end{enumerate}

Here, we show how the main induction implies the result.
By
\Cref{cond:beta-tilde-exists}
and
\Cref{lem:pop-minimizer},
it suffices to prove the result for
$\lambda_{1:m}^{*} = \lambda_{1:m}^{*,(2)}$.
We may relate the objective
evaluated at the chosen thresholds $(\hat \lambda_{1:m}^{(2)})$
to the objective
evaluated at $(\lambda_{1:m}^{*,(2)})$
via a swapping trick, namely
\begin{align}\label{eq:swapping-trick}
    &|\E V_{m+1}^{(n+1)} I(S_{1:m}^{(n+1)} \pce \hat \lambda_{1:m}^{(2)}) - \E V_{m+1}^{(n+1)} I(S_{1:m}^{(n+1)} \pce \lambda_{1:m}^{*,(2)})| \\
    &\le \sum_{j=1}^m | \E V_{m+1}^{(n+1)} I(S_{1:(j-1)}^{(n+1)} \pce \lambda_{1:(j-1)}^{*,(2)}, S_j^{(n+1)} \in ( \min\{\hat \lambda_j^{(2)},  \lambda_j^{*,(2)}\}, \max\{\hat \lambda_j^{(2)},  \lambda_j^{*,(2)}\} ], \nonumber \\
    & \qquad \qquad \qquad \qquad S_{(j+1):m}^{(n+1)} \pce \hat \lambda_{(j+1):m}^{(2)}) | \nonumber \\ 
    &\le V^{\mathrm{max}}_{m+1} \sum_{j=1}^m \PP{S_j^{(n+1)} \in ( \min\{\hat \lambda_j^{(2)},  \lambda_j^{*,(2)}\}, \max\{\hat \lambda_j^{(2)},  \lambda_j^{*,(2)}\}] }, \nonumber 
\end{align}
where in the second step
we used \Cref{cond:obj-loss-bds}.
In order to analyze a bound of this form,
we need
control of the random quantities
$\hat \lambda_j^{(2)} - \lambda_j^{*,(2)}$ for $j\in [m]$.

If we can show that
$|\hat \lambda_j^{(2)} - \lambda_j^{*,(2)}|$ is $O(\left( \frac{\log n}{n} \right)^{1/(2\nu^{j})})$
with probability $1-O(1/\sqrt n)$
for each $j\in \mathcal{I}_{\mathrm{cts}}$,
and
if we can show that
$|\hat \lambda_j^{(2)} - \lambda_j^{*,(2)}| = 0$
with probability $1-O(1/\sqrt n)$
for each $j\in \mathcal{I}_{\mathrm{discrete}}$,
then using \Cref{cond:cdf-lip},
we claim that the previous
display is also $O(\left( \frac{\log n}{n} \right)^{1/(2\nu^{m})})$.
Indeed,
define the event
\begin{align*}
    E_j^{(2)} = \left\{ |\hat \lambda_j^{(2)} - \lambda_j^{*,(2)}| 
    \le c_j^{(2)} 
    \left( \frac{\log n}{n} \right)^{1/(2\nu^{j})} 
    \right\}
\end{align*}
for $j\in [m]$,
where
for $j\in \mathcal{I}_{\mathrm{discrete}}$
we have $c_j^{(2)} = 0$,
and for $j\in \mathcal{I}_{\mathrm{cts}}$
we have that $c_j^{(2)} \ge 0$
is a nonrandom constant.
Then
for $j\in \mathcal{I}_{\mathrm{cts}}$,
we can write
\begin{align}\label{eq:split-on-event}
    &\PP{S_j^{(n+1)} \in ( \min\{\hat \lambda_j^{(2)},  \lambda_j^{*,(2)}\}, \max\{\hat \lambda_j^{(2)},  \lambda_j^{*,(2)}\}] } \\
    &\le \PP{ \{ S_j^{(n+1)} \in ( \min\{\hat \lambda_j^{(2)},  \lambda_j^{*,(2)}\}, \max\{\hat \lambda_j^{(2)},  \lambda_j^{*,(2)}\}] \} \cap E_j^{(2)}} + \PP{{\left( E_j^{(2)} \right)}^c} \nonumber \\
    &\le \PP{S_j^{(n+1)} \in 
    \left( \lambda_j^{*,(2)} 
    - c_j^{(2)} 
    \left( \frac{\log n}{n} \right)^{1/(2\nu^{j})}
    , \lambda_j^{*,(2)} 
    + c_j^{(2)} 
    \left( \frac{\log n}{n} \right)^{1/(2\nu^{j})} 
    \right]} + \PP{{\left( E_j^{(2)} \right)}^c} \nonumber \\
    &\le K_j \cdot 2c_j^{(2)} 
    \left( \frac{\log n}{n} \right)^{1/(2\nu^{j})} 
    + \PP{{\left( E_j^{(2)} \right)}^c}, \nonumber 
\end{align}
where in the first step we split on the event $E_j^{(2)}$,
in the second step
we used the definition of $E_j^{(2)}$
to deduce the inclusion
\begin{align*}
    &\left\{ S_j^{(n+1)} \in ( \min\{\hat \lambda_j^{(2)},  \lambda_j^{*,(2)}\}, \max\{\hat \lambda_j^{(2)},  \lambda_j^{*,(2)}\}] \right\} \cap E_j^{(2)} \\
    &\subseteq \left\{ S_j^{(n+1)} \in \left( \lambda_j^{*,(2)} 
    - c_j^{(2)} 
    \left( \frac{\log n}{n} \right)^{1/(2\nu^{j})}
    , \lambda_j^{*,(2)} 
    + c_j^{(2)} 
    \left( \frac{\log n}{n} \right)^{1/(2\nu^{j})} 
    \right] \right\} \cap E_j^{(2)}
\end{align*}
and then dropped the intersection with $E_j^{(2)}$,
and in the third step we used 
\Cref{cond:cdf-lip}.
If $\PP{{\left( E_j^{(2)} \right)}^c} = O(1 / \sqrt n)$,
then the right-hand side is 
$O(\left( \frac{\log n}{n} \right)^{1/(2\nu^{j})})$,
as claimed.
Similarly,
if $j\in \mathcal{I}_{\mathrm{discrete}}$,
then we can write
\begin{align}\label{eq:split-on-event-tilde}
    &\PP{S_j^{(n+1)} \in ( \min\{\hat \lambda_j^{(2)},  \lambda_j^{*,(2)}\}, \max\{\hat \lambda_j^{(2)},  \lambda_j^{*,(2)}\}] } \\
    &\le \PP{ \{ S_j^{(n+1)} \in ( \min\{\hat \lambda_j^{(2)},  \lambda_j^{*,(2)}\}, \max\{\hat \lambda_j^{(2)},  \lambda_j^{*,(2)}\}] \} \cap E_j^{(2)}} + \PP{{\left( E_j^{(2)} \right)}^c} \nonumber \\
    &= \PP{{\left( E_j^{(2)} \right)}^c} \nonumber,
\end{align}
which is $O(1/\sqrt n)$
if $\PP{{\left( E_j^{(2)} \right)}^c} = O(1 / \sqrt n)$,
as claimed.

Thus, it suffices to bound
$|\hat \lambda_j^{(2)} - \lambda_j^{*,(2)}|$
whp
for each $j\in [m]$, a special case of the main induction when $k=1$.
In the following sections, we prove the preliminary and main induction arguments.


\subsubsection{Preliminary induction}

We begin with the preliminary induction.

\textit{Base case:}
Suppose that $j=1$,
and fix $k\in [m-j+1] = [m]$.
Recall that
$\lambda_1^{*,(2k)} =  U_1^*(\beta_1^{(k-1)} )$
and
$\lambda_1^{*,(2k+2)} =  U_1^*(\beta_1^{(k)} )$,
so that
$|\lambda_1^{*,(2k)} - \lambda_1^{*,(2k+2)}| = |U_1^*(\beta_1^{(k-1)} ) - U_1^*(\beta_1^{(k)} )|$.

Setting
$\ep_1 =  \beta_1 - \beta_1^{(k-1)}$
and
$\ep_1' = \beta_1^{(k)} - \beta_1^{(k-1)}$,
since
$|\ep_1| = (k-1)\delta_1 = (k-1) \frac{V^{\mathrm{max}}_1-V^{\mathrm{min}}_1}{n+1}$
and
$|\ep_1'| = \delta_1 = \frac{V^{\mathrm{max}}_1-V^{\mathrm{min}}_1}{n+1}$
tend to zero
as $n\to \infty$,
we may apply \Cref{lem:crossing-pert}
for sufficiently large $n$
to deduce
\begin{align*}
    |U_1^*(\beta_1^{(k-1)} ) - U_1^*(\beta_1^{(k)} )|
    \le \mu_1 \delta_1^{1/\nu}
    \le \mu_1 
    (V^{\mathrm{max}}_1-V^{\mathrm{min}}_1)^{1/\nu}
    \frac{1}{n^{1/\nu}}.
\end{align*}
Thus, we have
$|\lambda_1^{*,(2k)} - \lambda_1^{*,(2k+2)}|
\le \frac{b_1^{(2k)}}{n^{1/\nu}}$,
where we defined
$b_1^{(2k)}
:= \mu_1 (V^{\mathrm{max}}_1-V^{\mathrm{min}}_1)^{1/\nu}$.
If $1\in \mathcal{I}_{\mathrm{discrete}}$,
then by definition,
$\mu_1 =0$,
and so
$b_1^{(2k)} = 0$.
This proves the base case.

\textit{Inductive step:}
Assume that for some $j\in \{2,\ldots,m\}$,
for each $\ell\in [j-1]$,
and for each $k\in [m-\ell+1]$,
the result holds.
We show the result holds 
for $j$
and $k\in [m-j+1]$.
Fix $k\in [m-j+1]$.
Recall that
$\lambda_j^{*,(2k)} =  U_j^*(\lambda_{1:(j-1)}^{*,(2k)} ; \beta_j^{(k-1)} )$
and
$\lambda_j^{*,(2k+2)} =  U_j^*(\lambda_{1:(j-1)}^{*,(2k+2)} ; \beta_j^{(k)} )$.
Thus, we may write
\begin{align}\label{eq:prelim-decomp}
    \lambda_j^{*,(2k)} - \lambda_j^{*,(2k+2)} &= U_j^*(\lambda_{1:(j-1)}^{*,(2k)} ; \beta_j^{(k-1)} ) - U_j^*(\lambda_{1:(j-1)}^{*,(2k+2)} ; \beta_j^{(k)} ) \\
    &= (U_j^*(\lambda_{1:(j-1)}^{*,(2k+2)}; \beta_j^{(k-1)} ) - U_j^*(\lambda_{1:(j-1)}^{*,(2k+2)} ; \beta_j^{(k)} )) 
        \nonumber \\
    &+ (U_j^*(\lambda_{1:(j-1)}^{*,(2k)} ; \beta_j^{(k-1)} ) - U_j^*(\lambda_{1:(j-1)}^{*,(2k+2)} ; \beta_j^{(k-1)} )) \nonumber \\
    &=: \text{Term (A)} + \text{Term (B)}. \nonumber
\end{align}
To bound Term (A),
note that
by the inductive hypothesis,
if we define
$\alpha_{1:(j-1)} 
= \lambda_{1:(j-1)}^{*,(2k+2)} - \lambda_{1:(j-1)}^{*,(2)}$,
then for $\ell\in [j-1]$,
we have
$|\alpha_{\ell}|
\le \sum_{k'=1}^{k} b_{\ell}^{(2k')} / n^{1/\nu^{\ell}}$,
where $b_{\ell}^{(2k')} = 0$ if $\ell\in \mathcal{I}_{\mathrm{discrete}}$.
Thus,
$\alpha_{\ell} = 0$ for $\ell\in [j-1]\cap \mathcal{I}_{\mathrm{discrete}}$
and
$|\alpha_{\ell}| 
\le \sum_{k'=1}^{k} b_{\ell}^{(2k')} / n^{1/\nu^{\ell}}$
if $\ell\in [j-1]\setminus \mathcal{I}_{\mathrm{discrete}}$.
Setting
$\ep_j = \beta_j^{(k-1)} - \beta_j$,
$\ep_j' = \beta_j^{(k)} - \beta_j^{(k-1)}$,
and $\alpha_{1:(j-1)}$ as above,
since
$|\ep_j| = (k-1) \delta_j = (k-1) \frac{V^{\mathrm{max}}_j-V^{\mathrm{min}}_j}{n+1}$,
$|\ep_j'| = \delta_j = \frac{V^{\mathrm{max}}_j-V^{\mathrm{min}}_j}{n+1}$,
and
$\|\alpha_{1:(j-1)}\|_1$
tend to zero
as $n\to \infty$,
we may apply \Cref{lem:crossing-pert}
for sufficiently large $n$
to deduce that
\begin{align*}
    |U_j^*(\lambda_{1:(j-1)}^{*,(2k+2)} ; \beta_j^{(k-1)} ) - U_j^*(\lambda_{1:(j-1)}^{*,(2k+2)} ; \beta_j^{(k)} )|
    \le \mu_j \delta_j^{1/\nu}
    \le 
    \mu_j 
    (V^{\mathrm{max}}_j-V^{\mathrm{min}}_j)^{1/\nu} 
    \frac{1}{n^{1/\nu}}.
\end{align*}
To bound Term (B),
note that by the inductive hypothesis,
we have the bounds
$|\lambda_{\ell}^{*,(2k)} - \lambda_{\ell}^{*,(2k+2)}|
\le \frac{b_{\ell}^{(2k)}}{n^{1/\nu^{\ell}}}$
for $\ell\in [j-1]$.
By
\Cref{lem:g-star-pert},
if we define
\begin{align*}
    \tilde b_j^{(2k)} := V^{\mathrm{max}}_j \sum_{\ell\in [j-1]\setminus \mathcal{I}_{\mathrm{discrete}}} K_{\ell} b_{\ell}^{(2k)}, 
\end{align*}
then we have the uniform control
\begin{align*}
    \| g_j^*(\cdot; \lambda_{1:(j-1)}^{*,(2k)}) - g_j^*(\cdot; \lambda_{1:(j-1)}^{*,(2k+2)}) \|_{L^{\infty}(\R)} \le \frac{\tilde b_j^{(2k)}}{n^{1/\nu^{(j-1)}}}.
\end{align*}
By
\Cref{eq:two-sided-crossing}
from
\Cref{lem:crossing-points},
this implies that
\begin{align*}
    U_j^*\left( 
    \lambda_{1:(j-1)}^{*,(2k+2)} ; \beta_j^{(k)} 
    + 
    \frac{\tilde b_j^{(2k)}}{n^{1/\nu^{(j-1)}}} 
    \right) 
    \le 
    U_j^*(\lambda_{1:(j-1)}^{*,(2k)} ; \beta_j^{(k)} ) 
    \le U_j^*\left( 
    \lambda_{1:(j-1)}^{*,(2k+2)} ; \beta_j^{(k)} 
    - 
    \frac{\tilde b_j^{(2k)}}{n^{1/\nu^{(j-1)}}} 
    \right).
\end{align*}
Thus,
\begin{align*}
    &U_j^*\left( 
    \lambda_{1:(j-1)}^{*,(2k+2)} ; \beta_j^{(k)} 
    + 
    \frac{\tilde b_j^{(2k)}}{n^{1/\nu^{(j-1)}}} 
    \right) 
    - U_j^*(\lambda_{1:(j-1)}^{*,(2k+2)} ; \beta_j^{(k)} ) \\
    &\le U_j^*(\lambda_{1:(j-1)}^{*,(2k)} ; \beta_j^{(k)} ) 
    - U_j^*(\lambda_{1:(j-1)}^{*,(2k+2)} ; \beta_j^{(k)} ) \\
    &\le U_j^*\left( 
    \lambda_{1:(j-1)}^{*,(2k+2)} ; \beta_j^{(k)} 
    - 
    \frac{\tilde b_j^{(2k)}}{n^{1/\nu^{(j-1)}}} 
    \right) 
    - U_j^*(\lambda_{1:(j-1)}^{*,(2k+2)} ; \beta_j^{(k)} ).
\end{align*}
Setting
$\ep_j = \beta_j^{(k)} - \beta_j$,
$\ep_j' = 
\pm
\frac{\tilde b_j^{(2k)}}{n^{1/\nu^{(j-1)}}}$,
and
$\alpha_{1:(j-1)}
= \lambda_{1:(j-1)}^{*,(2k+2)} - \lambda_{1:(j-1)}^{*,(2)}$
as in Term (A) above,
since
$|\ep_j| = k\delta_j = k \frac{V^{\mathrm{max}}_j-V^{\mathrm{min}}_j}{n+1}$,
$|\ep_j'|$,
and
$\|\alpha_{1:(j-1)}\|_1$
tend to zero
as $n\to \infty$,
we may apply \Cref{lem:crossing-pert}
for sufficiently large $n$
to deduce the bounds
\begin{align*}
    \left| U_j^*\left( \lambda_{1:(j-1)}^{*,(2k+2)} ; \beta_j^{(k)} 
    + 
    \frac{\tilde b_j^{(2k)}}{n^{1/\nu^{(j-1)}}} 
    \right) 
    - U_j^*(\lambda_{1:(j-1)}^{*,(2k+2)} ; \beta_j^{(k)} ) \right|
    \le 
    \mu_j \left( \frac{\tilde b_j^{(2k)}}{n^{1/\nu^{(j-1)}}} \right)^{1/\nu}
    =
    \mu_j 
    \frac{(\tilde b_j^{(2k)})^{1/\nu}}{n^{1/\nu^j}}
\end{align*}
and
\begin{align*}
    \left| U_j^*\left( 
    \lambda_{1:(j-1)}^{*,(2k+2)} ; \beta_j^{(k)} - 
    \frac{\tilde b_j^{(2k)}}{n^{1/\nu^{(j-1)}}}  
    \right) 
    - U_j^*(\lambda_{1:(j-1)}^{*,(2k+2)} ; \beta_j^{(k)} ) \right|
    \le \mu_j \left( \frac{\tilde b_j^{(2k)}}{n^{1/\nu^{(j-1)}}} \right)^{1/\nu}
    = \mu_j 
    \frac{(\tilde b_j^{(2k)})^{1/\nu}}{n^{1/\nu^j}},
\end{align*}
so that Term (B) is bounded by
$\mu_j 
\frac{(\tilde b_j^{(2k)})^{1/\nu}}{n^{1/\nu^j}}$.
Plugging these bounds into
\Cref{eq:prelim-decomp},
we deduce that
\begin{align*}
    |\lambda_j^{*,(2k)} - \lambda_j^{*,(2k+2)}|
    \le 
    \mu_j 
    (V^{\mathrm{max}}_j-V^{\mathrm{min}}_j)^{1/\nu} 
    \frac{1}{n^{1/\nu}} 
    + 
    \mu_j 
    \frac{(\tilde b_j^{(2k)})^{1/\nu}}{n^{1/\nu^j}}
    \le \frac{b_j^{(2k)}}{n^{1/\nu^j}},
\end{align*}
where we defined
$b_j^{(2k)}
:=
\mu_j 
(V^{\mathrm{max}}_j-V^{\mathrm{min}}_j)^{1/\nu} 
+ \mu_j 
(\tilde b_j^{(2k)})^{1/\nu}$.
If $j\in \mathcal{I}_{\mathrm{discrete}}$,
then by definition,
$\mu_j = 0$,
and so $b_j^{(2k)} = 0$.
The induction is complete.


Having completed the preliminary induction,
in the following sections,
we perform the main induction.

\subsubsection{Main induction: base case}\label{subsubsec:base-case}

We begin with the base case of 
$j=1$ and $k\in [m-j+1] = [m]$.
Recall that
$\hat \lambda_1^{(2k)} - \lambda_1^{*,(2k)} = U_1^+(\beta_1^{(k-1)}) - U_1^*(\beta_1^{(k-1)})$.
By
\Cref{lem:emp-proc-bd},
we have the uniform control
$\|g_1^+(\cdot) - g_1^*(\cdot)\|_{L^{\infty}(\R)} \le r_1 \sqrt{\frac{\log n}{n}}$
whp.
Thus, by
\Cref{eq:two-sided-crossing}
from
\Cref{lem:crossing-points},
we have
\begin{align*}
    U_1^*\left(\beta_1^{(k-1)} + r_1 \sqrt{\frac{\log n}{n}} \right) \le U_1^+(\beta_1^{(k-1)}) \le U_1^*\left(\beta_1^{(k-1)} - r_1 \sqrt{\frac{\log n}{n}} \right)
\end{align*}
whp.

Thus,
\begin{align}\label{eq:base-derandom-sandwich}
    &U_1^*\left(\beta_1^{(k-1)} + r_1 \sqrt{\frac{\log n}{n}} \right) - U_1^*(\beta_1^{(k-1)}) \\
    &\le U_1^+(\beta_1^{(k-1)}) - U_1^*(\beta_1^{(k-1)}) \nonumber \\
    &\le U_1^*\left(\beta_1^{(k-1)} - r_1 \sqrt{\frac{\log n}{n}} \right) - U_1^*(\beta_1^{(k-1)}) \nonumber 
\end{align}
whp.
Setting
$\ep_j = \beta_1^{(k-1)} - \beta_1$
and
$\ep_j' = \pm r_1 \sqrt{\frac{\log n}{n}}$,
since
$|\ep_j| = (k-1) \delta_1 = (k-1) \frac{V^{\mathrm{max}}_1-V^{\mathrm{min}}_1}{n+1}$
and
$|\ep_j'|$
tend to zero
as $n\to \infty$,
we may apply
\Cref{lem:crossing-pert}
for sufficiently large $n$
to deduce the bounds
\begin{align*}
    \left| U_1^*\left(\beta_1^{(k-1)} + r_1 \sqrt{\frac{\log n}{n}} \right) - U_1^*(\beta_1^{(k-1)}) \right|
    \le 
    \mu_1 
    \left( 
    r_1 \sqrt{\frac{\log n}{n}} \right)^{1/\nu}
    =
    \mu_1 r_1^{1/\nu}
    \left( \frac{\log n}{n} \right)^{1/(2\nu)}
\end{align*}
and
\begin{align*}
    \left| U_1^*\left(\beta_1^{(k-1)} - r_1 \sqrt{\frac{\log n}{n}} \right) - U_1^*(\beta_1^{(k-1)})  \right|
    \le 
    \mu_1 
    \left( 
    r_1 \sqrt{\frac{\log n}{n}} \right)^{1/\nu}
    =
    \mu_1 r_1^{1/\nu}
    \left( \frac{\log n}{n} \right)^{1/(2\nu)}
    ,
\end{align*}
which when plugged into
\Cref{eq:base-derandom-sandwich}
imply that
\begin{align*}
    |\hat \lambda_1^{(2k)} - \lambda_1^{*,(2k)} | := |U_1^+(\beta_1^{(k-1)}) - U_1^*(\beta_1^{(k-1)})|
    \le 
    \mu_1 r_1^{1/\nu}
    \left( \frac{\log n}{n} \right)^{1/(2\nu)}
    =: c_1^{(2k)} 
    \left( \frac{\log n}{n} \right)^{1/(2\nu)}
\end{align*}
whp,
where we defined
$c_1^{(2k)} := \mu_1 r_1^{1/\nu}$.
If $1\in \mathcal{I}_{\mathrm{discrete}}$,
then by definition,
$\mu_1=0$,
and so $c_1^{(2k)} = 0$.
This proves the base case.

\subsubsection{Main induction: inductive step, decomposition}\label{subsubsec:inductive-step}

Now suppose that
for some $j\in \{2, \ldots, m\}$,
for each $\ell\in [j-1]$,
and for each $k\in [m-\ell+1]$,
we have shown the result.
We show the result for $j$.

Recall that $\hat \lambda_j^{(2k)} = U_j^+(\hat \lambda_{1:(j-1)}^{(2k+2)}; \beta_j^{(k-1)})$,
so that
\begin{align}\label{eq:lambda-1-decomp}
    \hat \lambda_j^{(2k)} - \lambda_j^{*,(2k)} &= U_j^+(\hat \lambda_{1:(j-1)}^{(2k+2)}; \beta_j^{(k-1)}) - U_j^*(\lambda_{1:(j-1)}^{*,(2k)}; \beta_j^{(k-1)}) \\
    &=  (U_j^*(\hat \lambda_{1:(j-1)}^{(2k+2)}; \beta_j^{(k-1)}) - U_j^*(\lambda_{1:(j-1)}^{*,(2k+2)}; \beta_j^{(k-1)})) \nonumber \\
    &+ (U_j^+(\hat \lambda_{1:(j-1)}^{(2k+2)}; \beta_j^{(k-1)}) - U_j^*(\hat \lambda_{1:(j-1)}^{(2k+2)}; \beta_j^{(k-1)})) \nonumber \\
    &+  (U_j^*(\lambda_{1:(j-1)}^{*,(2k+2)}; \beta_j^{(k-1)}) - U_j^*(\lambda_{1:(j-1)}^{*,(2k)}; \beta_j^{(k-1)})) \nonumber \\
    &=: \text{Term (I)} + \text{Term (II)} + \text{Term (III)}. \nonumber
\end{align}
In the following sections,
we analyze each term in turn,
and then recombine at the end.

\subsubsection{Main induction: inductive step, Term (I)}

To bound Term (I),
we need uniform control of
$g_j^*(\cdot; \hat \lambda_{1:(j-1)}^{(2k+2)}) - g_j^*(\cdot; \lambda_{1:(j-1)}^{*,(2k+2)})$.
To do this,
we use monotonicity and induction.
By the inductive hypothesis,
$|\hat \lambda_{\ell}^{(2k+2)} - \lambda_{\ell}^{*,(2k+2)}| 
\le 
c_{\ell}^{(2k+2)} 
\left( \frac{\log n}{n} \right)^{1/(2\nu^{\ell})}$
whp
for each $\ell\in [j-1]$,
so that
\begin{align*}
    \hat \lambda_{\ell}^{(2k+2)} 
    \in 
    \left[ 
    \lambda_{\ell}^{*,(2k+2)} 
    - 
    c_{\ell}^{(2k+2)} 
    \left( \frac{\log n}{n} \right)^{1/(2\nu^{\ell})}
    ,
    \lambda_{\ell}^{*,(2k+2)} 
    + 
    c_{\ell}^{(2k+2)} 
    \left( \frac{\log n}{n} \right)^{1/(2\nu^{\ell})}
    \right]
\end{align*}
whp
for each $\ell\in [j-1]$.
Since by
\Cref{lem:pop-monot},
for $\lambda_{1:(j-1)}\in \R^{j-1}$
the map
$\lambda_{1:(j-1)}\mapsto g_j^*(\lambda_j; \lambda_{1:(j-1)})$
is $\nea$ in each argument,
and since
$\left( \frac{\log n}{n} \right)^{1/(2\nu^{\ell})}
\le 
\left( \frac{\log n}{n} \right)^{1/(2\nu^{j-1})}$
for $\ell\in [j-1]$,
we have the bounds
\begin{align*}
    g_j^*(\cdot; \hat \lambda_{1:(j-1)}^{(2k+2)}) - g_j^*(\cdot; \lambda_{1:(j-1)}^{*,(2k+2)}) 
    &\in 
    \bigg[   
    g_j^*\left( 
    \cdot; \lambda_{1:(j-1)}^{*,(2k+2)} 
    - c_{1:(j-1)}^{(2k+2)} 
    \left( \frac{\log n}{n} \right)^{1/(2\nu^{j-1})} 
    \right) 
    - g_j^*(\cdot; \lambda_{1:(j-1)}^{*,(2k+2)}), \\ 
    & g_j^*\left( 
    \cdot; \lambda_{1:(j-1)}^{*,(2k+2)} 
    + c_{1:(j-1)}^{(2k+2)} 
    \left( \frac{\log n}{n} \right)^{1/(2\nu^{j-1})}  
    \right) 
    - g_j^*(\cdot; \lambda_{1:(j-1)}^{*,(2k+2)}) \bigg]
\end{align*}
whp.
By 
\Cref{lem:g-star-pert},
if we define
\begin{align}\label{eq:C-j-def}
    C_j^{(2k+2)} := V^{\mathrm{max}}_j \sum_{\ell\in [j-1]\setminus \mathcal{I}_{\mathrm{discrete}}} K_{\ell} 2c_{\ell}^{(2k+2)},
\end{align}
then we have
the uniform control
\begin{align}\label{eq:star-minus-star-plus-ctrl}
    \left\| g_j^*\left( 
    \cdot; \lambda_{1:(j-1)}^{*,(2k+2)} 
    + c_{1:(j-1)}^{(2k+2)} 
    \left( \frac{\log n}{n} \right)^{1/(2\nu^{j-1})} 
    \right) 
    - g_j^*\left( \cdot; \lambda_{1:(j-1)}^{*,(2k+2)} \right) \right\|_{L^{\infty}(\R)}
    \le 
    \frac{1}{2} C_j^{(2k+2)} 
    \left( \frac{\log n}{n} \right)^{1/(2\nu^{j-1})}
\end{align}
and
\begin{align}\label{eq:star-minus-star-minus-ctrl}
    \left\| g_j^*\left( 
    \cdot; \lambda_{1:(j-1)}^{*,(2k+2)} 
    - c_{1:(j-1)}^{(2k+2)} 
    \left( \frac{\log n}{n} \right)^{1/(2\nu^{j-1})}
    \right) 
    - g_j^*\left( \cdot; \lambda_{1:(j-1)}^{*,(2k+2)} \right) \right\|_{L^{\infty}(\R)} 
    \le 
    \frac{1}{2} C_j^{(2k+2)} 
    \left( \frac{\log n}{n} \right)^{1/(2\nu^{j-1})}.
\end{align}
Then the previous display implies that
\begin{align*}
    g_j^*(\cdot; \hat \lambda_{1:(j-1)}^{(2k+2)}) - g_j^*(\cdot; \lambda_{1:(j-1)}^{*,(2k+2)}) &\in 
    \left[ 
    -\frac{1}{2} C_j^{(2k+2)} 
    \left( \frac{\log n}{n} \right)^{1/(2\nu^{j-1})} 
    , 
    \frac{1}{2} C_j^{(2k+2)} 
    \left( \frac{\log n}{n} \right)^{1/(2\nu^{j-1})} 
    \right]
\end{align*}
whp, as needed.

---

Finally, we show how this translates
to a bound on Term (I).
By 
\Cref{eq:two-sided-crossing}
from
\Cref{lem:crossing-points},
the previous bound implies
\begin{align*}
    &U_j^*\left( 
    \lambda_{1:(j-1)}^{*,(2k+2)}; 
    \beta_j^{(k-1)} 
    +\frac{1}{2} C_j^{(2k+2)} 
    \left( \frac{\log n}{n} \right)^{1/(2\nu^{j-1})} 
    \right) \\
    &\le U_j^*(\hat \lambda_{1:(j-1)}^{(2k+2)}; \beta_j^{(k-1)}) \\
    &\le U_j^*\left( 
    \lambda_{1:(j-1)}^{*,(2k+2)}; 
    \beta_j^{(k-1)} 
    - \frac{1}{2} C_j^{(2k+2)} 
    \left( \frac{\log n}{n} \right)^{1/(2\nu^{j-1})} 
    \right)
\end{align*}
whp.

Thus,
\begin{align*}
    &U_j^*\left( \lambda_{1:(j-1)}^{*,(2k+2)}; 
    \beta_j^{(k-1)} 
    + \frac{1}{2} C_j^{(2k+2)} 
    \left( \frac{\log n}{n} \right)^{1/(2\nu^{j-1})} 
    \right) 
    -  U_j^*(\lambda_{1:(j-1)}^{*,(2k+2)}; \beta_j^{(k-1)}) \\
    &\le 
    U_j^*(\hat \lambda_{1:(j-1)}^{(2k+2)}; \beta_j^{(k-1)}) - U_j^*(\lambda_{1:(j-1)}^{*,(2k+2)}; \beta_j^{(k-1)}) \\
    &\le 
    U_j^*\left( \lambda_{1:(j-1)}^{*,(2k+2)}; 
    \beta_j^{(k-1)} 
    - \frac{1}{2} C_j^{(2k+2)} 
    \left( \frac{\log n}{n} \right)^{1/(2\nu^{j-1})} 
    \right) 
    -  U_j^*(\lambda_{1:(j-1)}^{*,(2k+2)}; \beta_j^{(k-1)})
\end{align*}
whp.

Setting
$\ep_j = \beta_j^{(k-1)} - \beta_j$,
$\ep_j' 
= 
\pm 
\frac{1}{2} C_j^{(2k+2)}
\left( \frac{\log n}{n} \right)^{1/(2\nu^{j-1})}$,
and
$\alpha_{1:(j-1)}
= \lambda_{1:(j-1)}^{*,(2k+2)}
- \lambda_{1:(j-1)}^{*,(2)}$,
since by the preliminary induction
we have
$|\alpha_{\ell}| \le \sum_{k'=1}^{k} b_{\ell}^{(2k')}/n^{1/\nu^{\ell}}$
with $b_{\ell}^{(2k')} = 0$ if $\ell\in [j-1]\cap \mathcal{I}_{\mathrm{discrete}}$,
and since
$|\ep_j| = (k-1) \delta_j = (k-1) \frac{V^{\mathrm{max}}_j-V^{\mathrm{min}}_j}{n+1}$,
$|\ep_j'|$,
and $\|\alpha_{1:(j-1)}\|_1$
tend to zero
as $n\to \infty$,
we may apply 
\Cref{lem:crossing-pert}
for sufficiently large $n$
to deduce the bounds
\begin{align*}
    &\left| U_j^*\left( \lambda_{1:(j-1)}^{*,(2k+2)}; \beta_j^{(k-1)} 
    + 
    \frac{1}{2} C_j^{(2k+2)} 
    \left( \frac{\log n}{n} \right)^{1/(2\nu^{j-1})} 
    \right) 
    -  U_j^*(\lambda_{1:(j-1)}^{*,(2k+2)}; \beta_j^{(k-1)}) \right| \\
    &\le \mu_j 
    \left( 
    \frac{1}{2} C_j^{(2k+2)} 
    \left( \frac{\log n}{n} \right)^{1/(2\nu^{j-1})} \right)^{1/\nu} 
    = 
    \mu_j 
    \left( \frac{1}{2} C_j^{(2k+2)} \right)^{1/\nu}
    \left( \frac{\log n}{n} \right)^{1/(2\nu^{j})}
\end{align*}
whp
and
\begin{align*}
    &\left| U_j^*\left( \lambda_{1:(j-1)}^{*,(2k+2)}; \beta_j^{(k-1)} 
    - 
    \frac{1}{2} C_j^{(2k+2)} 
    \left( \frac{\log n}{n} \right)^{1/(2\nu^{j-1})} 
    \right) 
    -  U_j^*(\lambda_{1:(j-1)}^{*,(2k+2)}; \beta_j^{(k-1)}) \right| \\
    &\le \mu_j 
    \left( 
    \frac{1}{2} C_j^{(2k+2)} 
    \left( \frac{\log n}{n} \right)^{1/(2\nu^{j-1})} \right)^{1/\nu} 
    = 
    \mu_j 
    \left( \frac{1}{2} C_j^{(2k+2)} \right)^{1/\nu}
    \left( \frac{\log n}{n} \right)^{1/(2\nu^{j})}
\end{align*}
whp,
so that
\begin{align}\label{eq:term-i-final-bd}
    |\text{Term (I)}|
    \le
    \mu_j 
    \left( \frac{1}{2} C_j^{(2k+2)} \right)^{1/\nu}
    \left( \frac{\log n}{n} \right)^{1/(2\nu^{j})}
    =: 
    C_{I, j}^{(2k)} 
    \left( \frac{\log n}{n} \right)^{1/(2\nu^{j})}
\end{align}
whp,
where we defined
$C_{I,j}^{(2k)}
:= 
\mu_j 
\left( \frac{1}{2} C_j^{(2k+2)} \right)^{1/\nu}$.
We see that Term (I) is indeed $O(\left( \frac{\log n}{n} \right)^{1/(2\nu^{j})})$ whp.

\subsubsection{Main induction: inductive step, Term (II)}

To bound Term (II),
we need uniform control of
$g_j^+(\cdot; \hat \lambda_{1:(j-1)}^{(2k+2)}) - g_j^*(\cdot; \hat \lambda_{1:(j-1)}^{(2k+2)})$.
First, we claim that it suffices 
to control the difference
$g_j^+(\cdot; \lambda_{1:(j-1)}) - g_j^*(\cdot; \lambda_{1:(j-1)})$
for nonrandom $\lambda_{1:(j-1)}\in \R^{j-1}$.
To see this,
we use monotonicity
and induction.
By the inductive hypothesis,
$|\hat \lambda_{\ell}^{(2k+2)} - \lambda_{\ell}^{*,(2k+2)}| 
\le 
c_{\ell}^{(2k+2)} 
\left( \frac{\log n}{n} \right)^{1/(2\nu^{\ell})}$
whp
for each $\ell\in [j-1]$,
so that
\begin{align*}
    \hat \lambda_{\ell}^{(2k+2)} \in \left[ \lambda_{\ell}^{*,(2k+2)} 
    - 
    c_{\ell}^{(2k+2)} 
    \left( \frac{\log n}{n} \right)^{1/(2\nu^{\ell})}
    , 
    \lambda_{\ell}^{*,(2k+2)} 
    + 
    c_{\ell}^{(2k+2)} 
    \left( \frac{\log n}{n} \right)^{1/(2\nu^{\ell})} 
    \right]
\end{align*}
whp
for each $\ell\in [j-1]$.
Since by
\Cref{lem:emp-monot},
for $\lambda_j\in \R$
the map
$\lambda_{1:(j-1)} \mapsto g_j^+(\lambda_j; \lambda_{1:(j-1)})$
is $\nea$ in each argument
a.s.,
since by
\Cref{lem:pop-monot},
for $\lambda_j\in \R$
the map
$\lambda_{1:(j-1)}\mapsto g_j^*(\lambda_j; \lambda_{1:(j-1)})$
is $\nea$ in each argument,
and since
$\left( \frac{\log n}{n} \right)^{1/(2\nu^{\ell})}
\le 
\left( \frac{\log n}{n} \right)^{1/(2\nu^{j-1})}$
for $\ell\in [j-1]$,
we have the bounds
\begin{align*}
     &g_j^+(\cdot; \hat \lambda_{1:(j-1)}^{(2k+2)}) - g_j^*(\cdot; \hat \lambda_{1:(j-1)}^{(2k+2)}) \\
     &\in \bigg[ g_j^+\left( 
     \cdot; 
     \lambda_{1:(j-1)}^{*,(2k+2)} 
     - c_{1:(j-1)}^{(2k+2)} 
     \left( \frac{\log n}{n} \right)^{1/(2\nu^{j-1})} 
     \right) 
     - g_j^*\left( 
     \cdot; 
     \lambda_{1:(j-1)}^{*,(2k+2)} 
     + c_{1:(j-1)}^{(2k+2)} 
     \left( \frac{\log n}{n} \right)^{1/(2\nu^{j-1})} 
     \right), \\  
     & g_j^+\left( 
     \cdot; 
     \lambda_{1:(j-1)}^{*,(2k+2)} 
     + c_{1:(j-1)}^{(2k+2)} 
     \left( \frac{\log n}{n} \right)^{1/(2\nu^{j-1})} 
     \right) 
     - g_j^*\left( 
     \cdot; 
     \lambda_{1:(j-1)}^{*,(2k+2)} 
     - c_{1:(j-1)}^{(2k+2)} 
     \left( \frac{\log n}{n} \right)^{1/(2\nu^{j-1})} 
     \right) \bigg]
\end{align*}
whp.
We massage the two $g_j^*(\cdot;\cdot)$ terms,
just as we did in the analysis of Term (I).
Combining
\Cref{eq:star-minus-star-plus-ctrl}
and
\Cref{eq:star-minus-star-minus-ctrl}
with the previous display,
we obtain
\begin{align}\label{eq:term-ii-setup}
    & g_j^+(\cdot; \hat \lambda_{1:(j-1)}^{(2k+2)}) - g_j^*(\cdot; \hat \lambda_{1:(j-1)}^{(2k+2)}) \\
    &\in 
    \bigg[ 
    g_j^+\left( \cdot; \lambda_{1:(j-1)}^{*,(2k+2)} 
    - c_{1:(j-1)}^{(2k+2)} 
    \left( \frac{\log n}{n} \right)^{1/(2\nu^{j-1})} 
    \right) 
    - g_j^*\left( 
    \cdot; 
    \lambda_{1:(j-1)}^{*,(2k+2)} 
    - c_{1:(j-1)}^{(2k+2)} 
    \left( \frac{\log n}{n} \right)^{1/(2\nu^{j-1})} 
    \right) \nonumber \\
    &- C_j^{(2k+2)} 
    \left( \frac{\log n}{n} \right)^{1/(2\nu^{j-1})}, 
    \nonumber \\
     & g_j^+\left( 
     \cdot; 
     \lambda_{1:(j-1)}^{*,(2k+2)} 
     + c_{1:(j-1)}^{(2k+2)} 
     \left( \frac{\log n}{n} \right)^{1/(2\nu^{j-1})} 
     \right) 
     - g_j^*\left( 
     \cdot; 
     \lambda_{1:(j-1)}^{*,(2k+2)} 
     + c_{1:(j-1)}^{(2k+2)} 
     \left( \frac{\log n}{n} \right)^{1/(2\nu^{j-1})} 
     \right) \nonumber \\
     &+ C_j^{(2k+2)} 
     \left( \frac{\log n}{n} \right)^{1/(2\nu^{j-1})} 
     \bigg] \nonumber 
\end{align}
whp.
If we can achieve 
$O(\left( \frac{\log n}{n} \right)^{1/(2\nu^{j-1})})$ 
control
whp
in the nonrandom case when
$\lambda_{1:(j-1)} 
= \lambda_{1:(j-1)}^{*,(2k+2)} 
+ c_{1:(j-1)}^{(2k+2)} 
\left( \frac{\log n}{n} \right)^{1/(2\nu^{j-1})}$
and when 
$\lambda_{1:(j-1)} 
= \lambda_{1:(j-1)}^{*,(2k+2)} 
- c_{1:(j-1)}^{(2k+2)} 
\left( \frac{\log n}{n} \right)^{1/(2\nu^{j-1})}$,
then this interval lies within 
$O(\left( \frac{\log n}{n} \right)^{1/(2\nu^{j-1})})$
of the origin whp, as claimed.

By \Cref{lem:emp-proc-bd}
with 
$\lambda_{1:(j-1)} 
= \lambda_{1:(j-1)}^{*,(2k+2)} 
+ c_{1:(j-1)}^{(2k+2)} 
\left( \frac{\log n}{n} \right)^{1/(2\nu^{j-1})}$,
we have
\begin{align*}
    &\left\| g_j^+\left( 
    \cdot; 
    \lambda_{1:(j-1)}^{*,(2k+2)} 
    + c_{1:(j-1)}^{(2k+2)} 
    \left( \frac{\log n}{n} \right)^{1/(2\nu^{j-1})} 
    \right) 
    - g_j^*\left( 
    \cdot; 
    \lambda_{1:(j-1)}^{*,(2k+2)} 
    + c_{1:(j-1)}^{(2k+2)} 
    \left( \frac{\log n}{n} \right)^{1/(2\nu^{j-1})} 
    \right) \right\|_{L^{\infty}(\R)} 
    \\
    &\le r_j \sqrt{\frac{\log n}{n}}
\end{align*}
whp.
Similarly,
by \Cref{lem:emp-proc-bd}
with 
$\lambda_{1:(j-1)} 
= \lambda_{1:(j-1)}^{*,(2k+2)} 
- c_{1:(j-1)}^{(2k+2)} 
\left( \frac{\log n}{n} \right)^{1/(2\nu^{j-1})}$,
we also have
\begin{align*}
    &\left\| g_j^+\left( 
    \cdot; 
    \lambda_{1:(j-1)}^{*,(2k+2)} 
    - c_{1:(j-1)}^{(2k+2)} 
    \left( \frac{\log n}{n} \right)^{1/(2\nu^{j-1})} 
    \right) 
    - g_j^*\left( 
    \cdot; 
    \lambda_{1:(j-1)}^{*,(2k+2)} 
    - c_{1:(j-1)}^{(2k+2)} 
    \left( \frac{\log n}{n} \right)^{1/(2\nu^{j-1})} 
    \right) \right\|_{L^{\infty}(\R)} \\
    &\le r_j \sqrt{\frac{\log n}{n}}
\end{align*}
whp.
Plugging these bounds into \Cref{eq:term-ii-setup},
and using the fact that
$\sqrt{\frac{\log n}{n}} 
\le 
\left( \frac{\log n}{n} \right)^{1/(2\nu^{j-1})}$,
we deduce that
\begin{align*}
    \|  g_j^+(\cdot; \hat \lambda_{1:(j-1)}^{(2k+2)}) - g_j^*(\cdot; \hat \lambda_{1:(j-1)}^{(2k+2)}) \|_{L^{\infty}(\R)} 
    \le 
    (r_j + C_j^{(2k+2)}) 
    \left( \frac{\log n}{n} \right)^{1/(2\nu^{j-1})}
\end{align*}
whp.

---

Finally, we show how this translates
into a bound on Term (II).
By 
\Cref{eq:two-sided-crossing}
from
\Cref{lem:crossing-points},
the previous bound implies
\begin{align*}
    &U_j^*\left( \hat \lambda_{1:(j-1)}^{(2k+2)};
    \beta_j^{(k-1)} 
    + (r_j + C_j^{(2k+2)}) 
    \left( \frac{\log n}{n} \right)^{1/(2\nu^{j-1})} 
    \right) \\
    &\le U_j^+(\hat \lambda_{1:(j-1)}^{(2k+2)}; \beta_j^{(k-1)}) \\
    &\le U_j^*\left( \hat \lambda_{1:(j-1)}^{(2k+2)}; 
    \beta_j^{(k-1)} 
    - (r_j + C_j^{(2k+2)}) 
    \left( \frac{\log n}{n} \right)^{1/(2\nu^{j-1})} 
    \right)
\end{align*}
whp,
so that
\begin{align*}
    & U_j^*\left( \hat \lambda_{1:(j-1)}^{(2k+2)}; 
    \beta_j^{(k-1)} 
    + (r_j + C_j^{(2k+2)}) 
    \left( \frac{\log n}{n} \right)^{1/(2\nu^{j-1})} 
    \right) 
    -  U_j^*(\hat \lambda_{1:(j-1)}^{(2k+2)}; \beta_j^{(k-1)}) \\
    &\le U_j^+(\hat \lambda_{1:(j-1)}^{(2k+2)}; \beta_j^{(k-1)}) -  U_j^*(\hat \lambda_{1:(j-1)}^{(2k+2)}; \beta_j^{(k-1)}) \\
    &\le U_j^*\left( \hat \lambda_{1:(j-1)}^{(2k+2)}; 
    \beta_j^{(k-1)} 
    - (r_j + C_j^{(2k+2)}) 
    \left( \frac{\log n}{n} \right)^{1/(2\nu^{j-1})} 
    \right) 
    - U_j^*(\hat \lambda_{1:(j-1)}^{(2k+2)}; \beta_j^{(k-1)})
\end{align*}
whp.
Since
\begin{align*}
    \hat \lambda_{\ell}^{(2k+2)} \in 
    \left[ \lambda_{\ell}^{*,(2k+2)} 
    - c_{\ell}^{(2k+2)} 
    \left( \frac{\log n}{n} \right)^{1/(2\nu^{\ell})}
    , 
    \lambda_{\ell}^{*,(2k+2)} 
    + c_{\ell}^{(2k+2)} 
    \left( \frac{\log n}{n} \right)^{1/(2\nu^{\ell})} 
    \right]
\end{align*}
whp
for each $\ell\in [j-1]$,
since by \Cref{lem:pop-monot}
$U_j^*$ is $\nea$ in its first $j-1$ arguments,
and since
$\left( \frac{\log n}{n} \right)^{1/(2\nu^{\ell})}
\le 
\left( \frac{\log n}{n} \right)^{1/(2\nu^{j-1})}$
for $\ell\in [j-1]$,
this implies that
\begin{align}\label{eq:term-ii-random-sandwich}
    & U_j^*\left( \lambda_{1:(j-1)}^{*,(2k+2)} 
    - c_{1:(j-1)}^{(2k+2)} 
    \left( \frac{\log n}{n} \right)^{1/(2\nu^{j-1})} 
    ; 
    \beta_j^{(k-1)} 
    + (r_j + C_j^{(2k+2)}) 
    \left( \frac{\log n}{n} \right)^{1/(2\nu^{j-1})}  
    \right) \\
    &-  U_j^*\left( \lambda_{1:(j-1)}^{*,(2k+2)} 
    + c_{1:(j-1)}^{(2k+2)} 
    \left( \frac{\log n}{n} \right)^{1/(2\nu^{j-1})} 
    ; \beta_j^{(k-1)} \right) \nonumber \\
    &\le U_j^+(\hat \lambda_{1:(j-1)}^{(2k+2)}; \beta_j^{(k-1)}) -  U_j^*(\hat \lambda_{1:(j-1)}^{(2k+2)}; \beta_j^{(k-1)}) \nonumber \\
    &\le U_j^*\left( \lambda_{1:(j-1)}^{*,(2k+2)} 
    + c_{1:(j-1)}^{(2k+2)} 
    \left( \frac{\log n}{n} \right)^{1/(2\nu^{j-1})} 
    ; 
    \beta_j^{(k-1)} 
    - (r_j + C_j^{(2k+2)}) 
    \left( \frac{\log n}{n} \right)^{1/(2\nu^{j-1})}  
    \right) \nonumber \\
    &- U_j^*\left( \lambda_{1:(j-1)}^{*,(2k+2)} 
    - c_{1:(j-1)}^{(2k+2)} 
    \left( \frac{\log n}{n} \right)^{1/(2\nu^{j-1})} 
    ; \beta_j^{(k-1)} \right) \nonumber 
\end{align}
whp.

First,
we massage the lower bound
in \Cref{eq:term-ii-random-sandwich}.
Setting
$\ep_j = \beta_j^{(k-1)} - \beta_j$,
$\ep_j' = 
(r_j + C_j^{(2k+2)}) 
\left( \frac{\log n}{n} \right)^{1/(2\nu^{j-1})} $,
and
$\alpha_{1:(j-1)}
= \lambda_{1:(j-1)}^{*,(2k+2)}
- c_{1:(j-1)}^{(2k+2)} 
\left( \frac{\log n}{n} \right)^{1/(2\nu^{j-1})} 
- \lambda_{1:(j-1)}^{*,(2)}$,
since by
the preliminary induction
and
the inductive hypothesis 
we have
\begin{align*}
|\alpha_{\ell}|
\le
\sum_{k'=1}^{k} b_{\ell}^{(2k')}/n^{1/\nu^{\ell}}
+ c_{\ell}^{(2k+2)} 
\left( \frac{\log n}{n} \right)^{1/(2\nu^{j-1})}
\end{align*}
with $b_{\ell}^{(2k')} = 0$ if $\ell\in [j-1]\cap \mathcal{I}_{\mathrm{discrete}}$
and $c_{\ell}^{(2k+2)} = 0$ if $\ell\in [j-1]\cap \mathcal{I}_{\mathrm{discrete}}$,
and since
$|\ep_j| = (k-1) \delta_j = (k-1) \frac{V^{\mathrm{max}}_j-V^{\mathrm{min}}_j}{n+1}$,
$|\ep_j'|$,
and $\|\alpha_{1:(j-1)}\|_1$
tend to zero
as $n\to \infty$,
we may apply \Cref{lem:crossing-pert}
for sufficiently large $n$
to deduce the bounds
\begin{align*}
   &\bigg| U_j^*\left( \lambda_{1:(j-1)}^{*,(2k+2)} 
   - c_{1:(j-1)}^{(2k+2)} 
   \left( \frac{\log n}{n} \right)^{1/(2\nu^{j-1})} 
   ; 
   \beta_j^{(k-1)} 
   + (r_j + C_j^{(2k+2)}) 
   \left( \frac{\log n}{n} \right)^{1/(2\nu^{j-1})}  
   \right) \\
   &- U_j^*\left( \lambda_{1:(j-1)}^{*,(2k+2)} 
   - c_{1:(j-1)}^{(2k+2)} 
   \left( \frac{\log n}{n} \right)^{1/(2\nu^{j-1})} 
   ; 
   \beta_j^{(k-1)} \right) \bigg| \\
   &\le 
   \mu_j  
   \left( 
   (r_j + C_j^{(2k+2)}) 
   \left( \frac{\log n}{n} \right)^{1/(2\nu^{j-1})} 
   \right)^{1/\nu}
   = 
   \mu_j
   (r_j + C_j^{(2k+2)})^{1/\nu}
   \left( \frac{\log n}{n} \right)^{1/(2\nu^{j})}
\end{align*}
whp.

Similarly,
we massage the upper bound
in \Cref{eq:term-ii-random-sandwich}.
Setting
$\ep_j = \beta_j^{(k-1)} - \beta_j$,
$\ep_j' = 
(r_j + C_j^{(2k+2)}) 
\left( \frac{\log n}{n} \right)^{1/(2\nu^{j-1})}$,
and
$\alpha_{1:(j-1)}
= \lambda_{1:(j-1)}^{*,(2k+2)}
+ c_{1:(j-1)}^{(2k+2)} 
\left( \frac{\log n}{n} \right)^{1/(2\nu^{j-1})}
- \lambda_{1:(j-1)}^{*,(2)}$,
since by
the preliminary induction
and
the inductive hypothesis 
we have
\begin{align*}
|\alpha_{\ell}|
\le
\sum_{k'=1}^{k} b_{\ell}^{(2k')}/n^{1/\nu^{\ell}}
+ c_{\ell}^{(2k+2)} 
\left( \frac{\log n}{n} \right)^{1/(2\nu^{j-1})}
\end{align*}
with $b_{\ell}^{(2k')} = 0$ if $\ell\in [j-1]\cap \mathcal{I}_{\mathrm{discrete}}$
and $c_{\ell}^{(2k+2)} = 0$ if $\ell\in [j-1]\cap \mathcal{I}_{\mathrm{discrete}}$,
and since
$|\ep_j| = (k-1) \delta_j = (k-1) \frac{V^{\mathrm{max}}_j-V^{\mathrm{min}}_j}{n+1}$,
$|\ep_j'|$,
and $\|\alpha_{1:(j-1)}\|_1$
tend to zero
as $n\to \infty$,
we may apply \Cref{lem:crossing-pert}
for sufficiently large $n$
to deduce the bounds
\begin{align*}
   &\bigg| U_j^*\left( \lambda_{1:(j-1)}^{*,(2k+2)} 
   + c_{1:(j-1)}^{(2k+2)} 
   \left( \frac{\log n}{n} \right)^{1/(2\nu^{j-1})}
   ; 
   \beta_j^{(k-1)} 
   + (r_j + C_j^{(2k+2)}) 
   \left( \frac{\log n}{n} \right)^{1/(2\nu^{j-1})} 
   \right) \\
   &- U_j^*\left( \lambda_{1:(j-1)}^{*,(2k+2)} 
   + c_{1:(j-1)}^{(2k+2)} 
   \left( \frac{\log n}{n} \right)^{1/(2\nu^{j-1})}
   ; \beta_j^{(k-1)} \right) \bigg| \\
   &\le 
   \mu_j  
   \left( 
   (r_j + C_j^{(2k+2)}) 
   \left( \frac{\log n}{n} \right)^{1/(2\nu^{j-1})} 
   \right)^{1/\nu}
   = 
   \mu_j
   (r_j + C_j^{(2k+2)})^{1/\nu}
   \left( \frac{\log n}{n} \right)^{1/(2\nu^{j})}
\end{align*}
whp.

Plugging these into
\Cref{eq:term-ii-random-sandwich},
we deduce that
\begin{align}\label{eq:term-ii-derandom-sandwich}
    & U_j^*\left( \lambda_{1:(j-1)}^{*,(2k+2)} 
    - c_{1:(j-1)}^{(2k+2)} 
    \left( \frac{\log n}{n} \right)^{1/(2\nu^{j-1})}
    ; \beta_j^{(k-1)} \right) \\
    &-  U_j^*\left( \lambda_{1:(j-1)}^{*,(2k+2)} 
    + c_{1:(j-1)}^{(2k+2)} 
    \left( \frac{\log n}{n} \right)^{1/(2\nu^{j-1})}
    ; \beta_j^{(k-1)} \right) 
    - 
    \mu_j
   (r_j + C_j^{(2k+2)})^{1/\nu}
   \left( \frac{\log n}{n} \right)^{1/(2\nu^{j})} 
    \nonumber \\
    &\le U_j^+(\hat \lambda_{1:(j-1)}^{(2k+2)}; \beta_j^{(k-1)}) 
    - U_j^*(\hat \lambda_{1:(j-1)}^{(2k+2)}; \beta_j^{(k-1)}) \nonumber \\
    &\le U_j^*\left( \lambda_{1:(j-1)}^{*,(2k+2)} 
    + c_{1:(j-1)}^{(2k+2)} 
    \left( \frac{\log n}{n} \right)^{1/(2\nu^{j-1})}
    ; \beta_j^{(k-1)} \right) \nonumber \\
    &- U_j^*\left( \lambda_{1:(j-1)}^{*,(2k+2)} 
    - c_{1:(j-1)}^{(2k+2)} 
    \left( \frac{\log n}{n} \right)^{1/(2\nu^{j-1})}
    ; \beta_j^{(k-1)} \right) 
    + 
    \mu_j
   (r_j + C_j^{(2k+2)})^{1/\nu}
   \left( \frac{\log n}{n} \right)^{1/(2\nu^{j})} 
    \nonumber 
\end{align}
whp.

Finally,
by
\Cref{lem:g-star-pert},
we have the uniform control
\begin{align*}
    &\left| g_j^*\left( \cdot; \lambda_{1:(j-1)}^{*,(2k+2)} 
    + c_{1:(j-1)}^{(2k+2)} 
    \left( \frac{\log n}{n} \right)^{1/(2\nu^{j-1})} 
    \right) 
    - g_j^*\left( \cdot; \lambda_{1:(j-1)}^{*,(2k+2)} 
    - c_{1:(j-1)}^{(2k+2)} 
    \left( \frac{\log n}{n} \right)^{1/(2\nu^{j-1})} 
    \right) \right| \\
    &\le C_j^{(2k+2)} 
    \left( \frac{\log n}{n} \right)^{1/(2\nu^{j-1})}
\end{align*}
for $n\ge 1$,
so that by
\Cref{eq:two-sided-crossing}
from
\Cref{lem:crossing-points}
we have
\begin{align*}
    &U_j^*\left( \lambda_{1:(j-1)}^{*,(2k+2)} 
    + c_{1:(j-1)}^{(2k+2)} 
    \left( \frac{\log n}{n} \right)^{1/(2\nu^{j-1})}
    ; \beta_j^{(k-1)} 
    + C_j^{(2k+2)} 
    \left( \frac{\log n}{n} \right)^{1/(2\nu^{j-1})} 
    \right) \\
    &\le U_j^*\left( \lambda_{1:(j-1)}^{*,(2k+2)} 
    - c_{1:(j-1)}^{(2k+2)} 
    \left( \frac{\log n}{n} \right)^{1/(2\nu^{j-1})}
    ; \beta_j^{(k-1)} \right) \\
    &\le U_j^*\left( \lambda_{1:(j-1)}^{*,(2k+2)} 
    + c_{1:(j-1)}^{(2k+2)} 
    \left( \frac{\log n}{n} \right)^{1/(2\nu^{j-1})}
    ; \beta_j^{(k-1)} 
    - C_j^{(2k+2)} 
    \left( \frac{\log n}{n} \right)^{1/(2\nu^{j-1})} 
    \right),
\end{align*}
which in turn yields
\begin{align*}
    &U_j^*\left( \lambda_{1:(j-1)}^{*,(2k+2)} 
    + c_{1:(j-1)}^{(2k+2)} 
    \left( \frac{\log n}{n} \right)^{1/(2\nu^{j-1})}
    ; 
    \beta_j^{(k-1)} 
    + C_j^{(2k+2)} 
    \left( \frac{\log n}{n} \right)^{1/(2\nu^{j-1})} 
    \right) \\
    &- U_j^*\left( \lambda_{1:(j-1)}^{*,(2k+2)} 
    + c_{1:(j-1)}^{(2k+2)} 
    \left( \frac{\log n}{n} \right)^{1/(2\nu^{j-1})}
    ; \beta_j^{(k-1)} \right) \\
    &\le U_j^*\left( \lambda_{1:(j-1)}^{*,(2k+2)} 
    - c_{1:(j-1)}^{(2k+2)} 
    \left( \frac{\log n}{n} \right)^{1/(2\nu^{j-1})}
    ; \beta_j^{(k-1)} \right) \\
    &- U_j^*\left( \lambda_{1:(j-1)}^{*,(2k+2)} 
    + c_{1:(j-1)}^{(2k+2)} 
    \left( \frac{\log n}{n} \right)^{1/(2\nu^{j-1})}
    ; \beta_j^{(k-1)} \right) \\
    &\le U_j^*\left( \lambda_{1:(j-1)}^{*,(2k+2)} 
    + c_{1:(j-1)}^{(2k+2)} 
    \left( \frac{\log n}{n} \right)^{1/(2\nu^{j-1})}
    ; \beta_j^{(k-1)} 
    - C_j^{(2k+2)} 
    \left( \frac{\log n}{n} \right)^{1/(2\nu^{j-1})}
    \right) \\
    &- U_j^*\left( \lambda_{1:(j-1)}^{*,(2k+2)} 
    + c_{1:(j-1)}^{(2k+2)} 
    \left( \frac{\log n}{n} \right)^{1/(2\nu^{j-1})}
    ; \beta_j^{(k-1)} \right).
\end{align*}
Setting
$\ep_j = \beta_j^{(k-1)} - \beta_j$,
$\ep_j' = 
\pm C_j^{(2k+2)} 
\left( \frac{\log n}{n} \right)^{1/(2\nu^{j-1})}$,
and
$\alpha_{1:(j-1)}
= \lambda_{1:(j-1)}^{*,(2k+2)}
+ c_{1:(j-1)}^{(2k+2)} 
\left( \frac{\log n}{n} \right)^{1/(2\nu^{j-1})}
- \lambda_{1:(j-1)}^{*,(2)}$,
since by
the preliminary induction
and
the inductive hypothesis 
we have
\begin{align*}
|\alpha_{\ell}|
\le
\sum_{k'=1}^{k} b_{\ell}^{(2k')}/n^{1/\nu^{\ell}}
+ c_{\ell}^{(2k+2)} 
\left( \frac{\log n}{n} \right)^{1/(2\nu^{j-1})}
\end{align*}
with $b_{\ell}^{(2k')} = 0$ if $\ell\in [j-1]\cap \mathcal{I}_{\mathrm{discrete}}$
and $c_{\ell}^{(2k+2)} = 0$ if $\ell\in [j-1]\cap \mathcal{I}_{\mathrm{discrete}}$,
and since
$|\ep_j| = (k-1) \delta_j = (k-1) \frac{V^{\mathrm{max}}_j-V^{\mathrm{min}}_j}{n+1}$,
$|\ep_j'|$,
and $\|\alpha_{1:(j-1)}\|_1$
tend to zero
as $n\to \infty$,
we may apply \Cref{lem:crossing-pert}
for sufficiently large $n$
to deduce the bounds
\begin{align*}
    &\bigg| U_j^*\left( \lambda_{1:(j-1)}^{*,(2k+2)} 
    + c_{1:(j-1)}^{(2k+2)} 
    \left( \frac{\log n}{n} \right)^{1/(2\nu^{j-1})}
    ; \beta_j^{(k-1)} 
    + C_j^{(2k+2)} 
    \left( \frac{\log n}{n} \right)^{1/(2\nu^{j-1})} 
    \right) \\
    &- U_j^*\left( \lambda_{1:(j-1)}^{*,(2k+2)} 
    + c_{1:(j-1)}^{(2k+2)} 
    \left( \frac{\log n}{n} \right)^{1/(2\nu^{j-1})}
    ; \beta_j^{(k-1)} \right) \bigg|
    \le 
    \mu_j 
    \left(
    C_j^{(2k+2)}
    \left( \frac{\log n}{n} \right)^{1/(2\nu^{j-1})} 
    \right)^{1/\nu} \\
    &= 
    \mu_j 
    \left( C_j^{(2k+2)} \right)^{1/\nu}
    \left( \frac{\log n}{n} \right)^{1/(2\nu^{j})}
\end{align*}
and
\begin{align*}
    &\bigg| U_j^*\left( \lambda_{1:(j-1)}^{*,(2k+2)} 
    + c_{1:(j-1)}^{(2k+2)} 
    \left( \frac{\log n}{n} \right)^{1/(2\nu^{j-1})}
    ; \beta_j^{(k-1)} 
    - C_j^{(2k+2)} 
    \left( \frac{\log n}{n} \right)^{1/(2\nu^{j-1})} 
    \right) \\
    &- U_j^*\left( \lambda_{1:(j-1)}^{*,(2k+2)} 
    + c_{1:(j-1)}^{(2k+2)} 
    \left( \frac{\log n}{n} \right)^{1/(2\nu^{j-1})}
    ; \beta_j^{(k-1)} \right) \bigg|
    \le 
    \mu_j 
    \left(
    C_j^{(2k+2)}
    \left( \frac{\log n}{n} \right)^{1/(2\nu^{j-1})} 
    \right)^{1/\nu} \\
    &= 
    \mu_j 
    \left( C_j^{(2k+2)} \right)^{1/\nu}
    \left( \frac{\log n}{n} \right)^{1/(2\nu^{j})}
\end{align*}
it follows that 
\begin{align*}
    &\bigg| U_j^*\left( \lambda_{1:(j-1)}^{*,(2k+2)} 
    + c_{1:(j-1)}^{(2k+2)} 
    \left( \frac{\log n}{n} \right)^{1/(2\nu^{j-1})} 
    ; \beta_j^{(k-1)} \right) \\
    &- U_j^*\left( \lambda_{1:(j-1)}^{*,(2k+2)} 
    - c_{1:(j-1)}^{(2k+2)} 
    \left( \frac{\log n}{n} \right)^{1/(2\nu^{j-1})} 
    ; \beta_j^{(k-1)} \right) \bigg|
    \le 
    \mu_j 
    \left( C_j^{(2k+2)} \right)^{1/\nu}
    \left( \frac{\log n}{n} \right)^{1/(2\nu^{j})}
    .
\end{align*}
Putting everything together,
\Cref{eq:term-ii-derandom-sandwich}
implies
\begin{align}\label{eq:term-ii-final-bd}
    \text{Term (II)} &= | U_j^+(\hat \lambda_{1:(j-1)}^{(2k+2)}; \beta_j^{(k-1)}) -  U_j^*(\hat \lambda_{1:(j-1)}^{(2k+2)}; \beta_j^{(k-1)}) | \\
    &\le 
    \mu_j 
    \left( C_j^{(2k+2)} \right)^{1/\nu}
    \left( \frac{\log n}{n} \right)^{1/(2\nu^{j})} 
    + 
    \mu_j
   (r_j + C_j^{(2k+2)})^{1/\nu}
   \left( \frac{\log n}{n} \right)^{1/(2\nu^{j})} 
    \nonumber \\
    &=: C_{II, j}^{(2k)} 
    \left( \frac{\log n}{n} \right)^{1/(2\nu^{j})}
    \nonumber 
\end{align}
whp,
where we defined
$C_{II, j}^{(2k)} 
:= \mu_j \left( C_j^{(2k+2)} \right)^{1/\nu} 
+ \mu_j (r_j + C_j^{(2k+2)})^{1/\nu}$,
as desired.
We see that Term (II) is indeed
$O(\left( \frac{\log n}{n} \right)^{1/(2\nu^{j})})$ whp.

\subsubsection{Main induction: inductive step, Term (III)}

We may write Term (III) as
\begin{align}\label{eq:term-iii-decomp}
    \text{Term (III)}  &= (U_j^*(\lambda_{1:(j-1)}^{*, (2k+2)}; \beta_j^{(k-1)}) - U_j^*(\lambda_{1:(j-1)}^{*,(2k+2)}; \beta_j^{(k)})) \\
    &+ (U_j^*(\lambda_{1:(j-1)}^{*,(2k+2)}; \beta_j^{(k)}) - U_j^*(\lambda_{1:(j-1)}^{*,(2k)}; \beta_j^{(k-1)}))) \nonumber \\
    &= (U_j^*(\lambda_{1:(j-1)}^{*, (2k+2)}; \beta_j^{(k-1)}) - U_j^*(\lambda_{1:(j-1)}^{*,(2k+2)}; \beta_j^{(k)})) \nonumber \\
    &+ (\lambda_j^{*,(2k+2)} - \lambda_j^{*,(2k)}) \nonumber \\
    &=: \text{Term (X)} + \text{Term (Y)}, \nonumber 
\end{align}
where in the second step
we recalled the definitions of
$\lambda_j^{*,(2k+2)}$ and $\lambda_j^{*,(2k)}$.

To bound Term (X),
setting
$\ep_j = \beta_j^{(k-1)} - \beta_j$,
$\ep_j' = \beta_j^{(k)} - \beta_j^{(k-1)}$,
and
$\alpha_{1:(j-1)}
= \lambda_{1:(j-1)}^{*,(2k+2)}
- \lambda_{1:(j-1)}^{*,(2)}$,
since by
the preliminary induction
we have
$|\alpha_{\ell}|
\le
\sum_{k'=1}^{k} b_{\ell}^{(2k')}/n^{1/\nu^{\ell}}$
with $b_{\ell}^{(2k')} = 0$ if $\ell\in [j-1]\cap \mathcal{I}_{\mathrm{discrete}}$,
and since
$|\ep_j| = (k-1) \delta_j = (k-1) \frac{V^{\mathrm{max}}_j-V^{\mathrm{min}}_j}{n+1}$,
$|\ep_j'| = \delta_j = \frac{V^{\mathrm{max}}_j-V^{\mathrm{min}}_j}{n+1}$,
and $\|\alpha_{1:(j-1)}\|_1$
tend to zero
as $n\to \infty$,
we may apply \Cref{lem:crossing-pert}
for sufficiently large $n$
to deduce the bounds
\begin{align*}
    |U_j^*(\lambda_{1:(j-1)}^{*, (2k+2)}; \beta_j^{(k-1)}) - U_j^*(\lambda_{1:(j-1)}^{*,(2k+2)}; \beta_j^{(k)})|
    \le \mu_j \delta_j^{1/\nu} 
    \le \mu_j (V^{\mathrm{max}}_j-V^{\mathrm{min}}_j)^{1/\nu} \frac{1}{n^{1/\nu}}.
\end{align*}
To bound Term (Y),
note that by the preliminary induction,
$|\text{Term (Y)}| 
\le 
\frac{b_j^{(2k)}}{n^{1/\nu^j}}$.

Plugging these bounds into
\Cref{eq:term-iii-decomp},
we deduce that
\begin{align}\label{eq:term-iii-final-bd}
    |\text{Term (III)}| 
    \le 
    \mu_j (V^{\mathrm{max}}_j-V^{\mathrm{min}}_j)^{1/\nu} \frac{1}{n^{1/\nu}}
    + 
    \frac{b_j^{(2k)}}{n^{1/\nu^j}} 
    \le 
    C_{III, j}^{(2k)} 
    \left( \frac{\log n}{n} \right)^{1/(2\nu^{j})},
\end{align}
where in the second step
we defined
$C_{III, j}^{(2k)} 
:= 
\mu_j (V^{\mathrm{max}}_j-V^{\mathrm{min}}_j)^{1/\nu} + b_j^{(2k)}$
and used the inequality
$1/n \le \sqrt{\log n/n}$
for $n\ge 2$.
We see that Term (III) is indeed $O(\left( \frac{\log n}{n} \right)^{1/(2\nu^{j})})$ whp.

\subsubsection{Main induction: inductive step, conclusion}

Summing
\Cref{eq:term-i-final-bd},
\Cref{eq:term-ii-final-bd},
and
\Cref{eq:term-iii-final-bd},
and plugging into
\Cref{eq:lambda-1-decomp},
we deduce that
\begin{align*}
     |\hat \lambda_j^{(2k)} - \lambda_j^{*, (2k)}| 
     \le 
     (C_{I, j}^{(2k)} + C_{II, j}^{(2k)} + C_{III, j}^{(2k)}) 
     \left( \frac{\log n}{n} \right)^{1/(2\nu^{j})} 
     =: 
     c_j^{(2k)} 
     \left( \frac{\log n}{n} \right)^{1/(2\nu^{j})}
\end{align*}
whp,
where we defined
$c_j^{(2k)} := C_{I, j}^{(2k)} + C_{II, j}^{(2k)} + C_{III, j}^{(2k)}$.
Explicitly, we have
\begin{align*}
c_j^{(2k)} 
= 
\mu_j 
\left( \frac{1}{2} C_j^{(2k+2)} \right)^{1/\nu}
+ 
\left( \mu_j \left( C_j^{(2k+2)} \right)^{1/\nu} 
+ \mu_j (r_j + C_j^{(2k+2)})^{1/\nu} \right)
+
(\mu_j (V^{\mathrm{max}}_j-V^{\mathrm{min}}_j)^{1/\nu} + b_j^{(2k)}),
\end{align*}
so that if $j\in \mathcal{I}_{\mathrm{discrete}}$,
then by definition, we have $\mu_j = 0$,
and by the preliminary induction,
we also have $b_j^{(2k)} = 0$,
hence
$c_j^{(2k)} = 0$.
This completes the inductive step.

\subsection{Proof of \Cref{thm:disc-conc}}\label{subsec:pf-disc-conc}

We give the proof with $0.99$ replaced by an arbitrary $\zeta\in (0,1)$ in the following sections.
Fix $\zeta\in (0,1)$.
In this remainder of this section,
when we write ``whp", we mean
with probability $1-O(\exp(-n^{\zeta}))$.

\subsubsection{Reduction to consistency of thresholds}

As in the proof of \Cref{thm:cts-conc},
we begin by reducing the result to establishing consistency of the estimated thresholds. Specifically, we perform two inductive arguments:

\begin{enumerate}
    \item \textbf{Preliminary induction:} for each $j\in [m]$,
for each $k\in [m-j+1]$,
we show that
$\lambda_j^{*,(2k)} - \lambda_j^{*,(2k+2)} = 0$.
This proceeds via induction on $j$.

    \item \textbf{Main induction:} for each $j\in [m]$,
for each $k\in [m-j+1]$,
we show that
$\hat \lambda_j^{(2k)} - \lambda_j^{*,(2k)} = 0$
whp.
We also use induction on $j$.
\end{enumerate}

Here, we show how the main induction implies the result.
Define the event
$\widetilde E_j^{(2)} = \left\{ \hat \lambda_j^{(2)} = \lambda_j^{*,(2)} \right\}$
for $j\in [m]$.
By the main induction with $k=1$,
we have
\begin{align*}
\PP{{\left( \widetilde E_j^{(2)} \right)}^c} \le O\left( \exp(-n^{\zeta}) \right).
\end{align*}
Thus, by
\Cref{eq:split-on-event-tilde},
\Cref{eq:swapping-trick},
and
\Cref{cond:obj-loss-bds},
for any population minimizer
$\lambda_{1:m}^{*}$
we have
\begin{align*}
\E V_{m+1}^{(n+1)} I(S_{1:m}^{(n+1)} \pce \hat \lambda_{1:m}^{(2)})
&\le \E V_{m+1}^{(n+1)} I(S_{1:m}^{(n+1)} \pce \lambda_{1:m}^{*}) + V^{\mathrm{max}}_{m+1} \sum_{j=1}^m \PP{{\left( \widetilde E_j^{(2)} \right)}^c} \\
&\le \E V_{m+1}^{(n+1)} I(S_{1:m}^{(n+1)} \pce \lambda_{1:m}^{*}) + O\left( \exp(-n^{\zeta}) \right),
\end{align*}
which proves \Cref{thm:disc-conc}.

The following sections
are dedicated to the preliminary and main
induction arguments.

\subsubsection{Preliminary induction}

We begin with the preliminary induction.

\textit{Base case:}
Suppose that $j=1$,
and fix
$k\in [m-j+1] = [m]$.
Recall that
$\lambda_1^{*,(2k)} = U_1^*(\beta_1^{(k-1)})$
and
$\lambda_1^{*,(2k+2)} = U_1^*(\beta_1^{(k)})$.
By
\Cref{lem:g-star-pc},
$g_1^*(\cdot)$
is a step function.
Also,
since by
\Cref{cond:no-bad-beta-strong} 
we have
$\dist(\beta_1, \vv_1) > 0$,
since
$|\beta_1^{(k-1)} - \beta_1| = (k-1)\delta_1 := (k-1)\frac{V^{\mathrm{max}}_1-V^{\mathrm{min}}_1}{n+1}$,
and
since
$|\beta_1^{(k)} - \beta_1| = k\delta_1 := k\frac{V^{\mathrm{max}}_1-V^{\mathrm{min}}_1}{n+1}$,
for sufficiently large $n$
we have
$|\beta_1^{(k-1)} - \beta_1|
< \dist(\beta_1, \vv_1)$
and
$|\beta_1^{(k)} - \beta_1|
< \dist(\beta_1, \vv_1)$,
and thus
$\beta_1^{(k-1)}$ and $\beta_1^{(k)}$
lie in the same connected component
of $\R\setminus \vv_1$.
Hence,
by \Cref{lem:pc-inverse},
we have
$U_1^*(\beta_1^{(k-1)}) = U_1^*(\beta_1^{(k)})$.
Thus,
$\lambda_1^{*,(2k)} = \lambda_1^{*,(2k+2)}$,
as claimed.

\textit{Inductive step:}
Assume that for some
$j\in \{2,\ldots,m\}$,
for each $\ell\in [j-1]$,
and for each $k\in [m-\ell+1]$,
the result holds.
We show the result holds for
$j$ and $k\in [m-j+1]$.
Recall that
$\lambda_j^{*,(2k)} =  U_j^*(\lambda_{1:(j-1)}^{*,(2k)} ; \beta_j^{(k-1)} )$
and
$\lambda_j^{*,(2k+2)} =  U_j^*(\lambda_{1:(j-1)}^{*,(2k+2)} ; \beta_j^{(k)} )$,
so that
\begin{align}
    \lambda_j^{*,(2k)} - \lambda_j^{*,(2k+2)} &= U_j^*(\lambda_{1:(j-1)}^{*,(2k)} ; \beta_j^{(k-1)} ) - U_j^*(\lambda_{1:(j-1)}^{*,(2k+2)} ; \beta_j^{(k)} ) \\
    &= U_j^*(\lambda_{1:(j-1)}^{*,(2)}; \beta_j^{(k-1)} ) - U_j^*(\lambda_{1:(j-1)}^{*,(2)} ; \beta_j^{(k)} )
        \nonumber,
\end{align}
where in the second step
we used the inductive hypothesis that
$\lambda_{1:(j-1)}^{*,(2)} = \lambda_{1:(j-1)}^{*,(2k+2)}$.
By
\Cref{lem:g-star-pc},
$g_j^*(\cdot; \lambda_{1:(j-1)}^{*,(2)})$
is a step function.
Also,
since by
\Cref{cond:no-bad-beta-strong} 
we have
$\dist(\beta_j, \vv_j^{(0)}) > 0$,
since
$|\beta_j^{(k-1)} - \beta_j| = (k-1)\delta_j := (k-1)\frac{V^{\mathrm{max}}_j-V^{\mathrm{min}}_j}{n+1}$,
and
since
$|\beta_j^{(k)} - \beta_j| = k\delta_j := k\frac{V^{\mathrm{max}}_j-V^{\mathrm{min}}_j}{n+1}$,
for sufficiently large $n$
we have
$|\beta_j^{(k-1)} - \beta_j|
< \dist(\beta_j, \vv_j^{(0)})$
and
$|\beta_j^{(k)} - \beta_j|
< \dist(\beta_j, \vv_j^{(0)})$,
and thus
$\beta_j^{(k-1)}$ and $\beta_j^{(k)}$
lie in the same connected component
of $\R\setminus \vv_j^{(0)}$.
Hence, by
\Cref{lem:pc-inverse},
we have
$U_j^*(\lambda_{1:(j-1)}^{*,(2)}; \beta_j^{(k-1)} ) = U_j^*(\lambda_{1:(j-1)}^{*,(2)} ; \beta_j^{(k)} )$.
Thus,
$\lambda_j^{*,(2k)} = \lambda_j^{*,(2k+2)}$,
completing the induction.

Having completed the preliminary induction, in the following sections, we perform the main induction.

\subsubsection{Main induction: base case}

We begin with the base case of 
$j=1$ and $k\in [m-j+1] = [m]$.
Recall that
$\hat \lambda_1^{(2k)} - \lambda_1^{*,(2k)} = U_1^+(\beta_1^{(k-1)}) - U_1^*(\beta_1^{(k-1)})$.

To show that this difference is zero whp,
we need uniform control of
$g_1^+(\cdot) - g_1^*(\cdot)$.
Since
$\|g_1^+(\cdot) - g_1(\cdot)\|_{L^{\infty}(\R)} \le \delta_1$,
it suffices to control
$g_1(\cdot) - g_1^*(\cdot)$.

Define the event
\begin{align*}
    \widetilde E_1 = \bigg\{ \bigg| \frac{1}{n+1} \sum_{i=1}^{n+1} (&V_1^{(i)} I(S_1^{(i)} = s)
    - \E V_1^{(i)} I(S_1^{(i)} = s)) \bigg| \le \sqrt{2} V^{\mathrm{max}}_1 \sqrt{ \frac{n^{\zeta}}{n+1} } \text{ for all } s \in \mathcal{S}_1 \bigg\}.
\end{align*}
By \Cref{cond:loss-bds}
and
\Cref{cond:iid-observations},
the random variables
\begin{align*}
    \{ V_1^{(i)} I(S_1^{(i)} = s)
    -
    \E V_1^{(i)} I(S_1^{(i)} = s) : i\in [n+1] \}
\end{align*}
are i.i.d.
mean zero
and have modulus
bounded by $V^{\mathrm{max}}_1$ a.s.
Thus,
by Hoeffding's inequality
\citep{hoeffding1963probability}
and a union bound,
we have
\begin{align*}
    \PP{\widetilde E_1} 
    \ge 1 - 
    |\mathcal{S}_1| 
    \exp\left( -\frac{2(n+1)^2 (\sqrt{2} V^{\mathrm{max}}_1 \sqrt{(n^{\zeta}/(n+1)})^2}{(n+1) (2V^{\mathrm{max}}_1)^2} \right)
    = 1 - |\mathcal{S}_1| \exp(-n^{\zeta}),
\end{align*}
so that $\widetilde E_1$ occurs whp.
On the event $\widetilde E_1$,
we have
\begin{align*}
    \|g_1(\cdot) - g_1^*(\cdot)\|_{L^{\infty}(\R)} \le |\mathcal{S}_j| \sqrt{2} V^{\mathrm{max}}_1 \sqrt{ \frac{n^{\zeta}}{n+1} },
\end{align*}
which by the triangle inequality
implies that
\begin{align*}
    &\|g_1^+(\cdot) - g_1^*(\cdot)\|_{L^{\infty}(\R)} 
    \le 
    \|g_1(\cdot) - g_1^*(\cdot)\|_{L^{\infty}(\R)}
    + \|g_1^+(\cdot) - g_j(\cdot)\|_{L^{\infty}(\R)} 
    \le |\mathcal{S}_1| \sqrt{2} V^{\mathrm{max}}_1 \sqrt{ \frac{n^{\zeta}}{n+1} }
    + \delta_1 \\
    &\le ( |\mathcal{S}_1| \sqrt{2} V^{\mathrm{max}}_1 + (V^{\mathrm{max}}_1 - V^{\mathrm{min}}_1) ) \sqrt{ \frac{n^{\zeta}}{n+1} } 
    =: \widetilde C_1 \sqrt{ \frac{n^{\zeta}}{n+1} }
\end{align*}
for $n\ge 1$,
where in the third step
we used the definition
$\delta_1 := (V^{\mathrm{max}}_1-V^{\mathrm{min}}_1)/(n+1)$
and the inequality
$1/(n+1) \le \sqrt{n^{\zeta} / n}$
for $n\ge 1$,
and in the fourth step we defined
$\widetilde C_1 = |\mathcal{S}_1| \sqrt{2} V^{\mathrm{max}}_1 + (V^{\mathrm{max}}_1 - V^{\mathrm{min}}_1)$.

Define the positive integer
$N_1^{(k)}$
as
\begin{align*}
N_1^{(k)}
=
\min\left\{ n' \in \NN : 
\widetilde C_1
\sqrt{ \frac{(n')^{\zeta}}{n'+1} }
\le
\frac{1}{2} \dist(\beta_1^{(k-1)}, \vv_1)
\text{ and }
n' \ge 1
\right\}.
\end{align*}
We claim that $N_1^{(k)}$ is finite.
To see this,
note that
since
\Cref{cond:no-bad-beta-strong} 
implies that
$\dist(\beta_1, \vv_1) > 0$,
and since
$|\beta_1^{(k-1)} - \beta_1| = (k-1) \delta_1
:= (k-1) \frac{V^{\mathrm{max}}_1-V^{\mathrm{min}}_1}{n+1}$,
for sufficiently large $n$
we have
$|\beta_1^{(k-1)} - \beta_1| < \dist(\beta_1, \vv_1)$,
and hence by the triangle inequality
for the Hausdorff distance,
we have
$\dist(\beta_1^{(k-1)}, \vv_1) > 0$.
Since
$\zeta\in (0,1)$
implies that
$\sqrt{(n')^{\zeta}/(n'+1)} \to 0$
as $n'\to\infty$,
it follows that $N_1^{(k)}$ is finite.

Now,
for $n\ge N_1^{(k)}$,
we claim that
$U_1^+(\beta_1^{(k-1)}) = U_1^*(\beta_1^{(k-1)})$.
Indeed,
by
\Cref{eq:two-sided-crossing}
from
\Cref{lem:crossing-points},
we have
\begin{align*}
U_1^*\left( \beta_1^{(k-1)} 
+ \frac{1}{2} \dist(\beta_1^{(k-1)}, \vv_1) \right) 
\le U_1^+(\beta_1^{(k-1)}) 
\le U_1^*\left( \beta_1^{(k-1)}
- \frac{1}{2} \dist(\beta_1^{(k-1)}, \vv_1) \right).
\end{align*}
By \Cref{lem:g-star-pc},
$g_1^*(\cdot)$
is a step function,
so that by
\Cref{lem:pc-inverse},
we have
\begin{align*}
U_1^*\left( \beta_1^{(k-1)} 
+ \frac{1}{2} \dist(\beta_1^{(k-1)}, \vv_1) \right) 
= U_1^*\left( \beta_1^{(k-1)} \right) 
= U_1^*\left( \beta_1^{(k-1)} 
- \frac{1}{2} \dist(\beta_1^{(k-1)}, \vv_1) \right),
\end{align*}
which when combined with
the previous display
implies that
$U_1^+(\beta_1^{(k-1)})
= U_1^*(\beta_1^{(k-1)})$,
as desired.
It follows that
$\hat \lambda_1^{(2k)} = \lambda_1^{*,(2k)}$ whp, as desired.

\subsubsection{Main induction: inductive step}

Now suppose that
for some $j\in \{2, \ldots, m\}$,
for each $\ell\in [j-1]$,
and for each $k\in [m-\ell+1]$,
we have shown the result.
We show the result for $j$.
The argument is very similar to the base case.

Fix $k\in [m-j+1]$.
Recall that $\hat \lambda_j^{(2k)} = U_j^+(\hat \lambda_{1:(j-1)}^{(2k+2)}; \beta_j^{(k-1)})$,
so that
\begin{align}
    \hat \lambda_j^{(2k)} - \lambda_j^{*,(2k)} &= U_j^+(\hat \lambda_{1:(j-1)}^{(2k+2)}; \beta_j^{(k-1)}) - U_j^*(\lambda_{1:(j-1)}^{*,(2k)}; \beta_j^{(k-1)}) \\
    &= U_j^+(\lambda_{1:(j-1)}^{*, (2)}; \beta_j^{(k-1)}) - U_j^*(\lambda_{1:(j-1)}^{*, (2)}; \beta_j^{(k-1)}) \nonumber
\end{align}
whp,
where in the second step
we used the inductive hypothesis that
$\hat \lambda_{1:(j-1)}^{(2k+2)} = \lambda_{1:(j-1)}^{*,(2k+2)}$ whp
and the equality
$\lambda_{1:(j-1)}^{*,(2k+2)} = \lambda_{1:(j-1)}^{*,(2)}$
from the preliminary induction.

To show that this difference is zero whp,
we need uniform control of
$g_j^+(\cdot; \lambda_{1:(j-1)}^{*, (2)}) - g_j^*(\cdot; \lambda_{1:(j-1)}^{*, (2)})$.
Since
$\|g_j^+(\cdot;\lambda_{1:(j-1)}) - g_j(\cdot;\lambda_{1:(j-1)})\|_{L^{\infty}(\R)} \le \delta_j$
for any $\lambda_{1:(j-1)}\in \R^{j-1}$,
it suffices to control
$g_j(\cdot; \lambda_{1:(j-1)}^{*, (2)}) - g_j^*(\cdot; \lambda_{1:(j-1)}^{*, (2)})$.

Define the event
\begin{align*}
    \widetilde E_j^{(k)} = \bigg\{ \bigg| \frac{1}{n+1} \sum_{i=1}^{n+1} (&V_j^{(i)} I(S_{1:(j-1)}^{(i)} \le \lambda_{1:(j-1)}^{*, (2)}) I(S_j^{(i)} = s) \\
    - &\E V_j^{(i)} I(S_{1:(j-1)}^{(i)} \le \lambda_{1:(j-1)}^{*, (2)}) I(S_j^{(i)} = s)) \bigg| \le \sqrt{2} V^{\mathrm{max}}_j \sqrt{ \frac{n^{\zeta}}{n+1} } \text{ for all } s \in \mathcal{S}_j \bigg\}.
\end{align*}
By \Cref{cond:loss-bds}
and
\Cref{cond:iid-observations},
the random variables
\begin{align*}
    \{ V_j^{(i)} I(S_{1:(j-1)}^{(i)} \le \lambda_{1:(j-1)}^{*, (2)}) I(S_j^{(i)} = s)
    -
    \E V_j^{(i)} I(S_{1:(j-1)}^{(i)} \le \lambda_{1:(j-1)}^{*, (2)}) I(S_j^{(i)} = s) : i\in [n+1] \}
\end{align*}
are i.i.d.
mean zero
and have modulus
bounded by $V^{\mathrm{max}}_j$ a.s.
Thus,
by Hoeffding's inequality
\citep{hoeffding1963probability}
and a union bound,
we have
\begin{align*}
    \PP{\widetilde E_j^{(k)}} 
    \ge 1 - 
    |\mathcal{S}_j| 
    \exp\left( -\frac{2(n+1)^2 (\sqrt{2} V^{\mathrm{max}}_j \sqrt{(n^{\zeta}/(n+1)})^2}{(n+1) (2V^{\mathrm{max}}_j)^2} \right)
    = 1 - |\mathcal{S}_j|\exp(-n^{\zeta}),
\end{align*}
so that $\widetilde E_j^{(k)}$ occurs whp.
On the event $\widetilde E_j^{(k)}$,
we have
\begin{align*}
    \|g_j(\cdot; \lambda_{1:(j-1)}^{*, (2)}) - g_j^*(\cdot; \lambda_{1:(j-1)}^{*, (2)})\|_{L^{\infty}(\R)} \le |\mathcal{S}_j| \sqrt{2} V^{\mathrm{max}}_j \sqrt{ \frac{n^{\zeta}}{n+1} },
\end{align*}
which by the triangle inequality implies that
\begin{align*}
    &\|g_j^+(\cdot; \lambda_{1:(j-1)}^{*, (2)}) - g_j^*(\cdot; \lambda_{1:(j-1)}^{*, (2)})\|_{L^{\infty}(\R)} \\
    &\le 
    \|g_j(\cdot; \lambda_{1:(j-1)}^{*, (2)}) - g_j^*(\cdot; \lambda_{1:(j-1)}^{*, (2)})\|_{L^{\infty}(\R)}
    + \|g_j^+(\cdot; \lambda_{1:(j-1)}^{*, (2)}) - g_j(\cdot; \lambda_{1:(j-1)}^{*, (2)})\|_{L^{\infty}(\R)} \\
    &\le |\mathcal{S}_j| \sqrt{2} V^{\mathrm{max}}_j \sqrt{ \frac{n^{\zeta}}{n+1} }
    + \delta_j 
    \le ( |\mathcal{S}_j| \sqrt{2} V^{\mathrm{max}}_j + (V^{\mathrm{max}}_j - V^{\mathrm{min}}_j) ) \sqrt{ \frac{n^{\zeta}}{n+1} } 
    =: \widetilde C_j \sqrt{ \frac{n^{\zeta}}{n+1} }
\end{align*}
for $n\ge 1$,
where in the third step
we used the definition
$\delta_j := (V^{\mathrm{max}}_j-V^{\mathrm{min}}_j)/(n+1)$
and the inequality
$1/(n+1) \le \sqrt{n^{\zeta}/n} $
for $n\ge 1$,
and in the fourth step we defined
$\widetilde C_j = |\mathcal{S}_j| \sqrt{2} V^{\mathrm{max}}_j + (V^{\mathrm{max}}_j - V^{\mathrm{min}}_j)$.

Define the positive integer
$N_j^{(k)}$
as
\begin{align*}
N_j^{(k)}
=
\min\left\{ n' \in \NN : 
\widetilde C_j
\sqrt{ \frac{(n')^{\zeta}}{n'+1} }
\le
\frac{1}{2} \dist(\beta_j^{(k-1)}, \vv_j^{(0)})
\text{ and }
n' \ge 1
\right\}.
\end{align*}
We claim that $N_j^{(k)}$ is finite.
To see this,
note that
since
\Cref{cond:no-bad-beta-strong} 
implies that
$\dist(\beta_j, \vv_j^{(0)}) > 0$,
and since
$|\beta_j^{(k-1)} - \beta_j| = (k-1)\delta_j 
:= (k-1) \frac{V^{\mathrm{max}}_j-V^{\mathrm{min}}_j}{n+1}$,
for sufficiently large $n$
we have
$|\beta_j^{(k-1)} - \beta_j| < \dist(\beta_j, \vv_j^{(0)})$,
and hence
by the triangle inequality
for the Hausdorff distance,
we have
$\dist(\beta_j^{(k-1)}, \vv_j^{(0)}) > 0$.
Since
$\zeta\in (0,1)$
implies that
$\sqrt{(n')^{\zeta}/(n'+1)} \to 0$
as $n'\to\infty$,
it follows that
$N_j^{(k)}$ is finite.

Next,
for $n\ge N_j^{(k)}$,
we claim that
$U_j^+(\lambda_{1:(j-1)}^{*, (2)}; \beta_j^{(k-1)}) = U_j^*(\lambda_{1:(j-1)}^{*, (2)}; \beta_j^{(k-1)})$.
Indeed,
by
\Cref{eq:two-sided-crossing}
from
\Cref{lem:crossing-points},
we have
\begin{align*}
&U_j^*\left( \lambda_{1:(j-1)}^{*, (2)}; \beta_j^{(k-1)} 
+ \frac{1}{2} \dist(\beta_j^{(k-1)}, \vv_j^{(0)}) \right) 
\le U_j^+(\lambda_{1:(j-1)}^{*, (2)}; \beta_j^{(k-1)}) \\
&\le U_j^*\left( \lambda_{1:(j-1)}^{*, (2)}; \beta_j^{(k-1)}
- \frac{1}{2} \dist(\beta_j^{(k-1)}, \vv_j^{(0)}) \right).
\end{align*}
By \Cref{lem:g-star-pc},
$g_j^*(\cdot; \lambda_{1:(j-1)}^{*, (2)})$
is a step function,
so that by
\Cref{lem:pc-inverse},
we have
\begin{align*}
&U_j^*\left( \lambda_{1:(j-1)}^{*, (2)}; \beta_j^{(k-1)} 
+ \frac{1}{2} \dist(\beta_j^{(k-1)}, \vv_j^{(0)}) \right) 
= U_j^*\left( \lambda_{1:(j-1)}^{*, (2)}; \beta_j^{(k-1)} \right) \\
&= U_j^*\left( \lambda_{1:(j-1)}^{*, (2)}; \beta_j^{(k-1)} 
- \frac{1}{2} \dist(\beta_j^{(k-1)}, \vv_j^{(0)}) \right),
\end{align*}
which when combined with
the previous display
implies that
$U_j^+(\lambda_{1:(j-1)}^{*, (2)}; \beta_j^{(k-1)})
= U_j^*(\lambda_{1:(j-1)}^{*, (2)}; \beta_j^{(k-1)})$,
as desired.
It follows that
$\hat \lambda_j^{(2k)} = \lambda_j^{*,(2k)}$,
completing the induction.

\subsection{Additional lemmas for \Cref{thm:cts-conc} and \Cref{thm:disc-conc}}

\subsubsection{Uniform Lipschitz condition}\label{subsubsec:g-star-pert}

\begin{lemma}\label{lem:g-star-pert}
Assume that
the conditions
in \Cref{subsubsec:ub-conds}
and
\Cref{cond:cdf-lip}
hold.
Fix $j\in [m]$,
$\lambda_{1:(j-1)} \in \R^{j-1}$,
and $\alpha_{1:(j-1)}\in \R^{j-1}$
such that
$\alpha_{[j-1]\cap \mathcal{I}_{\mathrm{discrete}}} = 0$.
Then
\begin{align*}
    \left\| g_j^*\left( \cdot; \lambda_{1:(j-1)} + \alpha_{1:(j-1)} \right) - g_j^*\left( \cdot; \lambda_{1:(j-1)} \right) \right\|_{L^{\infty}(\R)} \le V^{\mathrm{max}}_j \sum_{\ell\in [j-1]\setminus \mathcal{I}_{\mathrm{discrete}}} K_{\ell} |\alpha_{\ell}|.
\end{align*}

\end{lemma}

\begin{proof}
We repeat the indicator manipulation
from the start of \Cref{subsubsec:reduction}.
We have for all $\lambda_j\in \R$ that
\begin{align}\label{eq:alpha-swapping-trick}
     & |g_j^*( \lambda_j; \lambda_{1:(j-1)} + \alpha_{1:(j-1)} ) - g_j^*( \lambda_j; \lambda_{1:(j-1)})| \\
     &= |\E V_j I(S_{1:(j-1)} \pce \lambda_{1:(j-1)} + \alpha_{1:(j-1)}) I(S_j > \lambda_j) - \E V_j I(S_{1:(j-1)} \pce \lambda_{1:(j-1)}) I(S_j > \lambda_j)| \nonumber \\
    &\le \sum_{\ell=1}^{j-1} | \E V_j I(S_{1:(\ell-1)} \pce \lambda_{1:(\ell-1)}, S_{\ell} \in ( \min\{ \lambda_{\ell}, \lambda_{\ell} + \alpha_{\ell}\}, \max\{ \lambda_{\ell}, \lambda_{\ell} + \alpha_{\ell}\} ], \nonumber \\
    & \qquad \qquad \qquad S_{(\ell+1):(j-1)} \pce \lambda_{(\ell+1):(j-1)} + \alpha_{(\ell+1):(j-1)}) I(S_j > \lambda_j) | \nonumber \\ 
    &\le V^{\mathrm{max}}_j \sum_{\ell=1}^{j-1} \PP{S_{\ell} \in ( \min\{ \lambda_{\ell}, \lambda_{\ell} + \alpha_{\ell}\}, \max\{ \lambda_{\ell}, \lambda_{\ell} + \alpha_{\ell}\} ]} \nonumber \\
    &= V^{\mathrm{max}}_j \sum_{\ell\in [j-1]\setminus \mathcal{I}_{\mathrm{discrete}}} \PP{S_{\ell} \in ( \min\{ \lambda_{\ell}, \lambda_{\ell} + \alpha_{\ell}\}, \max\{ \lambda_{\ell}, \lambda_{\ell} + \alpha_{\ell}\} ]} \nonumber \\
    &+ V^{\mathrm{max}}_j \sum_{\ell\in [j-1]\cap \mathcal{I}_{\mathrm{discrete}}} \PP{S_{\ell} \in ( \min\{ \lambda_{\ell}, \lambda_{\ell} + \alpha_{\ell}\}, \max\{ \lambda_{\ell}, \lambda_{\ell} + \alpha_{\ell}\} ]} \nonumber \\
    &\le V^{\mathrm{max}}_j \sum_{\ell\in [j-1]\setminus \mathcal{I}_{\mathrm{discrete}}} K_{\ell} |\alpha_{\ell}| \nonumber,
\end{align}
where in the fifth step we used
$\alpha_{\ell}=0$ for $\ell\in [j-1]\cap \mathcal{I}_{\mathrm{discrete}}$
and
\Cref{cond:cdf-lip}.
Note that we achieved a uniform bound over $\lambda_j\in \R$ by dropping the indicator $I(S_j > \lambda_j)$.
The result follows.
\end{proof}

\subsubsection{Reverse Hölder condition}\label{subsubsec:rev-lip-bds}

\begin{lemma}\label{lem:condl-rev-lip-bds}
Under
the conditions in
\Cref{subsubsec:ub-conds},
\Cref{cond:m-positive},
\Cref{cond:cpt-supp},
and
\Cref{cond:rev-lip},
for any $j\in \mathcal{I}_{\mathrm{cts}}$,
for any $\lambda_{1:(j-1)} \in \prod_{\ell=1}^{j-1} (0, B_{\ell})$,
for all $\lambda_j\le \lambda_j'$
with $\lambda_j, \lambda_j'\in (0,B_j)$,
we have the reverse Hölder condition
\begin{align}\label{eq:gen-rev-lip}
    g_j^*(\lambda_j;\lambda_{1:(j-1)}) - g_j^*(\lambda_j';\lambda_{1:(j-1)}) \ge \xi_j (\lambda_j' - \lambda_j)^{\nu},
\end{align}
where we define the constant $\xi_j := V^{\mathrm{min}}_j \widetilde K_j > 0$.
\end{lemma}

\begin{proof}
Fix $j\in \mathcal{I}_{\mathrm{cts}}$.
Given $\lambda_j \le \lambda_j'$
with $\lambda_j, \lambda_j'\in (0,B_j)$,
for any $\lambda_{1:(j-1)} \in \prod_{\ell=1}^{j-1} (0, B_{\ell})$,
we have the lower bound
\begin{align*}
    & g_j^*(\lambda_j;\lambda_{1:(j-1)}) - g_j^*(\lambda_j';\lambda_{1:(j-1)}) \\
    &= \E V_j I(S_{1:(j-1)} \pce \lambda_{1:(j-1)}, S_j\in ( \lambda_j, \lambda_j' ] ) \nonumber \\
    &\ge \E V^{\mathrm{min}}_j I(S_{1:(j-1)} \pce \lambda_{1:(j-1)}, S_j\in ( \lambda_j, \lambda_j' ] ) \nonumber \\
    &= V^{\mathrm{min}}_j (\PP{S_{1:(j-1)} \pce \lambda_{1:(j-1)}, S_j\le \lambda_j', S_{(j+1):m} \pce B_{(j+1):m}} \nonumber \\
    &\qquad - \PP{S_{1:(j-1)} \pce \lambda_{1:(j-1)}, S_j\le \lambda_j, S_{(j+1):m} \pce B_{(j+1):m}}) \nonumber \\ 
    &= V^{\mathrm{min}}_j (F(\lambda_{1:(j-1)}, \lambda_j', B_{(j+1):m}) - F(\lambda_{1:(j-1)}, \lambda_j, B_{(j+1):m})) \nonumber \\
    &\ge V^{\mathrm{min}}_j \widetilde K_j (\lambda_j' - \lambda_j)^{\nu}, \nonumber
\end{align*}
where in the second step we used
\Cref{cond:loss-bds},
in the third step we used
\Cref{cond:cpt-supp},
in the fourth step we let
$F$ denote the joint c.d.f.~
of $S_{1:m}$,
and in the fifth step we used
\Cref{cond:rev-lip}.
This implies
\Cref{eq:gen-rev-lip}
with
$\xi_j = V^{\mathrm{min}}_j \widetilde K_j$,
and by
\Cref{cond:m-positive},
$\xi_j > 0$.

\end{proof}

\subsubsection{Empirical process bound}\label{subsubsec:emp-proc-thy}

\begin{lemma}\label{lem:emp-proc-bd}
Under the conditions
in \Cref{subsubsec:ub-conds}
and under
\Cref{cond:iid-observations},
for $j\in [m]$,
for $n\ge 2$,
for any $\lambda_{1:(j-1)}\in \R^{j-1}$,
with probability at least $1 - 2/\sqrt n$,
\begin{align*}
     \|g_j^+(\cdot; \lambda_{1:(j-1)}) - g_j^*(\cdot; \lambda_{1:(j-1)})\|_{L^{\infty}(\R)} \le r_j \sqrt{\frac{\log n}{n}},
\end{align*}
where $r_j =  (16\sqrt{2} + 2)V^{\mathrm{max}}_j$.
\end{lemma}

\begin{proof}
Note that
\begin{align*}
    g_j^+(\lambda_j; \lambda_{1:(j-1)}) - g_j^*(\lambda_j; \lambda_{1:(j-1)}) &= \frac{1}{n+1} \sum_{i=1}^{n} (L_j^{(i)}(\lambda_{1:j}) - \E L_j^{(i)}(\lambda_{1:j})) + \frac{V^{\mathrm{max}}_j - \E L_j^{(1)}(\lambda_{1:j})}{n+1} \\
    &=: \frac{n}{n+1} W_{\lambda_j} + \frac{V^{\mathrm{max}}_j - \E L_j^{(1)}(\lambda_{1:j})}{n+1},
\end{align*}
where by
\Cref{cond:loss-bds}
and
\Cref{cond:iid-observations}
\begin{align*}
    W_{\lambda_j} := \frac{1}{n} \sum_{i=1}^{n} (L_j^{(i)}(\lambda_{1:j}) - \E L_j^{(i)}(\lambda_{1:j}))
\end{align*}
is an average of $n$ i.i.d. mean zero bounded random variables,
each having modulus bounded by $V^{\mathrm{max}}_j - V^{\mathrm{min}}_j$.
We want uniform control
of the process $W_{\lambda_j}$ over $\lambda_j\in \R$.
Recall that $V_j^{(i)} = V_j(X^{(i)}, Y^{(i)}, Y^{(i),*})$,
where by
\Cref{cond:loss-bds}
the function $V_j : \xx \times \yy \times \yy \to \R$
obeys $V_j \in [V^{\mathrm{min}}_j, V^{\mathrm{max}}_j]$
with $V^{\mathrm{min}}_j\ge 0$.
For fixed nonrandom $\lambda_{1:(j-1)} \in \R^{j-1}$,
define the function class
\begin{align*}
     \ff_j &= \{ f : \xx\times \yy\times \yy \to \R  \text{ of the form} \\ &f(x,y,y^*) = V_j(x,y,y^*) I(S_{1:(j-1)}(x,y,y^*)\le \lambda_{1:(j-1)}) I(S_j(x,y,y^*) > \lambda_j) : \lambda_j \in \R \}.
\end{align*}
For convenience,
let $\mz = \xx\times \yy\times \yy$,
and write $z = (x,y,y^*)$.
By
\Cref{cond:loss-bds},
$\ff_j$ is 
an $V^{\mathrm{max}}_j$-uniformly bounded function class.
By
\cite[Theorem 4.10]{wainwright2019high},
for each $\tau>0$,
with probability at least $1 - 2\exp(-\frac{n\tau^2}{2(V^{\mathrm{max}}_j)^2})$,
writing $Z^{(i)} = (X^{(i)}, Y^{(i)}, Y^{(i), *})$ for $i\in [n]$, 
we have
\begin{align*}
    \sup_{\lambda_j\in \R} |W_{\lambda_j}| = \sup_{f\in \ff_j} \left| \frac{1}{n} \sum_{i=1}^{n} (f(Z^{(i)}) - \E f(Z^{(i)})) \right| \le 2\rr_{n}(\ff_j) + \tau.
\end{align*}
Note that
$\delta = 2\exp(-\frac{n\tau^2}{2(V^{\mathrm{max}}_j)^2})$
is equivalent to
$\tau = V^{\mathrm{max}}_j \sqrt{  2 \log \frac{2}{\delta} } \frac{1}{\sqrt{n}}$.
Thus, for each $\delta>0$,
with probability at least $1-\delta$,
we have
\begin{align*}
    \sup_{\lambda_j\in \R} |W_{\lambda_j}| \le 2\rr_{n}(\ff_j) + V^{\mathrm{max}}_j \sqrt{  2 \log \frac{2}{\delta} } \frac{1}{\sqrt{n}},
\end{align*}
which implies that 
with probability at least $1-\delta$,
\begin{align*}
    &\|g_j^+(\cdot; \lambda_{1:(j-1)}) - g_j^*(\cdot; \lambda_{1:(j-1)})\|_{L^{\infty}(\R)} \\ 
    &\le \frac{n}{n+1} \left( 2\rr_{n}(\ff_j) + V^{\mathrm{max}}_j \sqrt{  2 \log \frac{2}{\delta} } \frac{1}{\sqrt{n}} \right) + \left| \frac{V^{\mathrm{max}}_j - \E L_j^{(1)}(\lambda_{1:j})}{n+1} \right| \\
    &\le 2\rr_{n}(\ff_j) + V^{\mathrm{max}}_j \sqrt{  2 \log \frac{2}{\delta} } \frac{1}{\sqrt{n}} + \frac{V^{\mathrm{max}}_j}{n+1},
\end{align*}
where in the second step
we used \Cref{cond:loss-bds}.
Setting $\delta = 2/\sqrt n$, we deduce that
with probability at least $1 - 2/\sqrt n$,
\begin{align*}
    \|g_j^+(\cdot; \lambda_{1:(j-1)}) - g_j^*(\cdot; \lambda_{1:(j-1)})\|_{L^{\infty}(\R)} \le 2\rr_{n}(\ff_j) + V^{\mathrm{max}}_j \sqrt{\frac{\log n}{n}} + \frac{V^{\mathrm{max}}_j}{n+1}.
\end{align*}
Assume that
$\rr_n(\ff_j) \le \tilde r_j \sqrt{\log n / n}$
for some constant $\tilde r_j > 0$,
uniformly
in $\lambda_{1:(j-1)} \in \R^{j-1}$.
We deduce that
for $n\ge 2$,
with probability at least $1 - 2/\sqrt n$,
\begin{align}\label{eq:emp-proc-final-bd}
    &\|g_j^+(\cdot; \lambda_{1:(j-1)}) - g_j^*(\cdot; \lambda_{1:(j-1)})\|_{L^{\infty}(\R)} \le 2\tilde r_j \sqrt{\frac{\log n}{n}} + V^{\mathrm{max}}_j \sqrt{\frac{\log n}{n}} + \frac{V^{\mathrm{max}}_j}{n+1} \\
    &\le (2\tilde r_j + 2V^{\mathrm{max}}_j) \sqrt{\frac{\log n}{n}}
    =: r_j \sqrt{\frac{\log n}{n}}, \nonumber 
\end{align}
where in the second step
we used the inequality
$1/(n+1) \le \sqrt{\log n / n}$
for $n\ge 2$,
and in the third step
we defined
$r_j = 2\tilde r_j + 2V^{\mathrm{max}}_j$.

It remains only to bound $\rr_n(\ff_j)$.
Note that $\ff_j = \tilde f_j \cdot \mathcal{G}_j$,
where $\tilde f_j : \mz \to \R$ is given by 
\begin{align*}
    \tilde f_j(x,y,y^*) = V_j(x,y,y^*) I(S_{1:(j-1)}(x,y,y^*)\le \lambda_{1:(j-1)}),
\end{align*}
where the function class $\mathcal{G}_j$ is given by
\begin{align*}
    \mathcal{G}_j = \{ f : \xx\times \yy\times \yy \to \R  \text{ of the form } f(x,y,y^*) = I(S_j(x,y,y^*) > \lambda_j) : \lambda_j \in \R \},
\end{align*}
and where for a fixed $\phi : \mz \to \R$
and for a class $\mathcal{G}$ of functions from $\mz\to \R$,
we define $\phi \cdot \mathcal{G} := \{ \phi \cdot g : g\in \mathcal{G}\}$.
Note that
since $\mathcal{G}_j$
is $1$-uniformly bounded
and has VC dimension 1,
\cite[Lemma 4.14]{wainwright2019high}
implies that
$\rr_k(\mathcal{G}_j) \le 4\sqrt{\log(k+1) / k}$
for all positive integers $k$.
Also, note that by
\Cref{cond:loss-bds},
$|\tilde f_j| \le V^{\mathrm{max}}_j$.

Fix $z_{1:n} \in \mz^n$.
For $i\in [n]$,
define the linear function
$\varphi_i : \R \to \R$ by
$\varphi_i(t) = \frac{\tilde f_j(z_i)}{V^{\mathrm{max}}_j} t$
for $t\in \R$.
Clearly $\varphi_i(0) = 0$,
and since
$|\tilde f_j| \le V^{\mathrm{max}}_j$,
$\varphi_i$ is a contraction
for $i\in [n]$.
Thus, by the Ledoux-Talagrand contraction lemma in its empirical form
\citep[Theorem 4.12]{ledoux2013probability},
we have the empirical Rademacher complexity bound
$\widehat \rr_{z_{1:n}}(\ff_j)\le 2V^{\mathrm{max}}_j \widehat \rr_{z_{1:n}}(\mathcal{G}_j)$.
Plugging in the random variables $Z_{1:n}$
and taking the expectation of both sides,
we deduce that
$\rr_n(\ff_j)\le 2 V^{\mathrm{max}}_j \rr_n(\mathcal{G}_j)$,
so that
\begin{align*}
    \rr_n(\ff_j) \le 2V^{\mathrm{max}}_j \cdot 4\sqrt{\frac{\log(n+1)}{n}} \le 8V^{\mathrm{max}}_j \sqrt{\frac{2\log n}{n}}  =: \tilde r_j \sqrt{\frac{\log n}{n}},
\end{align*}
where in the second step
we used the inequality
$\log(n+1) \le 2\log n$
for $n\ge 2$,
and in the third step we defined
$\tilde r_j := 8\sqrt{2} V^{\mathrm{max}}_j$.
The desired bound follows.
\end{proof}

\subsubsection{Step function properties}

\begin{lemma}\label{lem:g-star-pc}
Suppose that
$j\in \mathcal{I}_{\mathrm{discrete}}$
and suppose that
the conditions in
\Cref{subsubsec:ub-conds}
hold.
Suppose that
$S_j$ is a.s. $\mathcal{S}_j$-valued,
where $\mathcal{S}_j \subseteq \R$
with $|\mathcal{S}_j| < \infty$.
Then
for any $\lambda_{1:(j-1)}\in \R^{j-1}$,
$g_j^*(\cdot; \lambda_{1:(j-1)})$ is a
right-continuous 
step function
from $\R\to \R$
with at most $|\mathcal{S}_j|$ jump discontinuities.
\end{lemma}

\begin{proof}
By
\Cref{lem:pop-monot},
for any $j\in [m]$,
for any $\lambda_{1:(j-1)}\in \R^{j-1}$,
$g_j^*(\cdot;\lambda_{1:(j-1)})$
is right-continuous.
Sort the elements of $\mathcal{S}_j$
as
$s_j^{(1)} < \ldots < s_j^{(|\mathcal{S}_j|)}$.
If
$\lambda_j, \lambda_j' \in (s_j^{(i)}, s_j^{(i+1)})$
for some
$i\in [|\mathcal{S}_j|-1]$,
then since
$I(S_j > \lambda_j) = I(S_j > \lambda_j')$,
we have
\begin{align*}
&g_j^*(\lambda_j; \lambda_{1:(j-1)})
= \E V_j I(S_{1:(j-1)}\le \lambda_{1:(j-1)}) I(S_j > \lambda_j) \\
&= \E V_j I(S_{1:(j-1)}\le \lambda_{1:(j-1)}) I(S_j > \lambda_j')
= g_j^*(\lambda_j'; \lambda_{1:(j-1)}),
\end{align*}
hence
$g_j^*(\cdot; \lambda_{1:(j-1)})$
is constant on
$(s_j^{(i)}, s_j^{(i+1)})$.
If $\lambda_j, \lambda_j' < s_j^{(1)}$,
then
$I(S_j > \lambda_j) = I(S_j > \lambda_j') = 1$,
so that by the same logic as above,
we have
$g_j^*(\lambda_j; \lambda_{1:(j-1)}) 
= g_j^*(\lambda_j'; \lambda_{1:(j-1)})$,
hence 
$g_j^*(\cdot; \lambda_{1:(j-1)})$
is constant on
$(-\infty, s_j^{(1)})$.
Finally,
if
$\lambda_j, \lambda_j' > s_j^{(|\mathcal{S}_j|)}$,
then
$I(S_j > \lambda_j) = I(S_j > \lambda_j') = 0$,
so that 
$g_j^*(\lambda_j; \lambda_{1:(j-1)}) 
= g_j^*(\lambda_j'; \lambda_{1:(j-1)}) = 0$,
hence 
$g_j^*(\cdot; \lambda_{1:(j-1)})$
is constant
and equal to zero
on
$(s_j^{(|\mathcal{S}_j|)}, \infty)$.
The claim follows.
\end{proof}

\begin{lemma}\label{lem:pc-inverse}
Suppose that
$f : \R \to \R$
is a right-continuous
non-increasing
step function
with $|f(\R)| < \infty$.
Then if
$\theta\in \R\setminus f(\R)$
and
$\psi\in \R\setminus f(\R)$
lie in the same connected component
of $\R \setminus f(\R)$,
for any compact interval $I\subseteq \R$,
we have
\begin{align*}
\sup\{ x\in I : f(x) > \theta \} = 
\sup\{ x\in I : f(x) > \psi \}.
\end{align*}
Here, we define the supremum
of an empty subset of $I$ to be $\min I$.
\end{lemma}

\begin{proof}
Sort the elements of $f(\R)$ as
$v_1 < \ldots < v_{|f(\R)|}$.
Then we may write
\begin{align*}
    \R\setminus f(\R) = (-\infty, v_1)
    \cup \left( \bigcup_{i=1}^{|f(\R)|-1} (v_i, v_{i+1}) \right)
    \cup (v_{|f(\R)|}, \infty).
\end{align*}
If $\theta, \psi \in (v_i, v_{i+1})$
for some $i\in [|f(\R)|-1]$,
then since $f(x) > \theta$ iff $f(x) > \psi$ iff $f(x) > v_i$,
it follows that
\begin{align*}
\sup\{ x\in \R : f(x) > \theta \}
= \sup\{ x\in \R : f(x) > \psi \}
= \sup\{ x\in \R : f(x) > v_i \}.
\end{align*}
If $\theta,\psi \in (v_{|f(\R)|}, \infty)$,
then since
$f(x) \le v_{|f(\R)|} < \min\{\theta,\psi\}$
for all $x\in \R$,
by our convention
that the supremum
of an empty subset of $I$
is $\min I$,
we have
\begin{align*}
\sup\{ x\in I : f(x) > \theta \}
= \sup\{ x\in I : f(x) > \psi \}
= \min I. 
\end{align*}
Finally,
if $\theta,\psi \in (-\infty,v_1) $,
then since
$f(x) \ge v_1 > \max\{\theta,\psi\}$
for all $x\in \R$,
we have
\begin{align*}
\sup\{ x\in I : f(x) > \theta \}
= \sup\{ x\in I : f(x) > \psi \}
= \max I, 
\end{align*}
as desired.
\end{proof}

\subsubsection{Crossing point perturbations}

\begin{lemma}\label{lem:conn-cmpt-pert}
Assume that
the conditions
in \Cref{subsubsec:ub-conds}
and
\Cref{cond:cdf-lip}
hold.
Suppose that for some
$j\in \mathcal{I}_{\mathrm{discrete}}$,
$\beta_j'\in \R$,
and
$\lambda_{1:(j-1)}\in \R^{j-1}$,
we have
$\beta_j'\not\in g_j^*(\R; \lambda_{1:(j-1)})$.
Then
for
$\alpha_{1:(j-1)}\in \R^{j-1}$
such that
$\alpha_{[j-1]\cap \mathcal{I}_{\mathrm{discrete}}} = 0$,
if $\|\alpha_{1:(j-1)}\|_1$
is sufficiently small,
we have
$\beta_j'\not\in
g_j^*(\R; \lambda_{1:(j-1)} + \alpha_{1:(j-1)})$.

\end{lemma}

\begin{proof}
By
\Cref{lem:g-star-pert},
we have
\begin{align*}
    \| g_j^*(\cdot; \lambda_{1:(j-1)} + \alpha_{1:(j-1)}) - g_j^*(\cdot; \lambda_{1:(j-1)}) \|_{L^{\infty}(\R)}
    &\le V^{\mathrm{max}}_j \sum_{\ell\in [j-1]\setminus \mathcal{I}_{\mathrm{discrete}}} K_{\ell} |\alpha_{\ell}| \\
    &\le V^{\mathrm{max}}_j \max_{\ell\in [j-1]\setminus \mathcal{I}_{\mathrm{discrete}}} K_{\ell} \cdot \|\alpha_{1:(j-1)}\|_1.
\end{align*}
Since
$\dist(\beta_j', g_j^*(\R; \lambda_{1:(j-1)}))
> 0$,
for
\begin{align*}
\|\alpha_{1:(j-1)}\|_1
\le \frac{\frac{1}{2} \dist(\beta_j', g_j^*(\R; \lambda_{1:(j-1)}))}{V^{\mathrm{max}}_j \max_{\ell\in \mathcal{I}_{\mathrm{cts}}} K_{\ell}},
\end{align*}
we have
\begin{align*}
\| g_j^*(\cdot; \lambda_{1:(j-1)} + \alpha_{1:(j-1)}) - g_j^*(\cdot; \lambda_{1:(j-1)}) \|_{L^{\infty}(\R)} \le \frac{1}{2} \dist(\beta_j', g_j^*(\R; \lambda_{1:(j-1)})),
\end{align*}
or equivalently,
\begin{align*}
\dist(g_j^*(\R; \lambda_{1:(j-1)} + \alpha_{1:(j-1)}), g_j^*(\R; \lambda_{1:(j-1)}))
\le \frac{1}{2} \dist(\beta_j', g_j^*(\R; \lambda_{1:(j-1)})).
\end{align*}
By the triangle inequality
for the Hausdorff distance,
it follows that
\begin{align*}
\dist(\beta_j', g_j^*(\R; \lambda_{1:(j-1)} + \alpha_{1:(j-1)})) \ge \frac{1}{2} \dist(\beta_j', g_j^*(\R; \lambda_{1:(j-1)}))  > 0,
\end{align*}
as desired.
\end{proof}

\begin{lemma}\label{lem:g-star-cts}
Suppose that $j\in \mathcal{I}_{\mathrm{cts}}$,
and suppose that
\Cref{cond:loss-bds}
and
\Cref{cond:cdf-lip}
hold.
Then for any $\lambda_{1:(j-1)}\in \R^{j-1}$,
$g_j^*(\cdot; \lambda_{1:(j-1)})$
is continuous
from $\R\to \R$.
\end{lemma}

\begin{proof}
Fix $\lambda_{1:(j-1)}\in \R^{j-1}$.
Then given $\lambda_j, \lambda_j'\in \R$
with $\lambda_j\le \lambda_j'$,
we have
\begin{align*}
|g_j^*(\lambda_j; \lambda_{1:(j-1)})
-
g_j^*(\lambda_j'; \lambda_{1:(j-1)})|
&=
|\E V_j I(S_{1:(j-1)}\pce \lambda_{1:(j-1)}) I(S_j\in (\lambda_j, \lambda_j'])| \\
\le
V^{\mathrm{max}}_j \PP{S_j\in (\lambda_j, \lambda_j']} 
&\le V^{\mathrm{max}}_j K_j |\lambda_j' - \lambda_j|,
\end{align*}
where in the second step we used
\Cref{cond:loss-bds},
and
in the third step we used
\Cref{cond:cdf-lip}.
It follows that $g_j^*(\cdot; \lambda_{1:(j-1)})$
is Lipschitz, which implies the result.
\end{proof}

\begin{lemma}\label{lem:fn-inv-diff}
Suppose that
$j\in \mathcal{I}_{\mathrm{cts}}$,
and suppose that
\Cref{cond:loss-bds}
and
\Cref{cond:cdf-lip}
hold.
Then
for any $\lambda_{1:(j-1)}\in \R^{j-1}$,
for any $\beta \in \R$,
\begin{align}\label{eq:fn-inv}
g_j^*(U_j^*(\lambda_{1:(j-1)}; \beta); \lambda_{1:(j-1)})
= \max\{ g_j^*(\lambda_j^{\mathrm{max}}; \lambda_{1:(j-1)}), \min\{ g_j^*(\lambda_j^{\mathrm{min}}; \lambda_{1:(j-1)}), \beta \} \}.
\end{align}
Consequently,
for any $\lambda_{1:(j-1)}\in \R^{j-1}$,
for any $\beta, \beta'\in \R$,
we have
\begin{align}\label{eq:fn-inv-diff}
|g_j^*(U_j^*(\lambda_{1:(j-1)}; \beta); \lambda_{1:(j-1)})
- g_j^*(U_j^*(\lambda_{1:(j-1)}; \beta'); \lambda_{1:(j-1)})|
\le
|\beta-\beta'|.
\end{align}
\end{lemma}

\begin{proof}
To prove
\Cref{eq:fn-inv},
we consider three cases.
First, suppose that 
$\beta$ is in the interval
$[ g_j^*(\lambda_j^{\mathrm{max}}; \lambda_{1:(j-1)}), 
g_j^*(\lambda_j^{\mathrm{min}}; \lambda_{1:(j-1)}) ]$.
Since
\Cref{cond:loss-bds}
and
\Cref{cond:cdf-lip}
hold,
we may apply
\Cref{lem:g-star-cts}
to deduce that
$g_j^*(\cdot;\lambda_{1:(j-1)})$ is continuous on $\Lambda_j$.
Thus, by the intermediate value theorem,
we have
$g_j^*(U_j^*(\lambda_{1:(j-1)}; \beta); \lambda_{1:(j-1)}) = \beta$,
as claimed.
Next, if $\beta > g_j^*(\lambda_j^{\mathrm{min}}; \lambda_{1:(j-1)})$,
then by the definition of $U_j^*$,
we have
$U_j^*(\lambda_{1:(j-1)}; \beta) = \lambda_j^{\mathrm{min}}$,
so that
$g_j^*(U_j^*(\lambda_{1:(j-1)}; \beta); \lambda_{1:(j-1)}) = g_j^*(\lambda_j^{\mathrm{min}}; \lambda_{1:(j-1)})$,
as claimed.
Finally,
if $\beta < g_j^*(\lambda_j^{\mathrm{max}}; \lambda_{1:(j-1)})$,
then by the definition of $U_j^*$,
we have
$U_j^*(\lambda_{1:(j-1)}; \beta) = \lambda_j^{\mathrm{max}}$,
so that
$g_j^*(U_j^*(\lambda_{1:(j-1)}; \beta); \lambda_{1:(j-1)}) = g_j^*(\lambda_j^{\mathrm{max}}; \lambda_{1:(j-1)})$,
as claimed.

The inequality
\Cref{eq:fn-inv-diff}
then follows
from the fact that 
projection onto the interval
$[ g_j^*(\lambda_j^{\mathrm{max}}; \lambda_{1:(j-1)}), 
g_j^*(\lambda_j^{\mathrm{min}}; \lambda_{1:(j-1)}) ]$
is nonexpansive.
\end{proof}

\begin{lemma}\label{lem:crossing-pert}
Assume that
the conditions
in \Cref{subsubsec:ub-conds},
\Cref{cond:m-positive},
and the conditions in
\Cref{subsubsec:cts-conc-conds}
hold.
Then
for any $j\in [m]$,
for any $k\in [m-j+1]$,
for $\alpha_{1:(j-1)}\in \R^{j-1}$
with $\alpha_{[j-1]\cap \mathcal{I}_{\mathrm{discrete}}} = 0$,
for $\ep_j\in \R$,
for $\ep_j'\in \R$,
if $\|\alpha_{1:(j-1)}\|_1$ is
sufficiently small,
if $|\ep_j|$
is sufficiently small,
and
if $|\ep_j'|$
is sufficiently small,
we have
\begin{align*}
|U_j^*(\lambda_{1:(j-1)}^{*,(2)} + \alpha_{1:(j-1)}; (\beta_j + \ep_j) + \ep_j')
- U_j^*(\lambda_{1:(j-1)}^{*,(2)} + \alpha_{1:(j-1)}; \beta_j + \ep_j)|
\le
\mu_j |\ep_j'|^{1/\nu},
\end{align*}
where
$\mu_j := 1/\xi_j^{1/\nu}$ if $j\in \mathcal{I}_{\mathrm{cts}}$
and
$\mu_j := 0$ if $j\in \mathcal{I}_{\mathrm{discrete}}$,
where
the constant $\nu\ge 1$
is defined in
\Cref{cond:rev-lip},
and where
the positive constant
$\xi_j$ is defined in
\Cref{lem:condl-rev-lip-bds}.
\end{lemma}

\begin{proof}
First, suppose that
$j\in \mathcal{I}_{\mathrm{cts}}$.
Since by
\Cref{cond:beta-tilde-exists},
we have
$\lambda_{1:(j-1)}^{*,(2)}\in \prod_{\ell=1}^{j-1} (0,B_{\ell})$,
if $\alpha_{1:(j-1)}\in \R^{j-1}$
is such that
$\|\alpha_{1:(j-1)}\|_1$
is sufficiently small,
then we also have
$\lambda_{1:(j-1)}^{*,(2)} + \alpha_{1:(j-1)}
\in \prod_{\ell=1}^{j-1} (0,B_{\ell})$.
Thus, by
\Cref{lem:condl-rev-lip-bds}
and by
\Cref{eq:fn-inv-diff}
from
\Cref{lem:fn-inv-diff},
we have
\begin{align*}
    |\ep_j'| = &| ((\beta_j + \ep_j) + \ep_j') - (\beta_j + \ep_j) | \\
    &\ge 
    \xi_j 
    |U_j^*(\lambda_{1:(j-1)}^{*,(2)} + \alpha_{1:(j-1)}; (\beta_j + \ep_j) + \ep_j')
- U_j^*(\lambda_{1:(j-1)}^{*,(2)} + \alpha_{1:(j-1)}; \beta_j + \ep_j)|^{\nu},
\end{align*}
which implies that
\begin{align*}
&|U_j^*(\lambda_{1:(j-1)}^{*,(2)} + \alpha_{1:(j-1)}; (\beta_j + \ep_j) + \ep_j')
- U_j^*(\lambda_{1:(j-1)}^{*,(2)} + \alpha_{1:(j-1)}; \beta_j + \ep_j)| \\
&\le
\frac{1}{\xi_j^{1/\nu}} |\ep_j'|^{1/\nu} 
=: \mu_j |\ep_j'|^{1/\nu},
\end{align*}
as claimed.

Next,
suppose that
$j\in [m]\cap \mathcal{I}_{\mathrm{discrete}}$.
By
\Cref{cond:no-bad-beta},
$\dist(\beta_j, \vv_j^{(0)}) > 0$.
Thus, by
\Cref{lem:conn-cmpt-pert},
if $\|\alpha_{1:(j-1)}\|_1$ is sufficiently small,
then
$\dist(\beta_j, 
g_j^*(\R; \lambda_{1:(j-1)}^{*,(2)} + \alpha_{1:(j-1)}) )
> 0$,
so that for
$|\ep_j|$ sufficiently small
and
$|\ep_j'|$ sufficiently small,
we have that
$(\beta_j + \ep_j)$
and
$(\beta_j + \ep_j) + \ep_j'$
lie
in the same connected component of
$\R\setminus
g_j^*(\R; 
\lambda_{1:(j-1)}^{*,(2)} + \alpha_{1:(j-1)})$.
By
\Cref{lem:pc-inverse}
and the definition of $U_j^*$,
it follows that
\begin{align*}
U_j^*(\lambda_{1:(j-1)}^{*,(2)} + \alpha_{1:(j-1)}; (\beta_j + \ep_j) + \ep_j')
= U_j^*(\lambda_{1:(j-1)}^{*,(2)} + \alpha_{1:(j-1)}; \beta_j + \ep_j),
\end{align*}
hence
\begin{align*}
|U_j^*(\lambda_{1:(j-1)}^{*,(2)} + \alpha_{1:(j-1)}; (\beta_j + \ep_j) + \ep_j')
- U_j^*(\lambda_{1:(j-1)}^{*,(2)} + \alpha_{1:(j-1)}; \beta_j + \ep_j)|
= 0
= \mu_j |\ep_j'|,
\end{align*}
as claimed.
\end{proof}

\end{document}